\documentclass[12pt]{article}

\usepackage{etoolbox}
\newtoggle{colt}
\togglefalse{colt}

\usepackage[numbers,sort&compress]{natbib}
\usepackage{boxedminipage}
\usepackage{multirow,nicefrac}
\usepackage{makecell,upgreek}
\usepackage{footnote}
\usepackage{longtable}
\usepackage{tablefootnote}
\usepackage[T1]{fontenc}
\usepackage{verbatim} 
\usepackage[utf8]{inputenc}

\usepackage{booktabs}   
\usepackage{adjustbox}

\usepackage[margin=0.75in]{geometry}
\usepackage[ruled,vlined]{algorithm2e}

\usepackage[dvipsnames]{xcolor}

\usepackage{subfigure}

\usepackage{wrapfig}

\usepackage{colortbl}
\definecolor{lightgray}{gray}{0.9}  %

\usepackage{amsmath}
\usepackage{amsthm}

\usepackage[breaklinks=true]{hyperref}
\hypersetup{
	colorlinks=true,
	linkcolor=blue,
	citecolor=blue,
	urlcolor=blue
}

\usepackage[capitalise,nameinlink]{cleveref}
\usepackage{amsfonts,amssymb,mathrsfs}
\usepackage{mathtools}

\usepackage{appendix}

\theoremstyle{plain}
	\newtheorem{theorem}{Theorem}

	\newtheorem{lemma}{Lemma}
	
	\newtheorem{corollary}{Corollary}
	\newtheorem{proposition}{Proposition}

	\theoremstyle{definition}

	\newtheorem{remark}{Remark}

\Crefname{appendix}{Appendix}{Appendices}

\usepackage{autonum}

\usepackage[T1]{fontenc}
\usepackage[scaled]{beramono}
\usepackage{listings}

\definecolor{darkred}{RGB}{139,0.0,0.0}
\definecolor{varorange}{RGB}{230,126,34}     %
\definecolor{opblue}{RGB}{52,152,219}        %
\definecolor{ifred}{RGB}{192,57,43}          %
\definecolor{commentgray}{RGB}{150,150,150}  %
\definecolor{codebg}{rgb}{0.98,0.98,0.98}    %

\newsavebox\codebox

\usepackage{authblk}

\newcommand{\norm}[1]{\left\|#1\right\|}
\newcommand{\normop}[1]{\left\|#1\right\|_{\mathrm{op}}}
\newcommand{\normf}[1]{\left\|#1\right\|_{\mathrm{F}}}
\newcommand{\abs}[1]{\left|#1\right|}
\newcommand{\inprod}[2]{\left\langle #1, #2 \right\rangle}

\newcommand{\rr}{\mathbb{R}}
\newcommand{\ee}{\mathbb{E}}
\newcommand{\pp}{\mathbb{P}}

\newcommand{\algcomment}[1]{{\hfill \scriptsize\textcolor{blue}{(\%\% #1)}}}

\newcommand{\eye}{\mathbf{I}}

\DeclareMathOperator{\Var}{Var}
\DeclareMathOperator{\Cov}{Cov}
\DeclareMathOperator{\trace}{Tr}

\newcommand{\muhat}{\widehat{\mu}}

\newcommand{\mumle}{\widehat{\mu}^{\textsf{MLE}}}
\newcommand{\mumletil}{\widetilde{\mu}^{\textsf{MLE}}}
\newcommand{\muema}{\muhat^{\textsf{EMA}}}
\newcommand{\muouema}{\muhat^{\textsf{OUEMA}}}

\newcommand{\mustar}{\mu^\star}
\newcommand{\bSigma}{\mathbf{\Sigma}}
\newcommand{\bAtil}{\widetilde{\bA}}
\newcommand{\lambdamin}{\lambda_{\min}}
\newcommand{\lambdamax}{\lambda_{\max}}
\newcommand{\thetabar}{\bar{\theta}}
\newcommand{\thetaema}{\theta^{\textsf{EMA}}}
\newcommand{\Kbar}{\bar{K}}

\newcommand{\thetadema}{\theta^{\textsf{DEMA}}}
\newcommand{\thetaematwo}{\theta^{\textsf{EMA,EMA}}}

\newcommand{\thetamu}{\theta^\mu}
\newcommand{\ppmu}{\pp^\mu}
\newcommand{\ppW}{\pp^W}
\newcommand{\fisher}{\mathbb{I}}

\newcommand{\oumle}{\textsf{BEMA}}
\newcommand{\bema}{\oumle}
\newcommand{\ouema}{\textsf{OUEMA}}
\newcommand{\ema}{\textsf{EMA}}
\newcommand{\dema}{\textsf{DEMA}}
\newcommand{\qwen}{\textsf{Qwen2.5-1.5B}}
\newcommand{\gemma}{\textsf{Gemma3-1B}}
\newcommand{\llama}{\textsf{Llama3.2-1B}}
\newcommand{\boolq}{\textsf{BoolQ}}
\newcommand{\mmluhs}{\textsf{MMLU-HS}}
\newcommand{\gsmk}{\textsf{GSM8K}}
\newcommand{\tulu}{\textsf{Tulu-3-SFT}}

\renewcommand{\epsilon}{\varepsilon}

\def\ddefloop#1{\ifx\ddefloop#1\else\ddef{#1}\expandafter\ddefloop\fi}
\def\ddef#1{\expandafter\def\csname bb#1\endcsname{\ensuremath{\mathbb{#1}}}}
\ddefloop ABCDEFGHIJKLMNOPQRSTUVWXYZ\ddefloop

\def\ddefloop#1{\ifx\ddefloop#1\else\ddef{#1}\expandafter\ddefloop\fi}
\def\ddef#1{\expandafter\def\csname frak#1\endcsname{\ensuremath{\mathfrak{#1}}}}
\ddefloop ABCDEFGHIJKLMNOPQRSTUVWXYZ\ddefloop

\def\ddefloop#1{\ifx\ddefloop#1\else\ddef{#1}\expandafter\ddefloop\fi}
\def\ddef#1{\expandafter\def\csname fr#1\endcsname{\ensuremath{\mathfrak{#1}}}}
\ddefloop ABCDEFGHIJKLMNOPQRSTUVWXYZ\ddefloop

\def\ddefloop#1{\ifx\ddefloop#1\else\ddef{#1}\expandafter\ddefloop\fi}
\def\ddef#1{\expandafter\def\csname eul#1\endcsname{\ensuremath{\EuScript{#1}}}}
\ddefloop ABCDEFGHIJKLMNOPQRSTUVWXYZ\ddefloop

\def\ddefloop#1{\ifx\ddefloop#1\else\ddef{#1}\expandafter\ddefloop\fi}
\def\ddef#1{\expandafter\def\csname scr#1\endcsname{\ensuremath{\mathscr{#1}}}}
\ddefloop ABCDEFGHIJKLMNOPQRSTUVWXYZ\ddefloop

\def\ddefloop#1{\ifx\ddefloop#1\else\ddef{#1}\expandafter\ddefloop\fi}
\def\ddef#1{\expandafter\def\csname b#1\endcsname{\ensuremath{\mathbf{#1}}}}
\ddefloop ABCDEGHIJKLMNOPQRSTUVWXYZ\ddefloop

\def\ddefloop#1{\ifx\ddefloop#1\else\ddef{#1}\expandafter\ddefloop\fi}
\def\ddef#1{\expandafter\def\csname bhat#1\endcsname{\ensuremath{\hat{\mathbf{#1}}}}}
\ddefloop ABCDEFGHIJKLMNOPQRSTUVWXYZ\ddefloop

\def\ddefloop#1{\ifx\ddefloop#1\else\ddef{#1}\expandafter\ddefloop\fi}
\def\ddef#1{\expandafter\def\csname btil#1\endcsname{\ensuremath{\tilde{\mathbf{#1}}}}}
\ddefloop ABCDEFGHIJKLMNOPQRSTUVWXYZ\ddefloop

\def\ddefloop#1{\ifx\ddefloop#1\else\ddef{#1}\expandafter\ddefloop\fi}
\def\ddef#1{\expandafter\def\csname bst#1\endcsname{\ensuremath{\mathbf{#1}^\star}}}
\ddefloop ABCDEFGHIJKLMNOPQRSTUVWXYZ\ddefloop

\def\ddefloop#1{\ifx\ddefloop#1\else\ddef{#1}\expandafter\ddefloop\fi}
\def\ddef#1{\expandafter\def\csname bst#1\endcsname{\ensuremath{\mathbf{#1}^\star}}}
\ddefloop abcdeghijklmnopqrstuvwxyz\ddefloop

\def\ddefloop#1{\ifx\ddefloop#1\else\ddef{#1}\expandafter\ddefloop\fi}
\def\ddef#1{\expandafter\def\csname bhat#1\endcsname{\ensuremath{\hat{\mathbf{#1}}}}}
\ddefloop abcdefghijklmnopqrstuvwxyz\ddefloop

\def\ddefloop#1{\ifx\ddefloop#1\else\ddef{#1}\expandafter\ddefloop\fi}
\def\ddef#1{\expandafter\def\csname b#1\endcsname{\ensuremath{\mathbf{#1}}}}
\ddefloop abcdeghijklnopqrstuvwxyz\ddefloop

\def\ddefloop#1{\ifx\ddefloop#1\else\ddef{#1}\expandafter\ddefloop\fi}
\def\ddef#1{\expandafter\def\csname barb#1\endcsname{\ensuremath{\bar{\mathbf{#1}}}}}
\ddefloop abcdefghijklmnopqrstuvwxyz\ddefloop

\def\ddef#1{\expandafter\def\csname c#1\endcsname{\ensuremath{\mathcal{#1}}}}
\ddefloop ABCDEFGHIJKLMNOPQRSTUVWXYZ\ddefloop
\def\ddef#1{\expandafter\def\csname h#1\endcsname{\ensuremath{\widehat{#1}}}}
\ddefloop ABCDEFGHIJKLMNOPQRSTUVWXYZ\ddefloop
\def\ddef#1{\expandafter\def\csname hc#1\endcsname{\ensuremath{\widehat{\mathcal{#1}}}}}
\ddefloop ABCDEFGHIJKLMNOPQRSTUVWXYZ\ddefloop
\def\ddef#1{\expandafter\def\csname t#1\endcsname{\ensuremath{\widetilde{#1}}}}
\ddefloop ABCDEFGHIJKLMNOPQRSTUVWXYZ\ddefloop
\def\ddef#1{\expandafter\def\csname tc#1\endcsname{\ensuremath{\widetilde{\mathcal{#1}}}}}
\ddefloop ABCDEFGHIJKLMNOPQRSTUVWXYZ\ddefloop

\usepackage{color-edits}
\addauthor{ab}{red}

\title{EMA Without the Lag: \\ Bias-Corrected Iterate Averaging Schemes}

\author[1]{Adam Block\thanks{Correspondence to \href{mailto:adam.block@columbia.eud}{adam.block@columbia.edu}.}}
\author[2]{Cyril Zhang\thanks{Work done while both authors were at Microsoft Research NYC.}}
\affil[1]{Columbia University}
\affil[2]{OpenAI}

\date{}

\begin{document}

\maketitle

\begin{abstract}
    Stochasticity in language model fine-tuning, often caused by the small batch sizes typically used in this regime, can destabilize training by introducing large oscillations in generation quality.  A popular approach to mitigating this instability is to take an Exponential moving average (EMA) of weights throughout training.  While EMA reduces stochasticity, thereby smoothing training, the introduction of bias from old iterates often creates a lag in optimization relative to vanilla training.  In this work, we propose the Bias-Corrected Exponential Moving Average (BEMA), a simple and practical augmentation of EMA that retains variance-reduction benefits while eliminating bias. BEMA is motivated by a simple theoretical model wherein we demonstrate provable acceleration of BEMA over both a standard EMA and vanilla training.  Through an extensive suite of experiments on Language Models, we show that BEMA leads to significantly improved convergence rates and final performance over both EMA and vanilla training in a variety of standard LM benchmarks, making BEMA a practical and theoretically motivated intervention for more stable and efficient fine-tuning.
\end{abstract}

\section{Introduction}\label{sec:intro}

\begin{lrbox}{\codebox}
    \begin{minipage}[c]{0.8\textwidth}
        \begin{lstlisting}[language=Python,frame=single,breaklines=true]
## Online weight update with BEMA given EMA factor beta_t
optimizer.step()
param_EMA.data = (1. - beta_t) * param_EMA.data + beta_t * param.data
#### Bias-correction update (BEMA) ####
alpha_t = beta_t**0.4 
param_BEMA.data = alpha_t * (param.data - param0.data) + param_EMA.data
        \end{lstlisting}
        \end{minipage}
\end{lrbox}

\begin{figure}[t]   
	\centering
	\subfigure[]{  
		\includegraphics[width=0.29\textwidth]{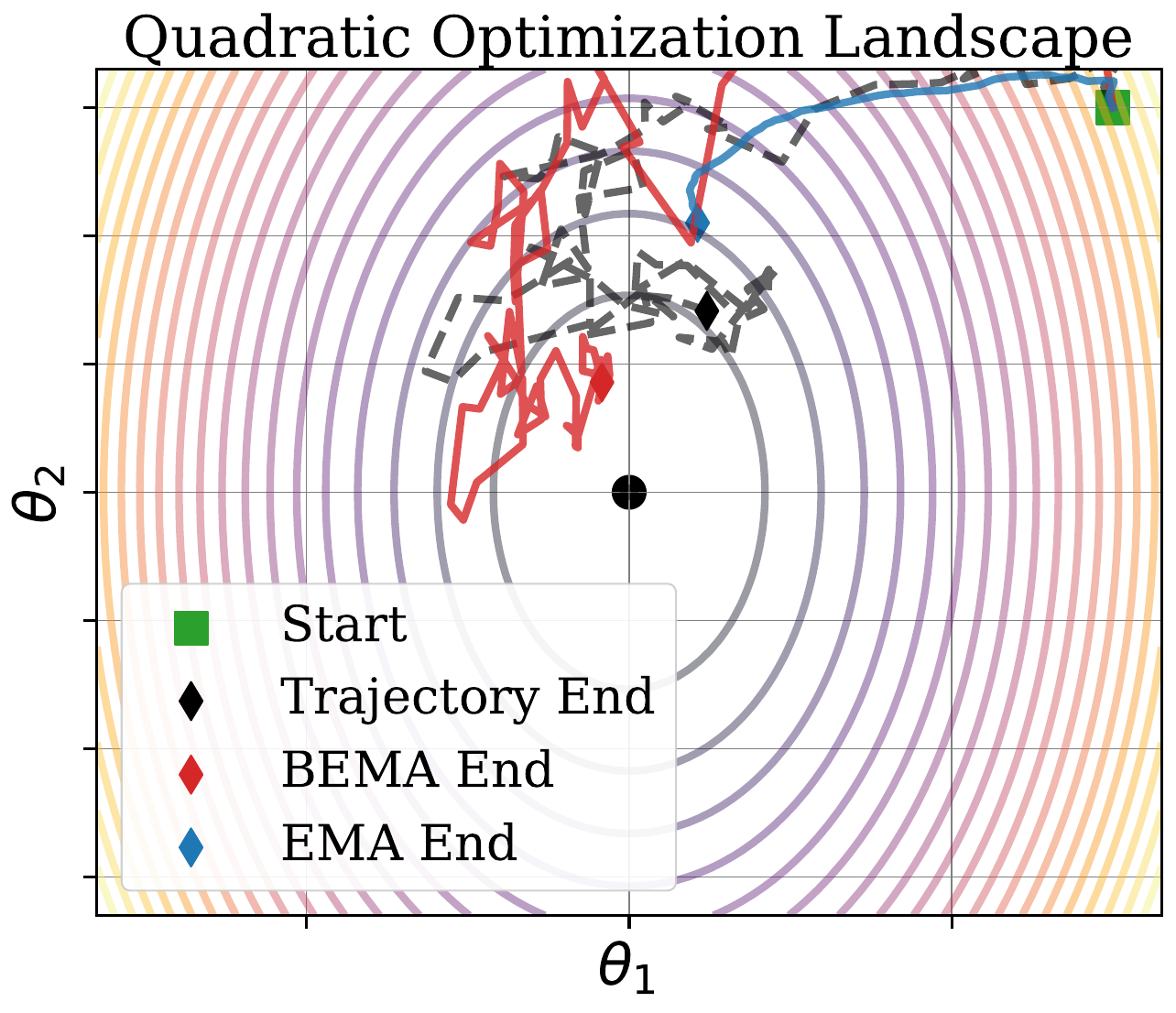}
		\label{fig1:sub1} 
	} \hfill \subfigure[]{ 
\includegraphics[width=0.29\textwidth]{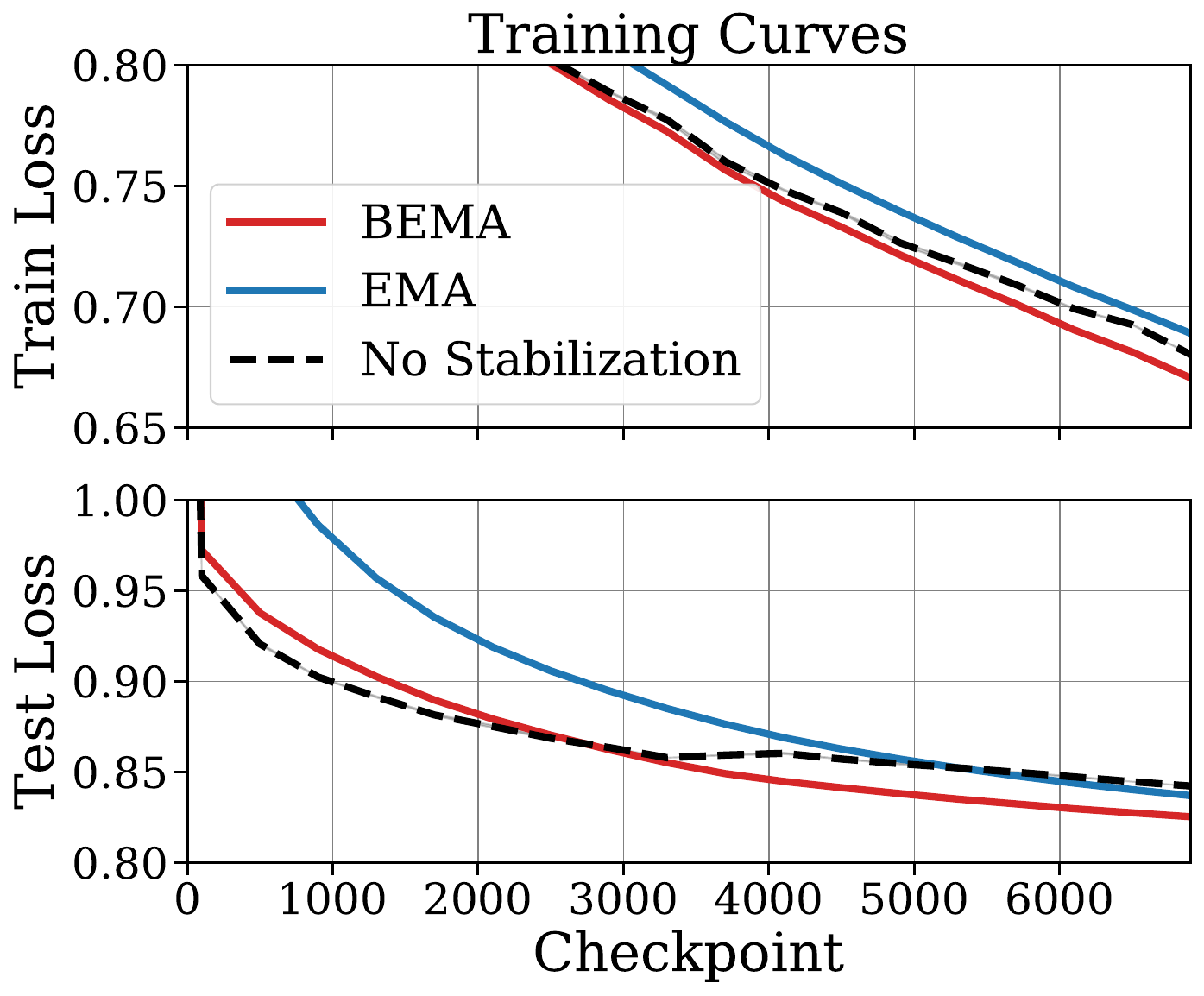} 
		\label{fig1:sub2}
	} \hfill \subfigure[]{ 
        \includegraphics[width=0.29\textwidth]{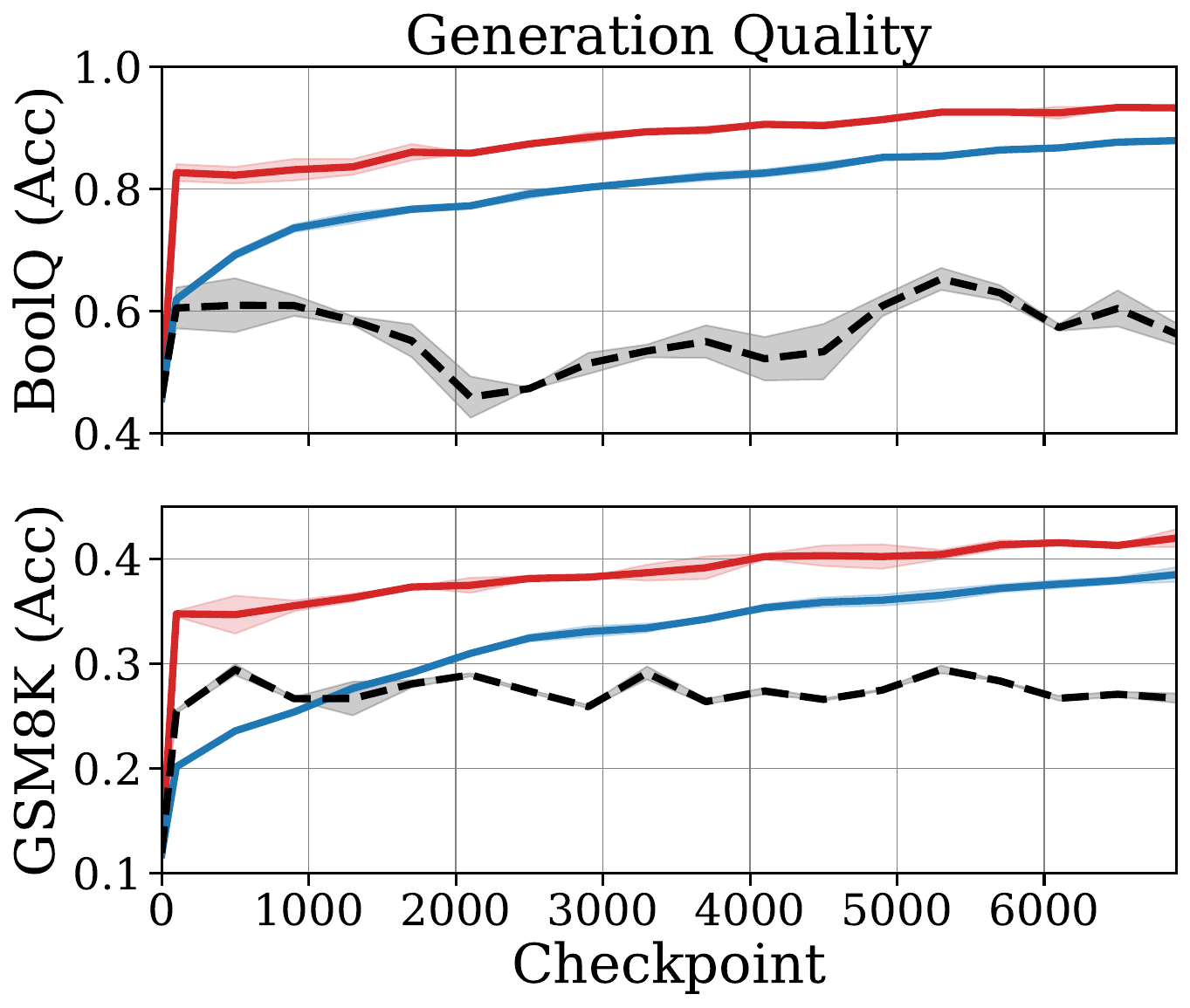} 
        \label{fig1:sub3}
    }
    \\
    \subfigure[]{
        \usebox{\codebox}
        \label{fig1:sub4}
    }
	\caption{\textbf{(a)} Example trajectory of stochastic quadratic optimization in two dimensions, stabilized both by \ema\  and \oumle; the vanilla trajectory does not converge to the minimum due to gradient variance, while \ema\  induces significant slowing down due to bias; only \bema\  converges to the minimum quickly. \textbf{(b)} Train and test loss curves and \textbf{(c)} \boolq~ and \gsmk~ benchmarks, for \qwen, without stabilization and with both \ema\  and \oumle.  While both \ema\  and \bema\  improve performance over vanilla training, \bema\  achieves better performance more quickly than \ema. \textbf{(d)} Snippet of Python code for implementing \oumle, demonstrating weight update, assuming EMA parameters $\beta_t$ are given.}  
	\label{fig:Fig-1} 
\end{figure}  

\footnote{The code used to run our evaluations can be found at \href{https://github.com/abblock/bema}{https://github.com/abblock/bema}.}With the increasing scale of Language Models (LMs), serious limitations on the quantity of new, high-quality data available for pre- and post-training have led to a renewed interest in understanding optimization and how best to use scarce data \citep{villalobos2024position,muennighoff2023scaling}.  Indeed, in regimes where the number of distinct, high-quality sequences of text is limited, e.g. in finetuning on corpora where data collection is expensive such as math \citep{liu2023improving,hendryckstest2021}, code \citep{austin2021program}, or specialized domain expertise \citep{lee2020biobert,chalkidis2020legal}, practitioners are often forced to use a small batch size in order to squeeze as much information out of the data as possible \citep{masters2018revisiting,rather2024breaking,zhang2024does}.  While small batch sizes allow for more gradient steps to be taken, they come at the cost of increased variance in the stochastic gradients, which can lead to instability.

Training instability is particularly pronounced in situations where a model is evaluated \emph{closed loop}, i.e. a model is rolled out by iterative application on its own outputs; such closed-loop rollout effects were originally observed in the context of imitation learning \citep{ross2010efficient,ross2011reduction,chang2023learning,blockbutterfly2024} and are a result of small learning errors in each step of a rollout being catastrophically amplified through repeated application.  LMs exhibit the same pathology because errors occurring at the \emph{token} level are repeatedly fed back into the model due to the autoregressive nature of generation; this connection between imitation learning and LMs has been explored extensively in the literature \citep{chang2023learning,block2023provable,blockbutterfly2024,foster2024behavior,rohatgi2025computational}.  In \citet{blockbutterfly2024}, the authors observed that error amplification in the context of closed-loop evaluation often results from stochasticity in the gradients, which they term \emph{Gradient Variance Amplification} (GVA), and can substantially degrade model performance even when cross-entropy loss is small; due to the many downstream problems that GVA can cause, \citet{blockbutterfly2024} recommends focusing on designing \emph{stabilizers} that mitigate these effects.  

The most empirically successful approach to stabilizization is \emph{iterate averaging}, wherein the training trajectory is postprocessed by applying a weighted average to the individual iterates in order to reduce variance, with the most popular such averaging scheme being an \emph{Exponential Moving Average} (\ema) \citep{ruppert1988efficient,polyak1992acceleration,izmailov2018averaging,sandler2023training,busbridge2023scale}.  In deep learning, \ema\  has seen great success both in stabilizing training \citep{blockbutterfly2024} and in improving the final performance of the model \citep{izmailov2018averaging}, but the variance reduction comes at the cost of introducing bias from earlier iterates, which empirically manifests as a \emph{lag} in the training trajectory: while the training curves of \ema\  are typically signficantly smoother than the optimization with no stabilization, they often converge more slowly.  This observation naturally leads to the following question:
\begin{quotation}
    \emph{Can we design a stabilizer achieving the benefits of \ema\  without the lag?}
\end{quotation}

We answer this question in the affirmative by introducing a new stabilizer, \bema\ (summarized in \Cref{alg:oumle} with sample Python code given in \Cref{fig1:sub4}), which achieves the best of both worlds.  We adopt a theoretical model (\Cref{sec:prelims,sec:theory}) inspired and justified by prior empirical work in deep learning and optimization \citep{duchi2011adaptive,gupta2018shampoo,zhang2019algorithmic,vyas2024soap} and derive \bema\ as the optimal stabilizer in this model.  We then discuss the practical implementation of \bema\  in \Cref{sec:implementation}, and \textbf{observe that it is a drop-in replacement for the commonly used \ema\  stabilizer, requiring changing only two lines of code}.
Finally, we evaluate \bema\  on a variety of Language Model (LM) finetuning tasks in \Cref{sec:empirics}, where we find that  \textbf{\bema\ significantly outperforms both vanilla training and \ema\  across a wide range of tasks}.  A brief survey of related work can be found in \Cref{subsec:related_work}, further empirical results can be found in \Cref{sec:app_further_empirical_results}, and all proofs are deferred to \Cref{app:theory}.

\section{Mathematical Preliminaries}\label{sec:prelims}

We are interested in the problem of stabilizing the training of language models (LMs) when the optimizer has a sufficiently small batch size so as to make gradient stochasticity a significant problem for closed-loop evaluation.\footnote{While the primary focus is LMs, we conjecture that our approach can be applied to other situations in which GVA presents a problem, such as in Imitation Learning \citep{blockbutterfly2024}.}  We formalize a language model as a conditional distribution $p_{\theta}(y | x)$ parameterized by some weight $\theta$, where $y \in \cV$ is the next \emph{token}, which is a member of the vocabulary $\cV$ and $x \in \cV^\ast$ is a \emph{prompt} or \emph{context} consisting of a sequence of tokens.  In this paper, we are primarily interested in \emph{Supervised Fine Tuning} (SFT), wherein we are given a dataset $\cD$ consisting of sequences and we attempt to maximize the log likelihood of a given sequence, i.e., minimize $-\ee_{(x,y) \sim \cD}\left[ \log p_\theta(y | x) \right]$. Due to the high dimension of the weights $\theta$, this optimization is typically accomplished via stochastic local search techniques.  Because we are interested in a model's performance on closed-loop rollouts (via autoregressive generation), the oscillations in model performance throughout training observed in \citet{blockbutterfly2024} pose a significant problem, which motivates the need for stabilizing the optimization process, which we now discuss.

In order to formalize the notion of a \emph{stabilizer}, we consider the classical setting of stochastic optimization \citep{robbins1951stochastic,ruppert1988efficient,polyak1992acceleration,nesterov2013introductory}, where we are given a function $f: \rr^d \to \rr$ taking its minimum at $0$ and access to a stochastic gradient oracle that returns a noisy gradient $\nabla f(\theta - \mustar) + \xi$ when queried at a point $\theta \in \rr^d$ for fixed minimum $\mustar \in \rr^d$, where $\xi$ is some random vector representing the noise.  We focus on the simplest algorithm for this problem, stochastic gradient descent (SGD), which updates the parameter $\theta$ according to the rule:
\begin{align}
    \theta_{t+1} &= \theta_t - \eta_t \left( \nabla f(\theta_t - \mustar) + \xi_t \right). 
\end{align} 

To aid analytical tractability, we will follow \citet{mandt2015continuous,li2017stochastic,malladi2022sdes} and consider the continuous time limit of this process assuming that $\xi$ is mean zero and has finite second moment, which is given by the stochastic differential equation (SDE):
\begin{align}
    d \theta_t &= -\nabla f(\theta_t - \mustar) ~dt + \sqrt{\eta} \cdot \bSigma ~ dW_t, \qquad \theta_0 \in \rr^d, \label{eq:sde_general}
\end{align}
where $\Sigma \in \rr^{d \times d}$ can be determined by the covariance of the noise in the stochastic gradient oracle, $\eta$ is the (constant) scale on the learning rate, and $W_t$ is a standard Brownian motion in $\rr^d$.\footnote{We will always assume that at least a weak solution to \eqref{eq:sde_general} exists and is unique, which is certainly the case for the OU process on which we focus.  For more details on SDEs, see \citet{le2016brownian}.}  We will adopt the perspective of stochastic optimization as a statistical parameter estimation problem, where the minimizer we seek is the parameter we wish to estimate \citep{robbins1951stochastic,polyak1992acceleration,ruppert1988efficient,nesterov2013introductory} and suppose that the stabilizer is some algorithm that is given an optimization trajectory and aims to return an estimate of the minimum.  More formally, our goal is the following.
\begin{quotation}
    \emph{For a fixed, finite horizon $T > 0$, given access to the trajectory $(\theta_t)_{0 \leq t \leq T}$, how can we best estimate $\mustar$ in a memory and computationally efficient way?}
\end{quotation}

\begin{wrapfigure}[22]{r}{0.4\textwidth}
    \centering
    \vspace{-0.5cm}
    \includegraphics[width=0.4\textwidth]{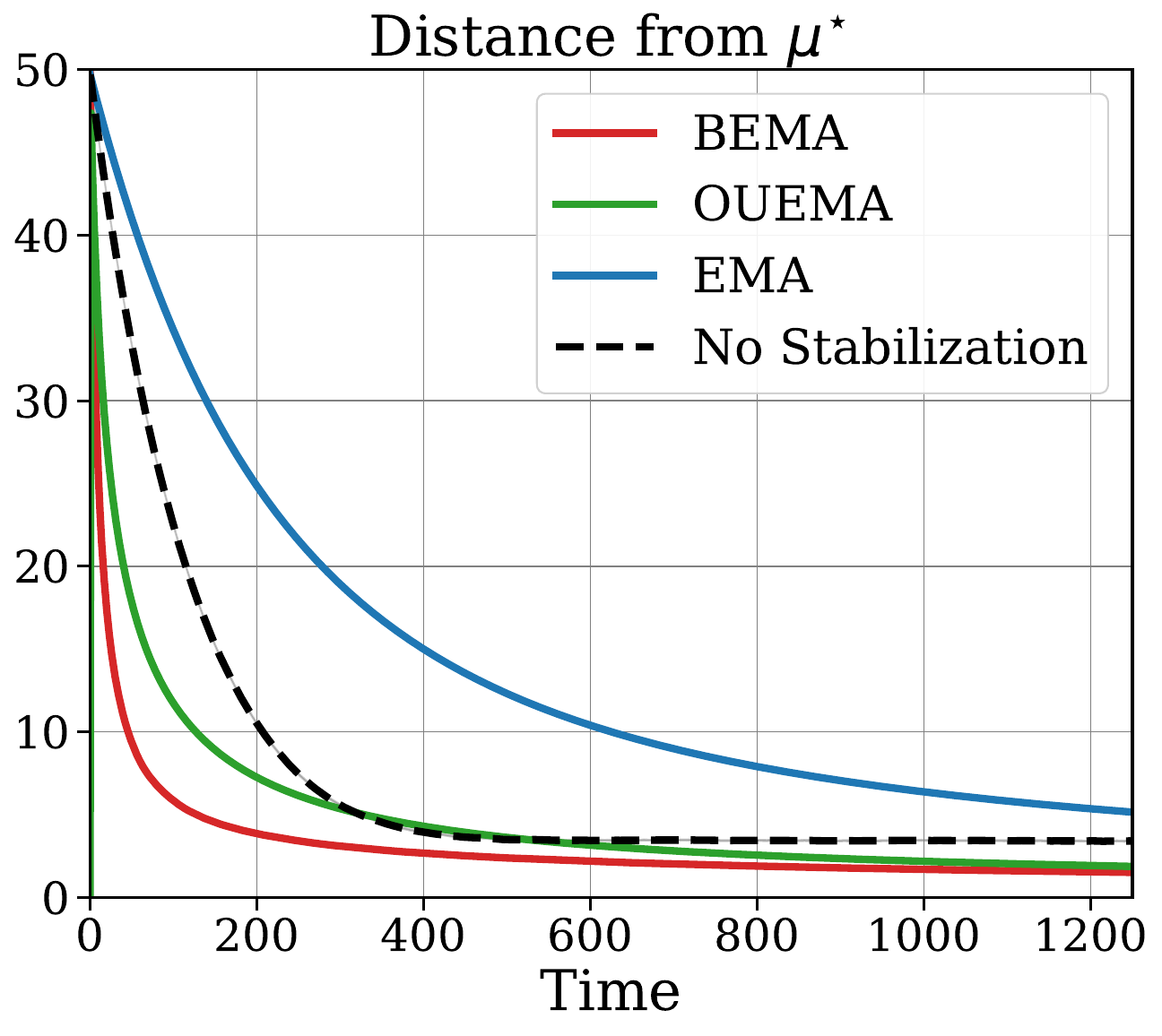}
    \caption{Expected distance from the minimum $\mustar$ in stochastic quadratic optimization with $d=20$ without stabilization, with \ema, \oumle, and \ouema.  \ema\  slows down the optimization process significantly, while \bema\  and \ouema\  converge significantly more quickly.}
    \label{fig:Fig-2}
\end{wrapfigure}
The ultimate goal is to construct an algorithm that improves optimization in practice, particularly in the context of finetuning language models; due to the size of LMs, the memory efficiency of the final estimator is crucial.  While LMs themselves are certainly not convex functions with respect to their parameters $\theta$, recent work has demonstrated that they can be locally well-approximated by a quadratic function and optimization insights arising from this regime often carry over to the more practical (and less analytically tractable) non-convex case of modern-day transformers, especially in the finetuning regime \citep{jacot2018neural,gupta2018shampoo,cohen2021gradient,malladi2023kernel,vyas2024soap}.  Thus, to further simplify our problem and permit us to develop a concrete, practical algorithm, we will focus our theory on the noisy quadratic model \citep{zhang2019algorithmic}, where $f(\theta) = \frac 12 \theta^\top \bA \theta$ for some positive definite matrix $\bA \in \rr^{d \times d}$, which can be interpreted as the Hessian of the loss of the LM at some fixed $\theta_0$; to reiterate, while this assumption does not hold for real LMs, it provides a useful testbed from which to develop intuition and algorithmic interventions \citep{duchi2011adaptive,gupta2018shampoo,vyas2024soap}.  

It has long been known that the continuous time limit of SGD applied to a quadratic loss is the Ornstein-Uhlenbeck (OU) process \citep{mandt2015continuous}, where \eqref{eq:sde_general} becomes:
\begin{align}
    d \theta_t &= \bA (\mustar - \theta_t) ~dt + \sqrt{\eta} \cdot \bSigma ~ dW_t, \qquad \theta_0 \in \rr^d, \label{eq:sde}
\end{align}
where we assume that $\bA, \bSigma \in \rr^d$ are symmetric positive definite matrices.  It is standard that the OU process admits a simple closed form, given in \Cref{app:theory}, that we use extensively in our analysis. 

\paragraph{Notation.}  We will use $\norm{\cdot}$ to denote the Euclidean norm in $\rr^d$ and $\normop{\cdot}$ and $\normf{\cdot}$ to denote the operator and Frobenius norms of matrices, respectively.  We will let $\eye$ be the identity matrix and reserve bold capital letters for matrices; the trace of a matrix is denoted by $\trace(\cdot)$.  Given random vectors $a, b$, we denote $\Cov(a, b) = \ee\left[ a b^\top \right] - \ee[a] \ee\left[ b^\top \right]$ the covariance matrix, and abbreviate $\Cov(a) = \Cov(a,a)$; furthermore, we let $\Var(a) = \trace(\Cov(a)) = \ee\left[ \norm{a - \ee[a]}^2 \right]$.

\section{Optimal Stabilization in Stochastic Quadratic Optimization}\label{sec:theory}

Above, we formalized stabilization as the statistical estimation problem of estimating $\mustar$ given access to a single trajectory $(\theta_t)_{0 \leq t \leq T}$ of the OU process defined in \eqref{eq:sde}.  For the sake of simplicity, we will in this section assume that $\bSigma = \sigma^2 \eye$ and defer the general case (as well as all proofs) to \Cref{app:theory}.  Before we proceed, we first present a standard lower bound on the expected squared error of any estimator $\hat{\mustar}$ using the Cramer-Rao and van Trees inequalities \citep{lehmann2006theory}; asymptotic versions of this standard bound can be found in \citet{liptser2013statistics1,liptser2013statistics2,kutoyants2013statistical}.
\begin{proposition}\label{prop:lower_boundbody}
    For any fixed $T < \infty$, let $(\theta_t)_{0 \leq t \leq T}$ be a trajectory from \eqref{eq:ou} with $\bSigma = \sigma^2 \eye$ and suppose that $\muhat$ is an estimator of $\mustar$ measurable with respect to the filtration generated by $(\theta_t)_{0 \leq t \leq T}$.  Suppose further that $\muhat$ is unbiased, i.e. $\ee[\muhat] = \mustar$.  Then it holds that
    \begin{align}
        \ee\left[ \norm{\muhat - \mustar}^2 \right] \geq \frac{\eta \sigma^2 \cdot \trace\left( \bA^{-2} \right)}{T}. \label{eq:cramer_rao}
    \end{align}
    More generally, if the bias of $\muhat$ is a contraction, i.e., the map $\mustar \mapsto \ee_{\mustar}\left[ \muhat - \mustar \right]$ is $L$-Lipschitz for some $L < 1$, then \eqref{eq:cramer_rao} holds with a prefactor of $(1 - L)^2$.
\end{proposition}
While the asymptotic performance (as $T \uparrow \infty$) of a number of standard estimators is well understood \citep{kutoyants2013statistical}, in this work we are interested in what occurs for finite $T$, which is the regime of interest in practice.  Perhaps the simplest approach to estimating $\mustar$ is that adopted by vanilla optimization: simply take the final iterate $\theta_T$ as the desired estimate.  In this case, we can precisely compute the expected squared error of this estimate, which is given in the following proposition.
\begin{proposition}\label{prop:vanilla}
    Let $(\theta_t)_{0 \leq t \leq T}$ be a trajectory from \eqref{eq:ou} with $\bSigma = \sigma^2 \eye$.  Then it holds that
    \begin{align}
        \ee\left[ \norm{\theta_T - \mustar}^2 \right] = \norm{ e^{- \bA T}  \left(\mustar - \theta_0\right)}^2 + \eta \sigma^2 \cdot \trace\left( \bA^{-1} \left( \eye - e^{-2 \bA T} \right) \right) \label{eq:vanilla_bound}
    \end{align}
\end{proposition}
While the last iterate estimator is attractive in its simplicity, and the first term in \eqref{eq:vanilla_bound} decays exponentially quickly, it leaves much to be desired because it is not consistent, i.e., $\theta_T \not\to \mustar$ even as $T \uparrow \infty$, unless $\eta \downarrow 0$.  This is a simple example of the well-understood phenomenon in stochastic optimization that motivates learning rate decay.  Absent computational constraints, it may well be advisable to train at a small (or aggressively decayed) learning rate; unfortunately, the number of optimizer steps required to reach time $T$ in the discrete approximation scales as $ T / \eta$, which quickly becomes prohibitive for small $\eta$.  Thus, practitioners often wish to train at as high a learning rate as possible in order to accelerate convergence \citep{smith2019super,loshchilov2022sgdr}, which helps explain why some degree of trajectory stabilization has become commonplace.

The most common approach to stabilizing training in modern deep learning, beyond learning rate decay, is to apply iterate averaging \citep{izmailov2018averaging,blockbutterfly2024,busbridge2023scale}, specifically an exponential moving average (EMA) of the model parameters \citep{ruppert1988efficient,polyak1992acceleration}, which is defined as $\muema_{t} = (1 - \alpha_t) \muema_{t} + \alpha_t \theta_t$ with $\thetaema_{t} = \theta_0$ for some sequence of weights $\alpha_t \in (0, 1)$.  While different choices for $\alpha_t$ are possible and discussed in the sequel, for the purpose of theory, we will consider $\alpha_t = t^{-1}$, which correpsonds to $\muema_{t}$ being a flat average of the iterates $\theta_0, \ldots, \theta_t$, or, in continuous time:
\begin{align}
    \muema_t = \frac{1}{t} \int_0^t \theta_s ~ds. \label{eq:ema}
\end{align}
A key advantage of $\muema_T$ is that it is memory-efficient to compute in a streaming fashion; indeed, in order to return $\muema_T$, a practitioner need only keep $\muema_t$ in memory at any given time, along with the current iterate $\theta_t$. Once again, the simplicity of the OU process allows us to provide tight bounds on the expected squared error of this estimate.
\begin{proposition}\label{prop:ema_bound}
    Let $\theta_t$ be the solution to \eqref{eq:sde} with $\Sigma = \sigma^2 \eye$ and let $\muema$ be the estimator given in \eqref{eq:ema}.  Then
    \begin{align}
        \ee\left[ \norm{\muema_T - \mustar}^2\right] \leq  \frac{\eta \sigma^2 \cdot\trace\left( \bA^{-2} \right)}{T} +  \frac{\normop{\bA^{-1}}^2 \norm{\mustar - \theta_0}^2}{T^2}. \label{eq:ema_bound}
    \end{align}
    Moreover,if $T \leq \frac{c}{2 \lambdamax(\bA)}$ for some constant $0 < c < 1$ then it holds that
    \begin{align}
        \ee\left[ \norm{\muema_T - \mustar}^2 \right] \geq \left( 1 - c \right)^2 \norm{\mustar - \theta_0}^2.
    \end{align}
\end{proposition}
Comparing the upper bound in \Cref{prop:ema_bound} to \Cref{prop:lower_boundbody} implies that $\muema$ is \emph{asymptotically} optimal as $T$ grows, but the lower bound demonstrates that the higher order bias term in \eqref{eq:ema_bound} is problematic for small $T$ when $\theta_0$ and $\mustar$ are not close to each other.  In the context of optimization, it is reasonable to assume that $\theta_0$ and $\mustar$ are far apart (otherwise minimal optimization would be required), suggesting that $\muema$ suffers from significant bias when $\lambdamin(\bA) T \lesssim 1$.  This bias manifests itself as \emph{lag} in the optimization process, which explains why the \ema\  curve in \Cref{fig:Fig-2} decreases more slowly than the vanilla optimization without stabilization.

One approach to improving upon $\muema$ is to debias the trajectory $\theta_t$ in a \emph{pointwise} manner, i.e., introducing a new trajectory $\thetabar_t$ such that for each $t$, $\ee_{\mustar}\left[ \thetabar_t \right] = \mustar$; then we could apply iterate averaging to the augmented trajectory $\thetabar_t$ and hopefully obtain a better estimator.  The following result details the benefit of this approach.\footnote{See \Cref{subsec:additional_results} for a remark on why we apply a flat average here as opposed to a weighted average.}
\begin{theorem}\label{thm:ouema}
    Let $(\theta_t)_{0 \leq t \leq T}$ be a trajectory from \eqref{eq:ou} with $\bSigma = \sigma^2 \eye$ and for some $\tau \in (0, T)$, let
    \begin{align}
        \muouema_T = \frac 1{T - \tau} \int_\tau^T \thetabar_t d t \quad \text{with} \quad \thetabar_t = \left( \eye - e^{- \bA t} \right)^{-1} \left( \theta_t - e^{- \bA T}  \theta_0\right). \label{eq:debiased_trajectory}
    \end{align}
    Then it holds for all $t$ that $\ee_{\mustar}\left[ \thetabar_t \right] = \ee_{\mustar}\left[ \muouema_T \right] = \mustar$ and
    \begin{align}
        \ee\left[ \norm{\muouema_T - \mustar}^2 \right] \leq \frac{\eta \sigma^2 \cdot\trace\left( \bA^{-2} \right)}{T \left[ \left( 1 - e^{-\lambdamin(\bA) \tau} \right)\left( 1 - \nicefrac \tau T \right) \right]^2}. \label{eq:ouema_bound}
    \end{align}
\end{theorem}
Like $\muema$, the estimator $\muouema$ can be implemented in a streaming fashion, as $\thetabar_t$ is easy to compute given $\theta_t$; indeed, a practitioner need only hold in memory $\theta_0$, $\thetabar_t$, and $\muouema_t$ at any given time.  Furthermore, in the regime where $\lambdamax(\bA) T \lesssim 1$, when $\norm{\mustar - \theta_0}$ is large relative to the conditioning of $\bA$, it holds that $\muouema$ improves upon $\muema$; indeed, this can lead to accelerated optimization of \ouema\  relative to \ema, as can be seen in \Cref{fig:Fig-2}.

While acceleration for small $T$ is a key benefit of $\muouema$, this estimator is not asymptotically optimal, as can be seen by comparing the upper bound in \eqref{eq:ouema_bound} to the lower bound in \Cref{prop:lower_boundbody}.  Thus, we instead propose $\mumle$, the \emph{Maximum Likelihood Estimator} of $\mustar$ as a candidate stabilizing algorithm (the practical instantiation of which is our main proposed intervention, \bema).
\begin{theorem}\label{thm:mle_body}
    Let $\theta_t$ be the solution to \eqref{eq:sde} with $\bSigma = \sigma^2 \eye$.  Then
    \begin{align}
        \mumle_T = \frac{\bA^{-1}}{T} \left( \theta_T - \theta_0 \right) + \frac 1T \int_0^T \theta_t ~dt \label{eq:oumle}
    \end{align}
    is the maximum likelihood estimator of $\mustar$ given the trajectory $(\theta_t)_{0 \leq t \leq T}$.  Furthermore, it holds that $\mumle_T$ is unbiased and
    \begin{align}
        \ee\left[ \norm{\mumle - \mustar}^2 \right] = \frac{\sigma^2 \eta \cdot  \trace\left( \bA^{-2} \right)}{T}. \label{eq:mle_bound_body}
    \end{align}
\end{theorem}

\begin{algorithm}[t]
    \caption{Bias corrected Exponential Moving Average (\oumle)}\label{alg:oumle}
    \textbf{Input:} Trajectory $\left\{ \theta_t | t \in [T] \right\}$, \ema\  power $\kappa$, bias power $\eta$, multiplier $\gamma$, lag $\rho$, burn in time $\tau$, frequency $\phi$. \\
    \textbf{Set} $\muhat \gets \theta_0$ and $\muema \gets \theta_0$.\\
    \For{$t = 1, \ldots, T$}{
        \If{$t \leq \tau$}{
            \textbf{Update} $\muhat \gets \theta_t$, $\theta_0 \gets \theta_t$, and $\muema \gets \theta_t$. \algcomment{No bias correction before $\tau$}
    }
    \ElseIf{$(t - \tau) \mod \phi \neq 0$}{
        \textbf{Continue} \algcomment{Only update every $\phi$ steps}
        }
    \Else{
        \textbf{Define} \color{darkred}{$\alpha_t \gets  \left( \rho + \gamma t \right)^{- \eta}$} \color{black} and $\beta_t \gets \left( \rho + \gamma t \right)^{- \kappa}$ \algcomment{Define weights for EMA and bias correction.}\\
        \textbf{Update} $\muema \gets (1 - \beta_t) \cdot \muema + \beta_t \cdot \theta_t$ \algcomment{Update EMA.}\\
        \color{darkred}\textbf{Update} $\muhat \gets \alpha_t \left( \theta_t - \theta_0 \right) + \muema$ 
    }
    }
    \textbf{Return} $\muhat$.
\end{algorithm}

The proof of \Cref{thm:mle_body} is deferred to \Cref{app:theory} and crucially relies on Girsanov's theorem to express the log-likelihood of the process $\theta_t$ as a function of $\mustar$.  Note that like $\muouema$, the new estimator $\mumle$ can be computed in a streaming fashion, requiring keeping in memory only two copies of the parameter at a time in addition to $\theta_t$.  Moreover, $\mumle$ is optimal even in finite time, as can be seen by comparing the first term of \eqref{eq:mle_bound_body} to \Cref{prop:lower_boundbody}; indeed, we can even show that $\mumle - \mustar$ is distributed precisely as a Gaussian with covariance $\nicefrac{\eta \sigma^2}{T} \cdot \bA^{-2}$ (\Cref{cor:mumle_distribution}).  In the regime where $T \cdot \lambdamax(\bA) \lesssim 1$, we expect $\mumle$ to significantly improve upon $\muema$ whenever $\mustar$ and $\theta_0$ are far apart.  

While thus far we have assumed that $\bA$ is known, this assumption is unlikely to hold in the context of LM training, as $\bA$ corresponds to the Hessian of the loss function, which is expensive to compute.  In \Cref{app:theory}, we prove \Cref{thm:mumle_performance}, which controls the extent to which performance degrades when plugging in an incorrect $\bAtil$ into \eqref{eq:oumle} in the place of $\bA$.  With respect to asymptotic in $T$ performance, the choice of $\bAtil$ is irrelevant, as it does not affect the leading term in the error bound, and the higher order terms in the resulting bound suggest that as long $\bAtil$ is not too far from $\bA$ and not too poorly conditioned, the resulting estimator will still perform well.  In practice, we find that taking $\bAtil = \alpha \eye$ for some $\alpha > 0$ works well.

To complement the above theory, we visualize trajectories of a base OU process, \ema, and \bema\ (the practical instantiation of $\mumle$)  in \Cref{fig1:sub1} in two dimensions with $\mustar = 0$, where we see that the base process jumps around due to the relatively large value of $\eta$, while \ema\  converges slowly because $\theta_0$ is relatively far from $\mustar$; on the other hand, \bema\  converges significantly more quickly than the other two, in line with theory.  A more involved comparison can also be seen in \Cref{fig:Fig-2}, where we compare the expected squared distance between different estimators and $\mustar$ in stochastic quadratic optimization with $d=20$ and $\bA = \eye$.  We see that for this setting, \ema\  significantly delays convergence, while \ouema\  and \bema\  converge significantly more quickly, with \bema\  being the fastest, as predicted by the continuous time theory once again.

\section{Practical Considerations: Introducing \bema}\label{sec:implementation}

In the previous section, we established rigorous theoretical guarantees for a number of stabilizers under the assumption that $\theta_t$ is evolving as if it were an OU process.  While these results can provide practical insight \citep{gupta2018shampoo,vyas2024soap,cohen2021gradient,blockbutterfly2024}, in reality it is obviously not the case that when finetuning LMs, the loss landscape is exactly quadratic.  Thus, in this section we describe our recommended intervention, \bema.  \textbf{We emphasize that \bema\ only requires a two line change to existing \ema\ implementations, making it an easy, drop-in replacement.}

We are interested in finetuning LMs for two reasons: first, the combination of the desire to take as many gradient steps as possible with limited high-quality data necessitate using small batch sizes, which require stabilization; second, recent empirical work \citep{malladi2023kernel} has suggested that post-training LMs can be well-approximated by a kernel (i.e. linear) setting due to the feature learning ocurring during pre-training, which suggests that the strong modelling assumptions made in \Cref{sec:theory} may approximately hold.  Even so, we must make several modifications to the theoretical estimator proposed in \eqref{eq:oumle}; pseudocode for the practical implementation, which we call \oumle, is given in \Cref{alg:oumle} (differences from standard \ema\  are colored in \color{darkred}red\color{black}).  A summary of default hyperparameters is given in \Cref{tab:stabilizer_hyperparams} in \Cref{app:empirical_setup}.

First, a key difference between the theory of \Cref{sec:theory} and practice is that we must apply \eqref{eq:oumle} in discrete time.  This manifests in two ways: (a) we need to choose a value for the time $T$; (b) we need to implement the average comprising the second term of the estimator.  Both issues already arise in existing implementations of \ema\ for deep learning \citep{izmailov2018averaging,busbridge2023scale,blockbutterfly2024} and so we draw on popular, pre-existing solutions.  Thus, instead of taking a flat average, we run an EMA with a constant that decays polynomially in the number of steps, with exponent $\kappa \in (0,1)$; note that $\kappa = 1$ leads to a flat average, but practitioners have found that $\kappa < 1$ performs significantly better \citep{Karras2023AnalyzingAI,Lee2024SlowAS,Li2024SwitchEA,lucidrains_emapytorch}.

\begin{figure}[t]   
	\centering
    \subfigure[]{  
		\includegraphics[width=0.29\textwidth]{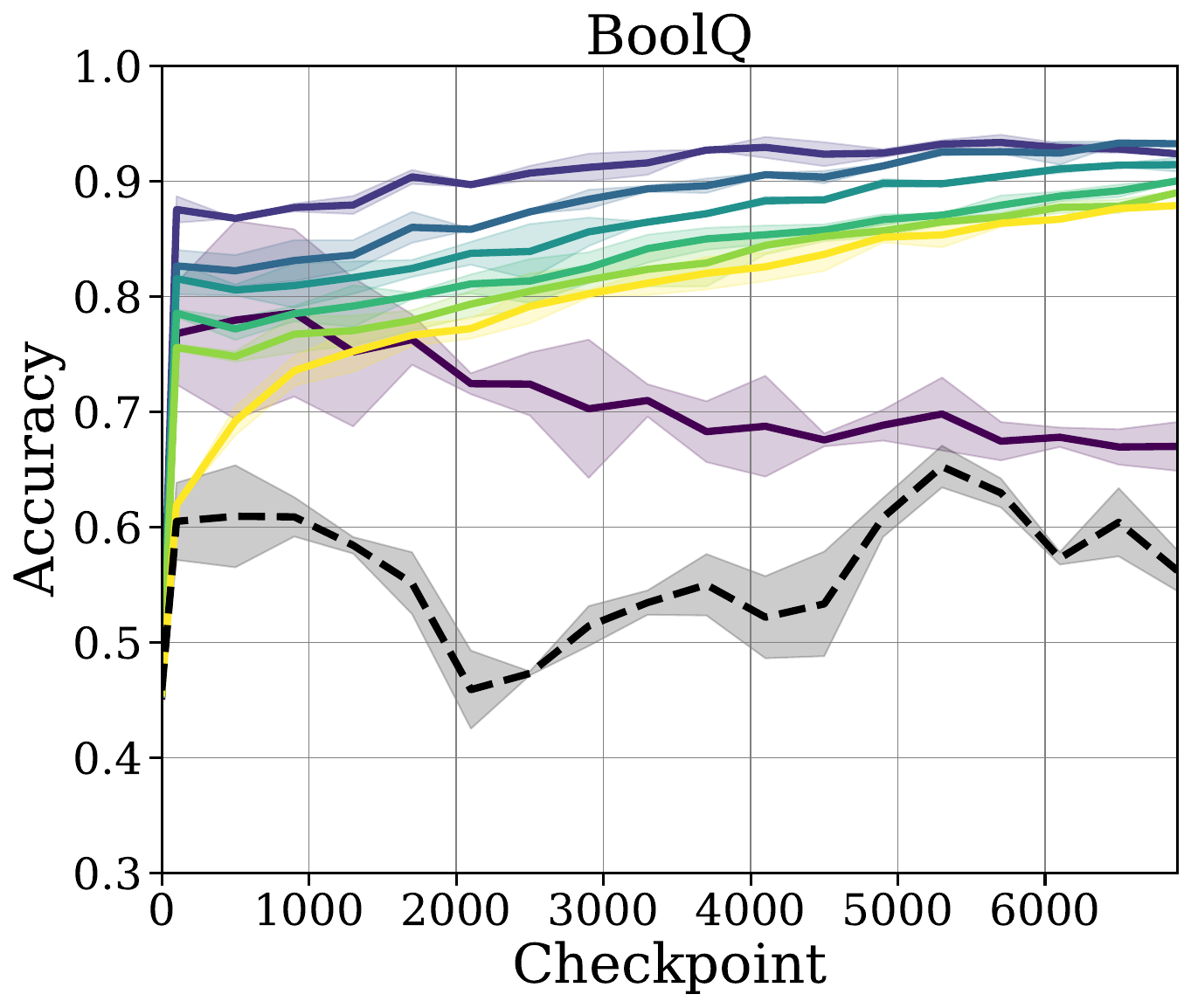}
		\label{fig3:sub1} 
	} \hfill \subfigure[]{ 
\includegraphics[width=0.29\textwidth]{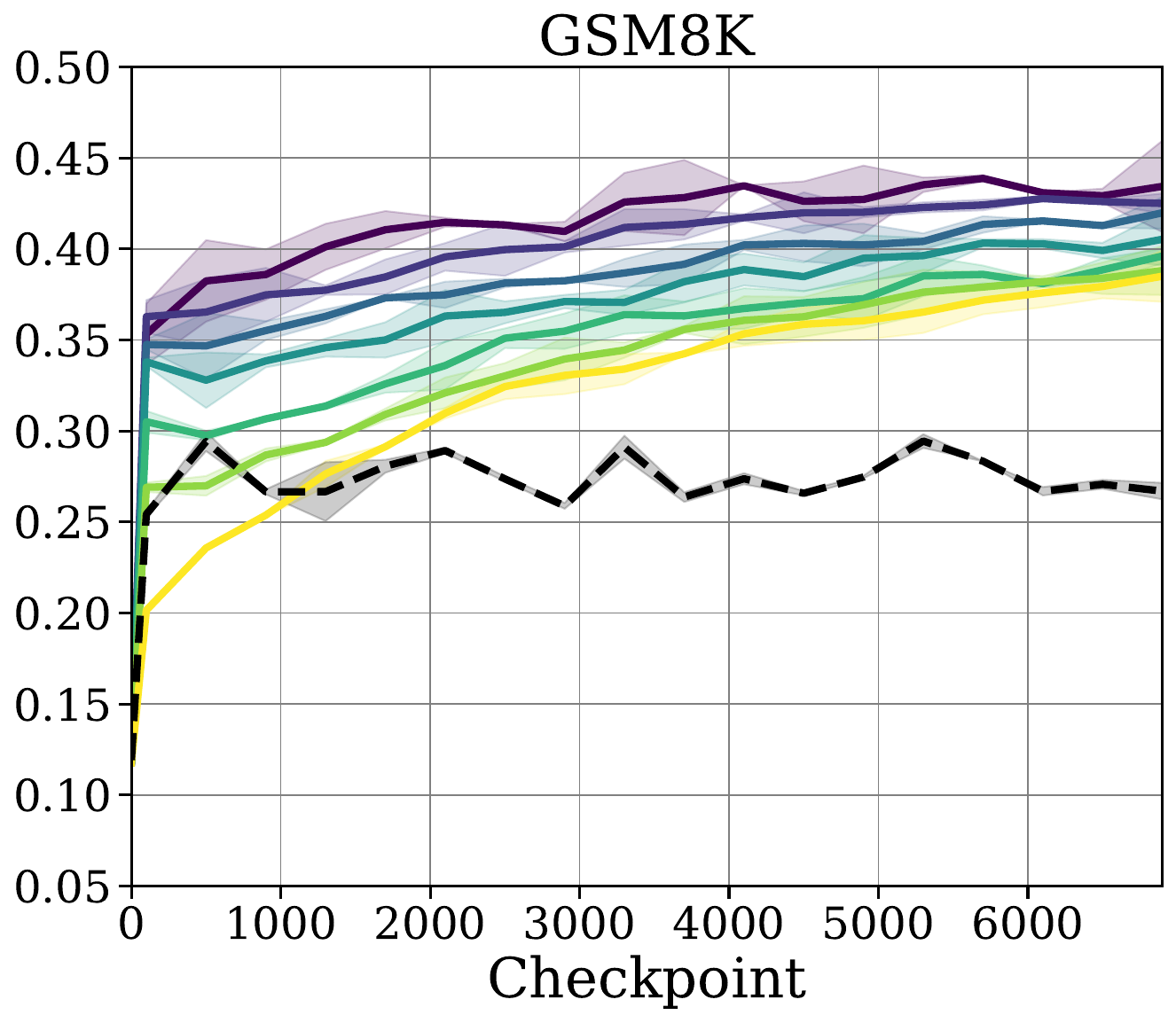} 
		\label{fig3:sub2}
	} \hfill \subfigure[]{ 
        \includegraphics[width=0.29\textwidth]{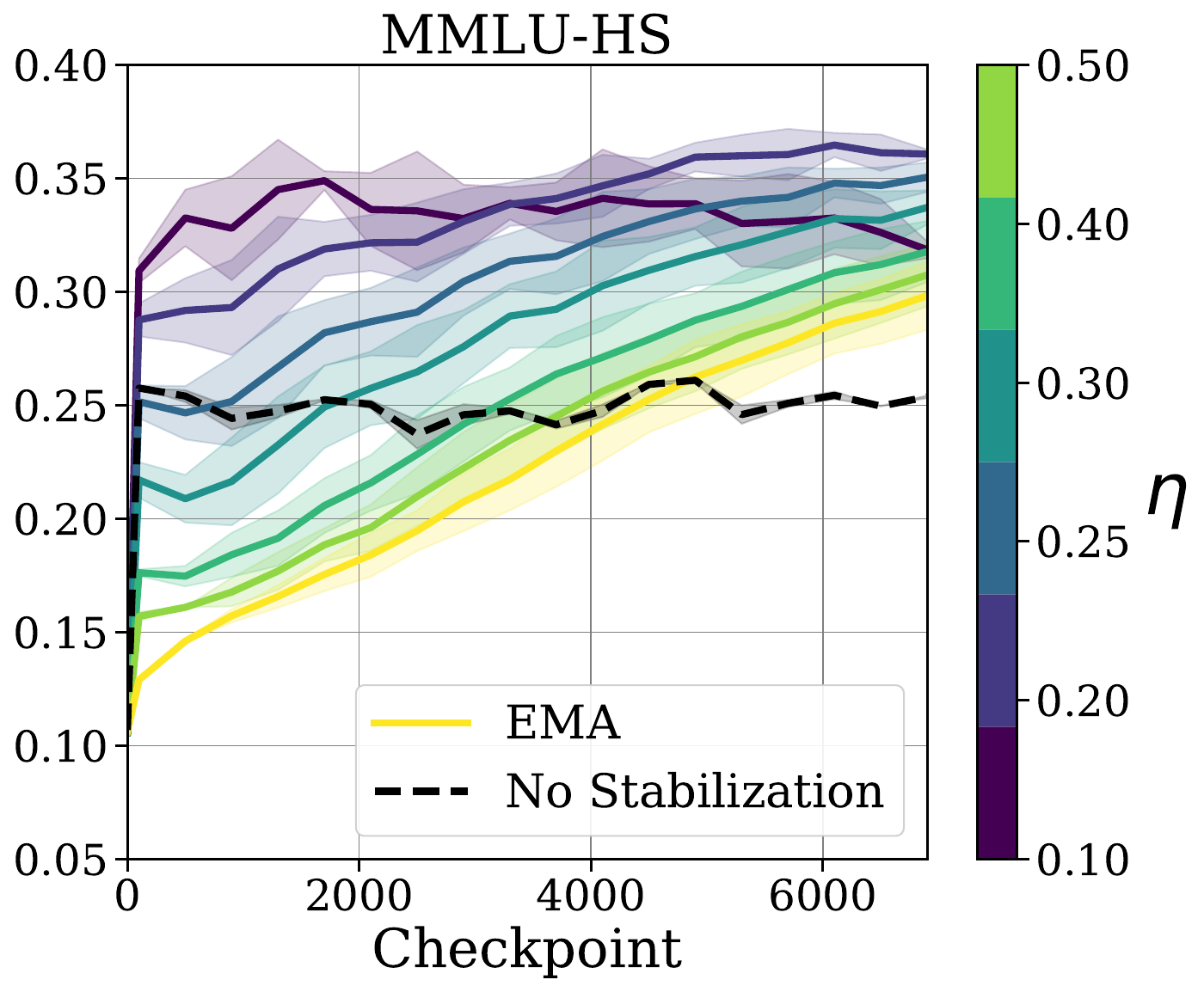} 
        \label{fig3:sub3}
    }
	\caption{
        Effect of varying the $\eta$ hyperparameter in \bema\  while finetuning \qwen~ on \boolq~ \textbf{(a)}, \gsmk~ \textbf{(b)}, and \mmluhs~ \textbf{(c)}.  In general, as $\eta$ decreases, the \bema\  intervention over vanilla \ema\  gets stronger, leading to better performance.  For too strong interventions ($\eta = 0.1$ in \boolq), performance collapses, likely as a result of the failure of the quadratic approximation.  In all cases, \bema\   with $\eta= 0.2$ outperforms vanilla \ema.
    }  
	\label{fig:Fig-3} 
\end{figure}

Second, in order to respect the nonconvexity of the loss landscape and in line with empirical best practice, we allow a burn-in time $\tau$ before applying stabilization (either \ema\ or \bema), although we find that $\tau = 0$ works best in practice when finetuning models.  Third, while the iterate average can be quickly computed in a streaming fashion, the computational cost of updating \bema\ at every step (which requires copying of model weights from one device to another in the likely event that CPU offloading is used to store $\theta_0$) can slow down training with respect to wall clock time.  Thus we introduce a parameter $\phi$, governing the frequency with which we update our stabilizer; ideally $\phi$ is set so as to be large enough to provide computational savings but small enough to ensure local convexity of the loss landscape.  We find that $\phi = 400$ works well in practice.

Finally, a practical consideration unique to \bema\ is the choice of $\bA$.  While in theory it would be natural to treat this as a nuisance parameter to be estimated and plugged in,\footnote{Normally one uses an orthogonalized approach to such plug-in estimators \citep{chernozhukov2018double}, but in our setting $\mumle$ is already orthogonalized so we could simply plug in an estimate.} na{\"i}vely doing this is infeasible: because $\theta_t \in \rr^d$ for some large $d$, the $d^2$-dimensional $\bA$ would likely not fit in memory.  Practical adaptive and preconditioned optimizers have taken a variety of approaches to this problem, and integrating these into \bema\  is an interesting direction for future work, but we find empirically that taking $\bA = \alpha_t \cdot \eye$ suffices to provide good performance in practice for some time-dependent scaling factor $\alpha_t$ akin to the $\beta_t$ used in \ema\ and discussed above.  We set $\alpha_t = \left( \rho + \gamma t \right)^{- \eta}$, where $\rho$ is a lag term, $\gamma$ is a multiplier, and $\eta$ determines the rate of decay: smaller $\eta$ leads to a stronger intervention of our algorithm relative to \ema.\footnote{In the case that $\eta = \infty$, we recover the standard \ema\  stabilizer.}  Combining these considerations yields our proposed algorithm, \oumle.

\section{Finetuning Language Models with \oumle}\label{sec:empirics}

\begin{figure}[t]   
	\centering
    \subfigure[]{  
		\includegraphics[width=0.29\textwidth]{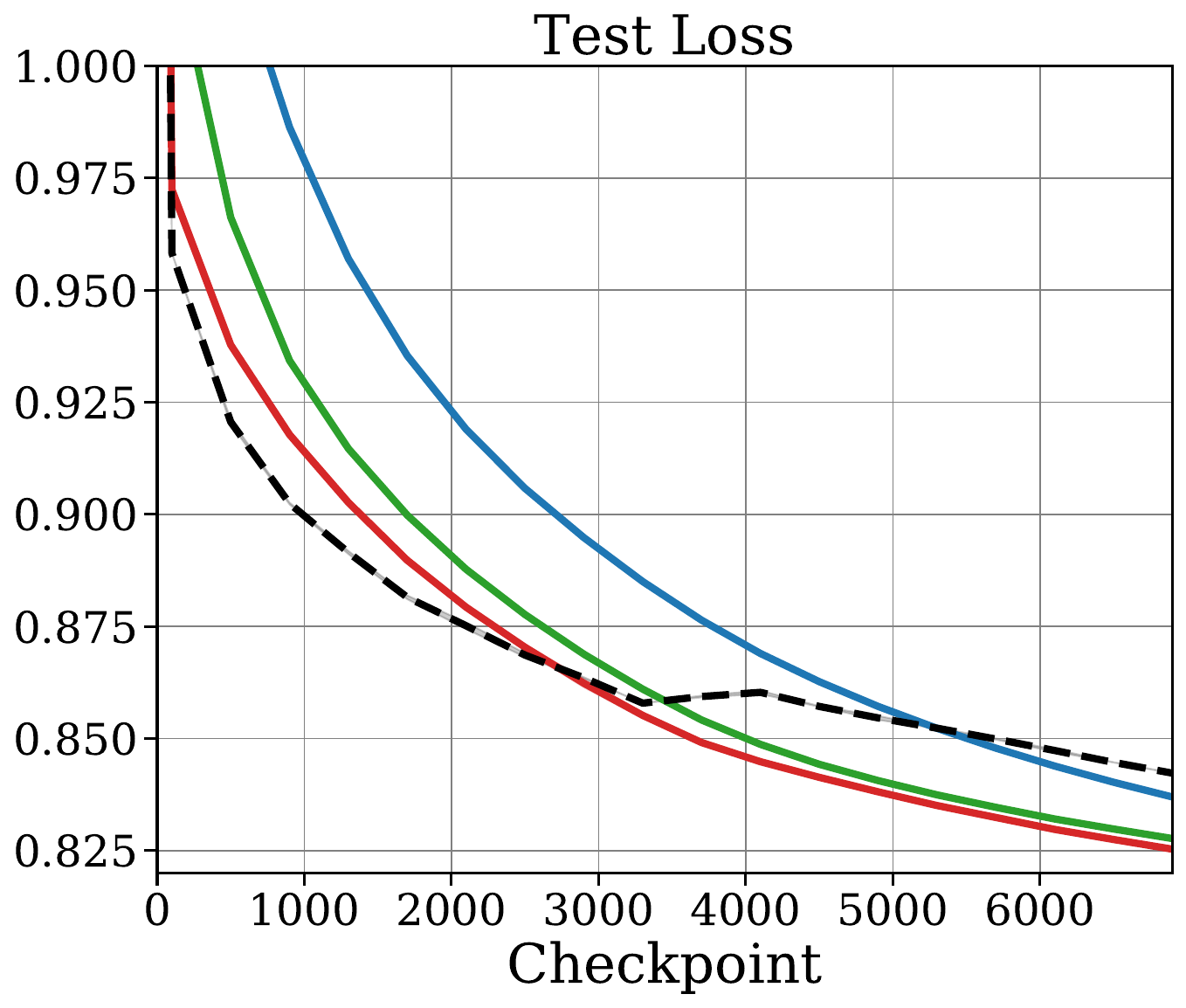}
		\label{fig4:sub1} 
	} \hfill \subfigure[]{ 
\includegraphics[width=0.29\textwidth]{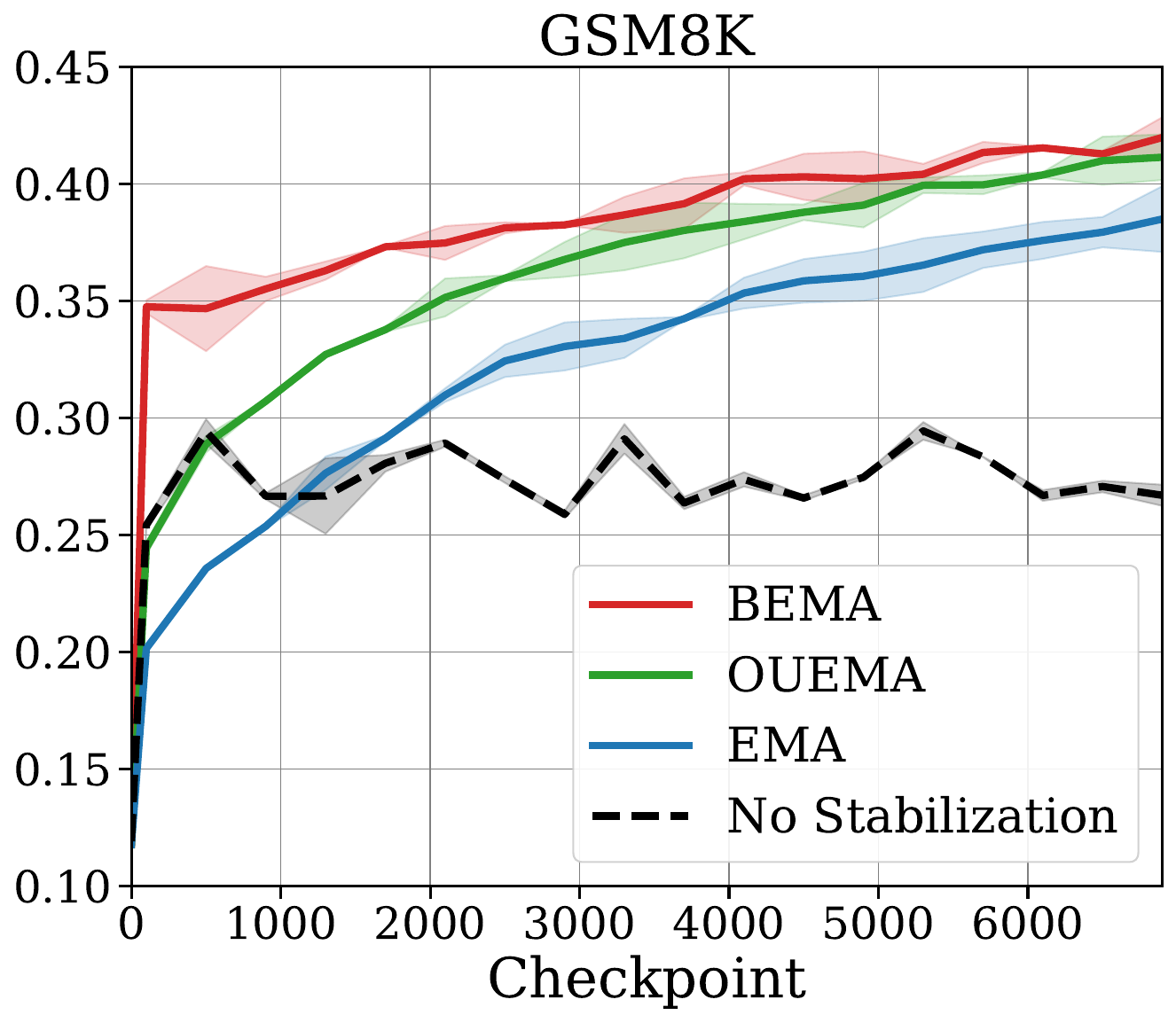} 
		\label{fig4:sub2}
	} \hfill \subfigure[]{ 
        \includegraphics[width=0.29\textwidth]{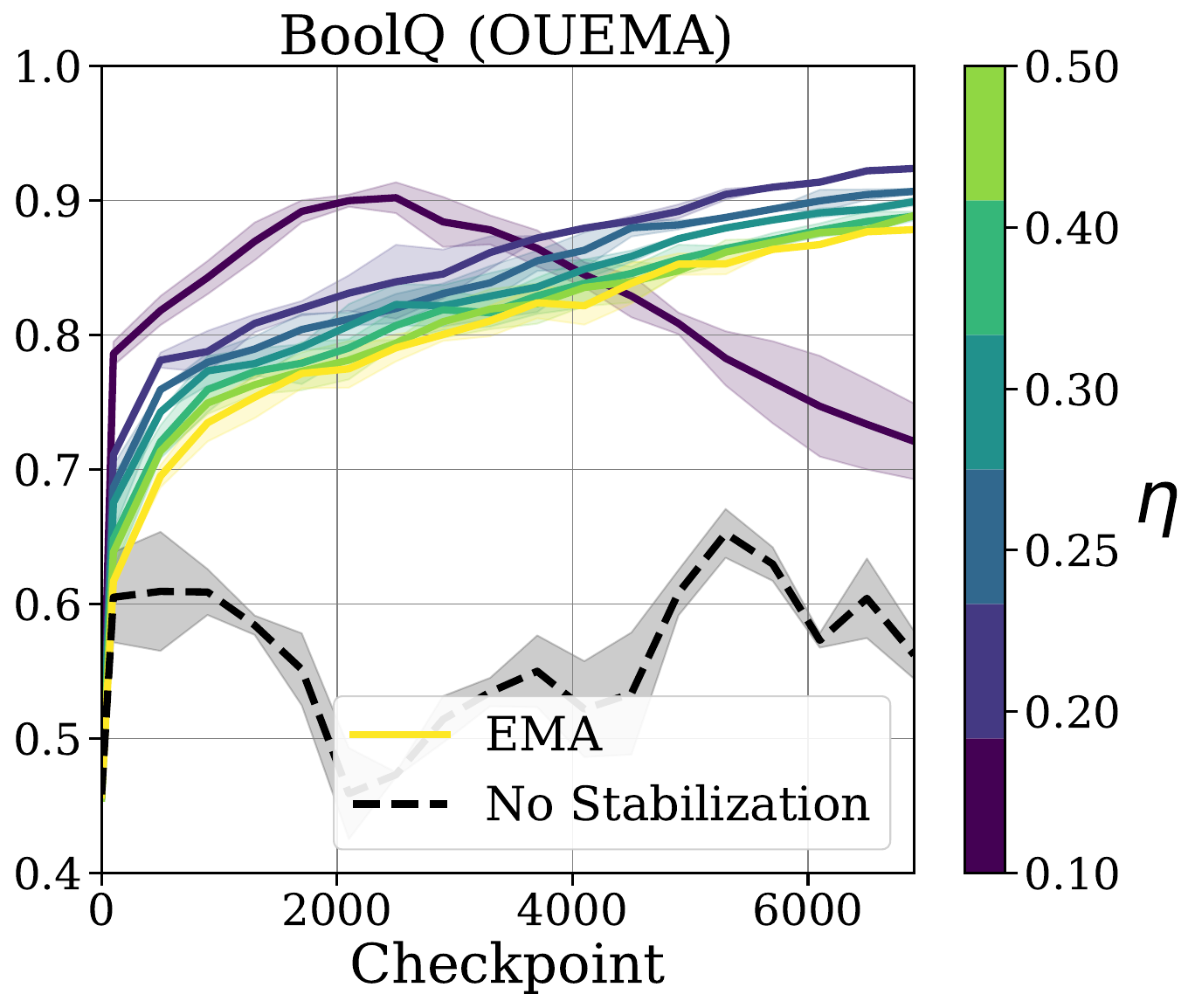} 
        \label{fig4:sub3}
    }
	\caption{
        Performance of \ouema\  as compared to \bema\  on test loss \textbf{(a)}, \gsmk~ \textbf{(b)}, and demonstration of the effect of tuning the $\eta$ hyperparameter in \ouema\  in \boolq~ \textbf{(c)}.  In all cases, \bema\  outperforms \ouema, which in turn outperforms vanilla \ema.  As in the case of \oumle, making $\eta$ smaller (leading to a stronger intervention) in general leads to improved performance until the intervention becomes too strong and performance collapses.
    }  
	\label{fig:Fig-4} 
\end{figure}  

We now empirically demonstrate the efficacy of \bema\  (\Cref{alg:oumle}) in post-training LMs and investigate the effect that the departure of the loss landscape from the idealized quadratic case has on the performance thereof.  We begin by describing our experimental setup and then briefly discuss the results of our experiments.  Further details and experiments are deferred to \Cref{app:empirical_setup,sec:app_further_empirical_results}.

\subsection{Empirical Setup}
We focus solely on the \emph{post-training} regime, where small batch sizes necessitate stabilization.  In this work, we finetune on the \tulu dataset \citep{lambert2024t}, one of the largest and highest quality opensource finetuning datasets for LMs.  Unless otherwise specified, all results are on the \emph{pre-trained} \qwen~ model \citep{team2024qwen2}, a 1.5B parameter model known for its strong performance on a number of benchmarks.  We also consider the pretrained \gemma~ and \llama~ models in the appendix.  For our training runs, we ran for the 2 epochs recommended by \citep{lambert2024t} with training hyperparameters summarized in \Cref{tab:training_hyperparams} in \Cref{app:empirical_setup}; we reran each run twice with different seeds to ensure our results are robust to the stochasticity in training.  All of our training runs were done on a single 80 Gb NVIDIA A100 GPU, while our evaluations were conducted on 40 Gb NVIDIA A100 GPUs.

\paragraph{Evaluation.}  We consider 5 metrics in order to demonstrate the broad efficacy of \oumle.  First, we consider train and test loss, which is the cross-entropy of the model on the current train batch and a fixed set of 200 held-out sequences respectively.  Second, we consider three benchmarks for language generation and reasoning.  For each prompt in each task, we generate 50 responses with temperature 1 and compare the model's responses to the groundtruth answer in order to estimate the average accuracy of the model on that prompt, then average across prompts to get the final score for each task.  We consider the following tasks:
\begin{enumerate}
    \item \boolq. We randomly select 64 fixed prompts from the language understanding dataset \boolq~ \citep{clark2019boolq}, which is a component of the standard SuperGLUE benchmark \citep{wang2019superglue}.  
    \item \gsmk. We randomly select 128 fixed prompts from the test split of the standard mathematical reasoning dataset \gsmk\ \citep{cobbe2021gsm8k}.
    \item \mmluhs.  In order to satisfy constraints on computation, we evaluate on a strict subset of the MMLU dataset \citep{hendryckstest2021}.  To ensure a fair and diverse selection of topics for which it is reasonable to expect good performance even for the relatively small models we consider, we select all of the topics labelled `high school', leading to 14 topics detailed in \Cref{app:empirical_setup}.  For each topic, we randomly select 64 prompts for which to compute average model accuracy, then average across topics to get the final \mmluhs~ score.
\end{enumerate}
We chose \mmluhs~ and \gsmk~ due to the fact that they are both (subsets of)  standard benchmarks for which improvement after finetuning on \tulu~ was demonstrated in \citet{lambert2024t}.  We chose \boolq~ because it is a standard language understanding benchmark that is cheap to evaluate due to the lack of required long reasoning chains.  

We emphasize that on all three of the above tasks, by checking whether a model's generated text producess a (mostly) correctly formatted answer that matches the groundtruth correct answer, \textbf{we evaluate on generations, unlike many of the benchmark scores reported in the literature \citep{team2025gemma,grattafiori2024llama,lambert2024t} that turn models into classifiers} on multiple choice questions by selecting the answer with the highest probability; this difference explains why we report significantly lower scores for the pre-trained model as compared to the relevant technical reports.  Indeed, in addition to the task at hand, our evaluations measure the model's ability to follow directions with respect to formatting, a skill that \tulu~ is known to aid \citep{lambert2024t}.  The fact that we are evaluating on generations is important as GVA  \citep{blockbutterfly2024} is caused by rolling out a policy (in this case autoregressive generation) in a closed loop fashion; 
in most practical scenarios, these closed loop effects are present as the primary purpose of a LM is to generate text, not solve multiple choice questions.  We discuss additional facets of our evaluation choices in \Cref{app:empirical_setup}.

\subsection{Main Results}

We now describe the main results of our empirical investigation.  \textbf{Overall, we find that \bema\ substantially improves upon both \ema\ and vanilla training on a diverse collection to tasks and models.} In \Cref{fig1:sub2} we plot the training and test cross-entropy as a function of the number of gradient steps taken with $\kappa = 0.2$ (EMA power) and $\eta = 0.2$ (strength of \bema\ correction) for vanilla training, \ema, and \bema\  on \qwen; in both cases, we see that \ema\  leads to an initially slower convergence, although does ultimately outperform vanilla training in test loss, possibly due to the regularizing effects of early stopping \citep{prechelt2002early,yao2007early} combined with the fact that \ema\  lag is functionally slowing down training.  In both train and test losses, however, we see a considerable benefit of \bema\  over \ema.  A similar picture emerges in \Cref{fig1:sub3}, where we plot the performance with $\kappa = 0.5$ for \boolq~ and \gsmk; in both cases, \bema\  significantly outperforms \ema\  and vanilla training, both in peak performance and number of gradient steps required to achieve a given level of accuracy.  

We provide further evidence for the empirical benefit of \bema\  in \Cref{fig:Fig-3}, where we examine the effect of varying $\eta$, with smaller $\eta$ leading to a stronger intervention of the bias correction term in \oumle.  In each of the generations evaluation tasks we consider (\boolq, \gsmk, and \mmluhs), we see considerable gains of \bema\  over mere \ema\  with $\eta = 0.2$ and this parameter appears fairly robust to the choice of task and other hyperparameters.    While these results are for particular choices of $\kappa$ and $\eta$, we explore plot corresponding training curves in \Cref{fig:lowema} for different values, as well as the optimal performance throughout training for several choices of $\kappa$ and $\eta$ in \Cref{fig:minlr_multiples_losses,fig:minlr_multiples_benchmarks} (cf. \Cref{sec:app_further_empirical_results}).  We observe that $\kappa = 0.5$ consistently leads to the best performance on downstream tasks, albeit with a significant increase in crossentropy that is somewhat mitigated by the bias correction term in \oumle. 

\subsection{Further Empirical Results}

We now proceed to briefly describe the results of a number of ablations and further experiments that we performed, with detailed descriptions and results deferred to \Cref{sec:app_further_empirical_results}.

\paragraph{Does \bema\  work with different optimizer hyperparameters?}  The results discussed so far are concordant with the theory presented in \Cref{sec:theory} insofar as training is conducted with a fixed learning rate.  However, practitioners often observe a benefit of decaying the learning rate due to the improved stability that comes of decreasing the noise floor of stochastic optimization.  Thus, we investigate in \Cref{fig:minlr_multiples,fig:minlr_multiples_losses,fig:minlr_multiples_benchmarks} the effect of decaying the learning rate to zero and to $0.3$ times the peak learning rate on \ema\  and \oumle.  Even with learning rate decay, both \ema\  and \bema\  continue to exhibit improvement across the board and \bema\  continues to outperform \ema; moreover, \textbf{we observe that training with a fixed learning rate and then applying \bema\  leads to the best performance throughout, providing preliminary evidence that applying post-hoc stabilization can obviate the need for learning rate decay in post-training}. We also demonstrate that \bema's performance is robust to the choice of batch size in \Cref{fig:batchsize}.

\paragraph{What are the effects of changing \bema\ hyperparameters?}  In \Cref{fig:updateafter_heatmaps}, we investigate the effect of changing $\tau$ (burn-in time) so that $\theta_0$ is the model after 500 or 1K gradient steps and stabilizing begins thereafter.  In our experiments, setting $\tau = 0$ is significantly superior and leads to by far the best performance.  It is natural to ask if a more sophisticated scheme for setting $\theta_0$, such as a very slow-moving \ema\  as is done in, e.g. \citet{grill2020bootstrap,pagliardini2024ademamix} would lead to improvement, and we leave this interesting question to future work.  In addition to examining the effect of changing $\tau$, we also investigate the effect that the lag parameter $\rho$ has on the performance of \bema\  in \Cref{fig:scalinglag}, finding minimal effect across two orders of magnitude.

Of the three parameters, $\phi$ (the frequency with which we update) is by far the most important and we investigate its effect in \Cref{fig:updatefreqs,fig:updatefreqs_heatmaps}.  Increasing $\phi$ leads to a decrease in the time of evaluation, as we do not need to perform the \bema\  update as frequently, and the referenced figures demonstrate it has a significant regularizing effect as well.  Indeed, we see that increasing the frequency (making $\phi$ smaller) leads to significant acceleration in convergence of training loss, but sometimes at the cost of overfitting.

\paragraph{Is \bema\  competitive?} In \Cref{fig:Fig-4} we compare \bema\  to \ouema, \ema, and vanilla training without stabilization on test loss, \gsmk, and \boolq. In all cases, we see that \ouema\  significantly improves over \ema, as well as vanilla training, but is worse than \bema\  across the board.  This observation is further validated in \Cref{fig:ouema}.   In \Cref{fig:dema}, we compare \bema\  to the so-called Double EMA (\dema) \citep{mulloy1994smoothing,chen2023bidirectional}, detailed in \Cref{sec:app_further_empirical_results}.  While we find that \dema~ considerably improves over \ema, it holds that \bema\  arrives at better performance substantially more quickly on all generation tasks we consider.

\paragraph{Does \bema\  work on other models?} In \Cref{fig:gemma,fig:llama}, we apply the identical procedure to \gemma~ \citep{team2025gemma} and \llama~ \citep{grattafiori2024llama}.  We continue to see improvement of \bema\  over \ema\  and vanilla training in crossentropy in all cases.  In the case of \gemma, we see substantial gains in \boolq, especially when $\kappa = 0.0$ (no EMA), while in \llama~ we see more modest gains in downstream performance.  In both cases, especially the latter, the performance of the models is significantly lower than that of \qwen, at least partly due to the fact that neither model is as good at following instructions and thus both have trouble returning correctly formatted answers.

\section{Related Work}\label{subsec:related_work}
We now summarize some relevant related work in the areas of stochastic optimization, stochastic differential equations, and the statistical estimation thereof.

\paragraph{Stochastic Optimization.} The study of stochastic optimization is classical, dating back to the introduction of SGD \citep{robbins1951stochastic}.  The theory in the convex setting is also classical \citep{nesterov2013introductory}, with momentum and especially iterate averaging \citep{ruppert1988efficient,polyak1992acceleration} emerging as powerful tools in variance reduction and acceleration.  More recent works have investigated questions of nonasymptotic optimality in this setting both in the convex and nonconvex cases \citep{johnson2013accelerating,defazio2014saga,schmidt2017minimizing}.  In the quadratic setting in particular, \citet{defossez2015averaged} considered asymptotic bounds on the performance of itererate-averaged SGD with momentum while noting the asymptotically vanishing impact of the initial iterate on convergence; while this is tolerable their setting, the focus of the present paper is on reducing the impact thereof.  Of particular relevance is \citet{dieuleveut2017harder}, which analyzes the performance of iterate-averaged SGD and attains the optimal bias (in discrete time, for high-dimensional problems, and up to constants) \citep{nesterov2013introductory}; note that our analysis does not violate this lower bound because (a) it occurs in continuous time and (b) the theoretically optimal estimator uses knowledge of the matrix $\bA$, which is absent from the first-order stochastic optimization setting.

In addition to the well-developed classical theory, there has been a plethora of work spanning the past decade on incorporating ideas from stochastic convex optimization into practical deep learning algorithms, especially in the context of preconditioning \citep{duchi2011adaptive,kingma2014adam,gupta2018shampoo,vyas2024soap}.  While our theory does not directly address the question of pre-conditioned optimizers (nor does it incorporate momentum), all of our empirical results use the standard AdamW optimizer \citep{loshchilov2017decoupled}, which incorporates both.  Many of these popular training interventions already function in a stabilizing role, including increasing the batch size, shrinking the learning rate, and using a more aggressive learning rate decay schedule \citep{blum1954approximation,krizhevsky2012imagenet}, but these approaches all come with their own drawbacks: increasing the batch size is often infeasible due to data scarcity, while smaller learning rates and more aggressive decay schedules can significantly slow training \citep{smith2019super,loshchilov2022sgdr}.  Thus, one of the most empirically successful approaches to stabilizing training is iterate averaging, which plays a crucial tool both in achieving optimal algorithms in the theory of stochastic optimization \citep{ruppert1988efficient,polyak1992acceleration,defazio2014saga,johnson2013accelerating}, and in empirical deep learning, beginning with \citet{izmailov2018averaging}.  Many recent works have specifically investigated the effects of \ema\  in this context \citep{sandler2023training,kaddour2023no,kaddour2022stop,busbridge2023scale}.  Most relevant to the present work is \citet{blockbutterfly2024}, which focused on the effect \ema\  has on evaluating closed-loop rollouts and introduced the notion of Gradient Variance Amplification (GVA).  While that work primarily observed the benefits of \ema, our work is devoted to understanding ways to improve stabilizer design, formalized as a statistical estimation problem.

\paragraph{Stochastic Differential Equations: Estimation and Optimization.} The study of stochastic Differential Equations (SDEs) is classical and has a rich theory built around it \citep{liptser2013statistics1,liptser2013statistics1,le2016brownian}. In the context of optimization in deep learning, SDEs have been considered as a useful tool for analyzing scaling limits of the discrete optimization trajectory \citep{li2017stochastic,malladi2023kernel,busbridge2023scale}, which in turn help understand how to tune hyperparameters such as the learning rate and momentum when other training settings are changed. In addition to analyzing scaling limits, many works have analyzed generalization properties of stochastic optimization as sampling from an SDE \citep{raginsky2017non,ben2022high}, with \citet{mandt2015continuous} in particular analyzing the OU process.  While we make use of the continuous time limit of SGD for the sake of analysis, our focus is more on designing new stabilizers than on the precise scaling limits thereof.

A more classical question in SDEs than that of deep learning scaling limits is the problem of parameter estimation, where a learner observes a single trajectory from the solution to an unknown SDE and seeks to estimate the parameters thereof \citep{liptser2013statistics2,kutoyants2013statistical}; typically, authors have focused on the infinite time limits of such questions, examining the asymptotic properties of a variety of estimators.  Of particular relevance to the present work is the application of Maximum Likelihood Estimation (MLE) \citep{wilks1938large} to this problem, with many authors considering the asymptotic performance and optimality of the MLE in a variety of settings \citep{kutoyants2013statistical}, including in the OU process.  In addition, parameter estimation of stochastic processes (especially the OU process) has been extensively studied in the context of mathematical finance (cf. e.g. \citet{zhang2014parameter} and the references therein), although the focus of those works tend to involve estimation from discrete time observations rather than the full continuous time trajectory.  In contradistinction to those works, our primary interest is in memory- and computationally-efficient estimators that can be practically deployed at the scale of modern LMs.  While the form of the MLE has certainly been worked out in the literature, to the best of our knowledge, its application to stochastic optimization is novel to the present work.

\section{Discussion}\label{sec:discussion}

In this work, we conceptualized the problem of designing \emph{stabilizers} for stochastic optimization as a statistical estimation problem and proposed a new algorithm, \oumle, that is able to achieve the variance reduction advantages of \ema\  while at the same time providing accelerated convergence.  Through theoretical analysis and extensive empirical evaluation, we demonstrated that \textbf{\bema\  can provide substantial gains in finetuning language models, especially on their downstream performance, while being memory efficient and computationally inexpensive}.  

In addition to further investigation as to the optimal choice of $\theta_0$ in \oumle, several immediate questions arise.  First, we empirically instantiated \bema\  with an isotropic prior, i.e., $\bA = \eye$, but it is natural to ask if some more adaptive choice could be made, e.g. using the existing second order information that AdamW and its variants provide.  Second, our analysis arose out of the OU process, which is the scaling limit of SGD in quadratic optimization, but current LM training (including all of our experiments) uses adaptive methods like AdamW; can we design stabilizers that are more directly linked to the scaling limit of these more sophisticated optimization algorithms?  Third, we focused purely on the stabilizing process itself, treating the optimization trajectory as fixed; in practice, it may be even more effective to consider the question of \emph{optimizer-stabilizer co-design}, wherein stabilization is considered to be a stochastic control problem as opposed to merely statistical estimation.  In this case, we may find memory-efficient controllers that can adaptively accelerate convergence to the minimum while at the same time stabilizing the optimization trajectory, potentially leading to significant performance gains.  Finally, our empirical focus was purely on SFT and did not consider alternative training paradigms, such as Reinforcement Learning from Human Feedback (RLHF) \citep{christiano2017deep} or RL with Verifiable Rewards algorithms such as GRPO \citep{shao2024deepseekmath}; exploring the effects of \bema\  in these settings is an interesting direction for future work.

\section*{Acknowledgements}
We would like to thank Jordan T. Ash, Dylan J. Foster, Sham Kakade, Akshay Krishnamurthy, Depen Morwani, and Abhishek Shetty for helpful discussions and feedback on an earlier draft as well as Nikhil Ghosh for insightful comments on the Double EMA.

\bibliographystyle{plainnat}
\bibliography{refs}

\begin{thebibliography}{82}
\providecommand{\natexlab}[1]{#1}
\providecommand{\url}[1]{\texttt{#1}}
\expandafter\ifx\csname urlstyle\endcsname\relax
  \providecommand{\doi}[1]{doi: #1}\else
  \providecommand{\doi}{doi: \begingroup \urlstyle{rm}\Url}\fi

\bibitem[Austin et~al.(2021)Austin, Odena, Nye, Bosma, Michalewski, Dohan, Jiang, Cai, Terry, Le, et~al.]{austin2021program}
Jacob Austin, Augustus Odena, Maxwell Nye, Maarten Bosma, Henryk Michalewski, David Dohan, Ellen Jiang, Carrie Cai, Michael Terry, Quoc Le, et~al.
\newblock Program synthesis with large language models.
\newblock \emph{arXiv preprint arXiv:2108.07732}, 2021.

\bibitem[Ben~Arous et~al.(2022)Ben~Arous, Gheissari, and Jagannath]{ben2022high}
Gerard Ben~Arous, Reza Gheissari, and Aukosh Jagannath.
\newblock High-dimensional limit theorems for sgd: Effective dynamics and critical scaling.
\newblock \emph{Advances in neural information processing systems}, 35:\penalty0 25349--25362, 2022.

\bibitem[Block et~al.(2023)Block, Jadbabaie, Pfrommer, Simchowitz, and Tedrake]{block2023provable}
Adam Block, Ali Jadbabaie, Daniel Pfrommer, Max Simchowitz, and Russ Tedrake.
\newblock Provable guarantees for generative behavior cloning: Bridging low-level stability and high-level behavior.
\newblock \emph{Advances in Neural Information Processing Systems}, 36:\penalty0 48534--48547, 2023.

\bibitem[Block et~al.(2024)Block, Foster, Krishnamurthy, Simchowitz, and Zhang]{blockbutterfly2024}
Adam Block, Dylan~J Foster, Akshay Krishnamurthy, Max Simchowitz, and Cyril Zhang.
\newblock Butterfly effects of sgd noise: Error amplification in behavior cloning and autoregression.
\newblock In \emph{The Twelfth International Conference on Learning Representations}, 2024.

\bibitem[Blum(1954)]{blum1954approximation}
Julius~R Blum.
\newblock Approximation methods which converge with probability one.
\newblock \emph{The Annals of Mathematical Statistics}, pages 382--386, 1954.

\bibitem[Busbridge et~al.(2023)Busbridge, Ramapuram, Ablin, Likhomanenko, Dhekane, Suau~Cuadros, and Webb]{busbridge2023scale}
Dan Busbridge, Jason Ramapuram, Pierre Ablin, Tatiana Likhomanenko, Eeshan~Gunesh Dhekane, Xavier Suau~Cuadros, and Russell Webb.
\newblock How to scale your ema.
\newblock \emph{Advances in Neural Information Processing Systems}, 36:\penalty0 73122--73174, 2023.

\bibitem[Chalkidis et~al.(2020)Chalkidis, Fergadiotis, Malakasiotis, Aletras, and Androutsopoulos]{chalkidis2020legal}
Ilias Chalkidis, Manos Fergadiotis, Prodromos Malakasiotis, Nikolaos Aletras, and Ion Androutsopoulos.
\newblock Legal-bert: The muppets straight out of law school.
\newblock \emph{arXiv preprint arXiv:2010.02559}, 2020.

\bibitem[Chandak et~al.(2025)Chandak, Goel, and Prabhu]{llmrl2025incorrect}
Nikhil Chandak, Shashwat Goel, and Ameya Prabhu.
\newblock Incorrect baseline evaluations call into question recent llm-rl claims, 2025.
\newblock Notion Blog.

\bibitem[Chang et~al.(2023)Chang, Brantley, Ramamurthy, Misra, and Sun]{chang2023learning}
Jonathan~D Chang, Kiante Brantley, Rajkumar Ramamurthy, Dipendra Misra, and Wen Sun.
\newblock Learning to generate better than your llm.
\newblock \emph{arXiv preprint arXiv:2306.11816}, 2023.

\bibitem[Chen et~al.(2023)Chen, Li, Zhang, Du, and Zhao]{chen2023bidirectional}
Yineng Chen, Zuchao Li, Lefei Zhang, Bo~Du, and Hai Zhao.
\newblock Bidirectional looking with a novel double exponential moving average to adaptive and non-adaptive momentum optimizers.
\newblock In \emph{International Conference on Machine Learning}, pages 4764--4803. PMLR, 2023.

\bibitem[Chernozhukov et~al.(2018)Chernozhukov, Chetverikov, Demirer, Duflo, Hansen, Newey, and Robins]{chernozhukov2018double}
Victor Chernozhukov, Denis Chetverikov, Mert Demirer, Esther Duflo, Christian Hansen, Whitney Newey, and James Robins.
\newblock Double/debiased machine learning for treatment and structural parameters, 2018.

\bibitem[Christiano et~al.(2017)Christiano, Leike, Brown, Martic, Legg, and Amodei]{christiano2017deep}
Paul~F Christiano, Jan Leike, Tom Brown, Miljan Martic, Shane Legg, and Dario Amodei.
\newblock Deep reinforcement learning from human preferences.
\newblock \emph{Advances in neural information processing systems}, 30, 2017.

\bibitem[Clark et~al.(2019)Clark, Lee, Chang, Kwiatkowski, Collins, and Toutanova]{clark2019boolq}
Christopher Clark, Kenton Lee, Ming-Wei Chang, Tom Kwiatkowski, Michael Collins, and Kristina Toutanova.
\newblock Boolq: Exploring the surprising difficulty of natural yes/no questions.
\newblock \emph{arXiv preprint arXiv:1905.10044}, 2019.

\bibitem[Cobbe et~al.(2021)Cobbe, Kosaraju, Bavarian, Chen, Jun, Kaiser, Plappert, Tworek, Hilton, Nakano, Hesse, and Schulman]{cobbe2021gsm8k}
Karl Cobbe, Vineet Kosaraju, Mohammad Bavarian, Mark Chen, Heewoo Jun, Lukasz Kaiser, Matthias Plappert, Jerry Tworek, Jacob Hilton, Reiichiro Nakano, Christopher Hesse, and John Schulman.
\newblock Training verifiers to solve math word problems.
\newblock \emph{arXiv preprint arXiv:2110.14168}, 2021.

\bibitem[Cohen et~al.(2021)Cohen, Kaur, Li, Kolter, and Talwalkar]{cohen2021gradient}
Jeremy~M Cohen, Simran Kaur, Yuanzhi Li, J~Zico Kolter, and Ameet Talwalkar.
\newblock Gradient descent on neural networks typically occurs at the edge of stability.
\newblock \emph{arXiv preprint arXiv:2103.00065}, 2021.

\bibitem[Defazio et~al.(2014)Defazio, Bach, and Lacoste-Julien]{defazio2014saga}
Aaron Defazio, Francis Bach, and Simon Lacoste-Julien.
\newblock Saga: A fast incremental gradient method with support for non-strongly convex composite objectives.
\newblock \emph{Advances in neural information processing systems}, 27, 2014.

\bibitem[D{\'e}fossez and Bach(2015)]{defossez2015averaged}
Alexandre D{\'e}fossez and Francis Bach.
\newblock Averaged least-mean-squares: Bias-variance trade-offs and optimal sampling distributions.
\newblock In \emph{Artificial Intelligence and Statistics}, pages 205--213. PMLR, 2015.

\bibitem[Dieuleveut et~al.(2017)Dieuleveut, Flammarion, and Bach]{dieuleveut2017harder}
Aymeric Dieuleveut, Nicolas Flammarion, and Francis Bach.
\newblock Harder, better, faster, stronger convergence rates for least-squares regression.
\newblock \emph{Journal of Machine Learning Research}, 18\penalty0 (101):\penalty0 1--51, 2017.

\bibitem[Duchi et~al.(2011)Duchi, Hazan, and Singer]{duchi2011adaptive}
John Duchi, Elad Hazan, and Yoram Singer.
\newblock Adaptive subgradient methods for online learning and stochastic optimization.
\newblock \emph{Journal of machine learning research}, 12\penalty0 (7), 2011.

\bibitem[Foster et~al.(2024)Foster, Block, and Misra]{foster2024behavior}
Dylan~J Foster, Adam Block, and Dipendra Misra.
\newblock Is behavior cloning all you need? understanding horizon in imitation learning.
\newblock In \emph{The Thirty-eighth Annual Conference on Neural Information Processing Systems}, 2024.

\bibitem[Grattafiori et~al.(2024)Grattafiori, Dubey, Jauhri, Pandey, Kadian, Al-Dahle, Letman, Mathur, Schelten, Vaughan, et~al.]{grattafiori2024llama}
Aaron Grattafiori, Abhimanyu Dubey, Abhinav Jauhri, Abhinav Pandey, Abhishek Kadian, Ahmad Al-Dahle, Aiesha Letman, Akhil Mathur, Alan Schelten, Alex Vaughan, et~al.
\newblock The llama 3 herd of models.
\newblock \emph{arXiv preprint arXiv:2407.21783}, 2024.

\bibitem[Grill et~al.(2020)Grill, Strub, Altch{\'e}, Tallec, Richemond, Buchatskaya, Doersch, Avila~Pires, Guo, Gheshlaghi~Azar, et~al.]{grill2020bootstrap}
Jean-Bastien Grill, Florian Strub, Florent Altch{\'e}, Corentin Tallec, Pierre Richemond, Elena Buchatskaya, Carl Doersch, Bernardo Avila~Pires, Zhaohan Guo, Mohammad Gheshlaghi~Azar, et~al.
\newblock Bootstrap your own latent-a new approach to self-supervised learning.
\newblock \emph{Advances in neural information processing systems}, 33:\penalty0 21271--21284, 2020.

\bibitem[Gupta et~al.(2018)Gupta, Koren, and Singer]{gupta2018shampoo}
Vineet Gupta, Tomer Koren, and Yoram Singer.
\newblock Shampoo: Preconditioned stochastic tensor optimization.
\newblock In \emph{International Conference on Machine Learning}, pages 1842--1850. PMLR, 2018.

\bibitem[Hendrycks et~al.(2021)Hendrycks, Burns, Basart, Zou, Mazeika, Song, and Steinhardt]{hendryckstest2021}
Dan Hendrycks, Collin Burns, Steven Basart, Andy Zou, Mantas Mazeika, Dawn Song, and Jacob Steinhardt.
\newblock Measuring massive multitask language understanding.
\newblock \emph{Proceedings of the International Conference on Learning Representations (ICLR)}, 2021.

\bibitem[Izmailov et~al.(2018)Izmailov, Podoprikhin, Garipov, Vetrov, and Wilson]{izmailov2018averaging}
Pavel Izmailov, Dmitrii Podoprikhin, Timur Garipov, Dmitry Vetrov, and Andrew~Gordon Wilson.
\newblock Averaging weights leads to wider optima and better generalization.
\newblock \emph{arXiv preprint arXiv:1803.05407}, 2018.

\bibitem[Jacot et~al.(2018)Jacot, Gabriel, and Hongler]{jacot2018neural}
Arthur Jacot, Franck Gabriel, and Cl{\'e}ment Hongler.
\newblock Neural tangent kernel: Convergence and generalization in neural networks.
\newblock \emph{Advances in neural information processing systems}, 31, 2018.

\bibitem[Johnson and Zhang(2013)]{johnson2013accelerating}
Rie Johnson and Tong Zhang.
\newblock Accelerating stochastic gradient descent using predictive variance reduction.
\newblock \emph{Advances in neural information processing systems}, 26, 2013.

\bibitem[Kaddour(2022)]{kaddour2022stop}
Jean Kaddour.
\newblock Stop wasting my time! saving days of imagenet and bert training with latest weight averaging.
\newblock \emph{arXiv preprint arXiv:2209.14981}, 2022.

\bibitem[Kaddour et~al.(2023)Kaddour, Key, Nawrot, Minervini, and Kusner]{kaddour2023no}
Jean Kaddour, Oscar Key, Piotr Nawrot, Pasquale Minervini, and Matt~J Kusner.
\newblock No train no gain: Revisiting efficient training algorithms for transformer-based language models.
\newblock \emph{Advances in Neural Information Processing Systems}, 36:\penalty0 25793--25818, 2023.

\bibitem[Kamath et~al.(2025)Kamath, Ferret, Pathak, Vieillard, Merhej, Perrin, Matejovicova, Ram{\'e}, Rivi{\`e}re, et~al.]{team2025gemma}
Aishwarya Kamath, Johan Ferret, Shreya Pathak, Nino Vieillard, Ramona Merhej, Sarah Perrin, Tatiana Matejovicova, Alexandre Ram{\'e}, Morgane Rivi{\`e}re, et~al.
\newblock Gemma 3 technical report.
\newblock \emph{arXiv preprint arXiv:2503.19786}, 2025.

\bibitem[Karras et~al.(2023)Karras, Aittala, Lehtinen, Hellsten, Aila, and Laine]{Karras2023AnalyzingAI}
Tero Karras, Miika Aittala, Jaakko Lehtinen, Janne Hellsten, Timo Aila, and Samuli Laine.
\newblock Analyzing and improving the training dynamics of diffusion models.
\newblock \emph{ArXiv}, abs/2312.02696, 2023.
\newblock URL \url{https://api.semanticscholar.org/CorpusID:265659032}.

\bibitem[Kingma and Ba(2014)]{kingma2014adam}
Diederik~P Kingma and Jimmy Ba.
\newblock Adam: A method for stochastic optimization.
\newblock \emph{arXiv preprint arXiv:1412.6980}, 2014.

\bibitem[Krizhevsky et~al.(2012)Krizhevsky, Sutskever, and Hinton]{krizhevsky2012imagenet}
Alex Krizhevsky, Ilya Sutskever, and Geoffrey~E. Hinton.
\newblock Imagenet classification with deep convolutional neural networks.
\newblock In \emph{Proceedings of the 26th International Conference on Neural Information Processing Systems - Volume 1}, NIPS'12, page 1097–1105, Red Hook, NY, USA, 2012. Curran Associates Inc.

\bibitem[Kutoyants(2013)]{kutoyants2013statistical}
Yury~A Kutoyants.
\newblock \emph{Statistical inference for ergodic diffusion processes}.
\newblock Springer Science \& Business Media, 2013.

\bibitem[Kwon et~al.(2023)Kwon, Li, Zhuang, Sheng, Zheng, Yu, Gonzalez, Zhang, and Stoica]{kwon2023efficient}
Woosuk Kwon, Zhuohan Li, Siyuan Zhuang, Ying Sheng, Lianmin Zheng, Cody~Hao Yu, Joseph~E. Gonzalez, Hao Zhang, and Ion Stoica.
\newblock Efficient memory management for large language model serving with pagedattention.
\newblock In \emph{Proceedings of the ACM SIGOPS 29th Symposium on Operating Systems Principles}, 2023.

\bibitem[Lambert et~al.(2024)Lambert, Morrison, Pyatkin, Huang, Ivison, Brahman, Miranda, Liu, Dziri, Lyu, et~al.]{lambert2024t}
Nathan Lambert, Jacob Morrison, Valentina Pyatkin, Shengyi Huang, Hamish Ivison, Faeze Brahman, Lester James~V Miranda, Alisa Liu, Nouha Dziri, Shane Lyu, et~al.
\newblock Tulu 3: Pushing frontiers in open language model post-training.
\newblock \emph{arXiv preprint arXiv:2411.15124}, 2024.

\bibitem[Le~Gall(2016)]{le2016brownian}
Jean-Fran{\c{c}}ois Le~Gall.
\newblock \emph{Brownian motion, martingales, and stochastic calculus}.
\newblock Springer, 2016.

\bibitem[Lee et~al.(2024)Lee, Cho, Kim, Kim, Min, Choo, and Lyle]{Lee2024SlowAS}
Hojoon Lee, Hyeonseo Cho, Hyunseung Kim, Donghu Kim, Dugki Min, Jaegul Choo, and Clare Lyle.
\newblock Slow and steady wins the race: Maintaining plasticity with hare and tortoise networks.
\newblock \emph{ArXiv}, abs/2406.02596, 2024.
\newblock URL \url{https://api.semanticscholar.org/CorpusID:270258586}.

\bibitem[Lee et~al.(2020)Lee, Yoon, Kim, Kim, Kim, So, and Kang]{lee2020biobert}
Jinhyuk Lee, Wonjin Yoon, Sungdong Kim, Donghyeon Kim, Sunkyu Kim, Chan~Ho So, and Jaewoo Kang.
\newblock Biobert: a pre-trained biomedical language representation model for biomedical text mining.
\newblock \emph{Bioinformatics}, 36\penalty0 (4):\penalty0 1234--1240, 2020.

\bibitem[Lehmann and Casella(2006)]{lehmann2006theory}
Erich~L Lehmann and George Casella.
\newblock \emph{Theory of point estimation}.
\newblock Springer Science \& Business Media, 2006.

\bibitem[Li et~al.(2017)Li, Tai, et~al.]{li2017stochastic}
Qianxiao Li, Cheng Tai, et~al.
\newblock Stochastic modified equations and adaptive stochastic gradient algorithms.
\newblock In \emph{International Conference on Machine Learning}, pages 2101--2110. PMLR, 2017.

\bibitem[Li et~al.(2024)Li, Liu, Tian, Wang, Wang, Jin, Wu, Tan, Lin, Liu, Sun, and Li]{Li2024SwitchEA}
Siyuan Li, Zicheng Liu, Juanxi Tian, Ge~Wang, Zedong Wang, Weiyang Jin, Di~Wu, Cheng Tan, Tao Lin, Yang Liu, Baigui Sun, and Stan~Z. Li.
\newblock Switch ema: A free lunch for better flatness and sharpness.
\newblock \emph{ArXiv}, abs/2402.09240, 2024.
\newblock URL \url{https://api.semanticscholar.org/CorpusID:267657558}.

\bibitem[Liptser and Shiryaev(2013{\natexlab{a}})]{liptser2013statistics1}
Robert~S Liptser and Albert~N Shiryaev.
\newblock \emph{Statistics of random processes: I. General theory}, volume~5.
\newblock Springer Science \& Business Media, 2013{\natexlab{a}}.

\bibitem[Liptser and Shiryaev(2013{\natexlab{b}})]{liptser2013statistics2}
Robert~S Liptser and Albert~N Shiryaev.
\newblock \emph{Statistics of random processes II: Applications}, volume~6.
\newblock Springer Science \& Business Media, 2013{\natexlab{b}}.

\bibitem[Liu et~al.(2023)Liu, Singh, Freeman, Co-Reyes, and Liu]{liu2023improving}
Yixin Liu, Avi Singh, C~Daniel Freeman, John~D Co-Reyes, and Peter~J Liu.
\newblock Improving large language model fine-tuning for solving math problems.
\newblock \emph{arXiv preprint arXiv:2310.10047}, 2023.

\bibitem[Loshchilov and Hutter(2017)]{loshchilov2017decoupled}
Ilya Loshchilov and Frank Hutter.
\newblock Decoupled weight decay regularization.
\newblock \emph{arXiv preprint arXiv:1711.05101}, 2017.

\bibitem[Loshchilov and Hutter(2022)]{loshchilov2022sgdr}
Ilya Loshchilov and Frank Hutter.
\newblock Sgdr: Stochastic gradient descent with warm restarts.
\newblock In \emph{International Conference on Learning Representations}, 2022.

\bibitem[Malladi et~al.(2022)Malladi, Lyu, Panigrahi, and Arora]{malladi2022sdes}
Sadhika Malladi, Kaifeng Lyu, Abhishek Panigrahi, and Sanjeev Arora.
\newblock On the sdes and scaling rules for adaptive gradient algorithms.
\newblock \emph{Advances in Neural Information Processing Systems}, 35:\penalty0 7697--7711, 2022.

\bibitem[Malladi et~al.(2023)Malladi, Wettig, Yu, Chen, and Arora]{malladi2023kernel}
Sadhika Malladi, Alexander Wettig, Dingli Yu, Danqi Chen, and Sanjeev Arora.
\newblock A kernel-based view of language model fine-tuning.
\newblock In \emph{International Conference on Machine Learning}, pages 23610--23641. PMLR, 2023.

\bibitem[Mandt et~al.(2015)Mandt, Hoffman, Blei, et~al.]{mandt2015continuous}
Stephan Mandt, Matthew~D Hoffman, David~M Blei, et~al.
\newblock Continuous-time limit of stochastic gradient descent revisited.
\newblock \emph{NIPS-2015}, 2015.

\bibitem[Masters and Luschi(2018)]{masters2018revisiting}
Dominic Masters and Carlo Luschi.
\newblock Revisiting small batch training for deep neural networks.
\newblock \emph{arXiv preprint arXiv:1804.07612}, 2018.

\bibitem[Muennighoff et~al.(2023)Muennighoff, Rush, Barak, Le~Scao, Tazi, Piktus, Pyysalo, Wolf, and Raffel]{muennighoff2023scaling}
Niklas Muennighoff, Alexander Rush, Boaz Barak, Teven Le~Scao, Nouamane Tazi, Aleksandra Piktus, Sampo Pyysalo, Thomas Wolf, and Colin~A Raffel.
\newblock Scaling data-constrained language models.
\newblock \emph{Advances in Neural Information Processing Systems}, 36:\penalty0 50358--50376, 2023.

\bibitem[Mulloy(1994)]{mulloy1994smoothing}
Patrick~G Mulloy.
\newblock Smoothing data with faster moving averages.
\newblock \emph{Stocks \& Commodities}, 12\penalty0 (1):\penalty0 11--19, 1994.

\bibitem[Nesterov(2013)]{nesterov2013introductory}
Yurii Nesterov.
\newblock \emph{Introductory lectures on convex optimization: A basic course}, volume~87.
\newblock Springer Science \& Business Media, 2013.

\bibitem[Pagliardini et~al.(2024)Pagliardini, Ablin, and Grangier]{pagliardini2024ademamix}
Matteo Pagliardini, Pierre Ablin, and David Grangier.
\newblock The ademamix optimizer: Better, faster, older.
\newblock \emph{arXiv preprint arXiv:2409.03137}, 2024.

\bibitem[Polyak and Juditsky(1992)]{polyak1992acceleration}
Boris~T Polyak and Anatoli~B Juditsky.
\newblock Acceleration of stochastic approximation by averaging.
\newblock \emph{SIAM journal on control and optimization}, 30\penalty0 (4):\penalty0 838--855, 1992.

\bibitem[Prechelt(2002)]{prechelt2002early}
Lutz Prechelt.
\newblock Early stopping-but when?
\newblock In \emph{Neural Networks: Tricks of the trade}, pages 55--69. Springer, 2002.

\bibitem[Raginsky et~al.(2017)Raginsky, Rakhlin, and Telgarsky]{raginsky2017non}
Maxim Raginsky, Alexander Rakhlin, and Matus Telgarsky.
\newblock Non-convex learning via stochastic gradient langevin dynamics: a nonasymptotic analysis.
\newblock In \emph{Conference on Learning Theory}, pages 1674--1703. PMLR, 2017.

\bibitem[Rather et~al.(2024)Rather, Kumar, and Gandomi]{rather2024breaking}
Ishfaq~Hussain Rather, Sushil Kumar, and Amir~H Gandomi.
\newblock Breaking the data barrier: a review of deep learning techniques for democratizing ai with small datasets.
\newblock \emph{Artificial Intelligence Review}, 57\penalty0 (9):\penalty0 226, 2024.

\bibitem[Robbins and Monro(1951)]{robbins1951stochastic}
Herbert Robbins and Sutton Monro.
\newblock A stochastic approximation method.
\newblock \emph{The annals of mathematical statistics}, pages 400--407, 1951.

\bibitem[Rohatgi et~al.(2025)Rohatgi, Block, Huang, Krishnamurthy, and Foster]{rohatgi2025computational}
Dhruv Rohatgi, Adam Block, Audrey Huang, Akshay Krishnamurthy, and Dylan~J Foster.
\newblock Computational-statistical tradeoffs at the next-token prediction barrier: Autoregressive and imitation learning under misspecification.
\newblock \emph{arXiv preprint arXiv:2502.12465}, 2025.

\bibitem[Ross and Bagnell(2010)]{ross2010efficient}
St{\'e}phane Ross and Drew Bagnell.
\newblock Efficient reductions for imitation learning.
\newblock In \emph{Proceedings of the thirteenth international conference on artificial intelligence and statistics}, pages 661--668. JMLR Workshop and Conference Proceedings, 2010.

\bibitem[Ross et~al.(2011)Ross, Gordon, and Bagnell]{ross2011reduction}
St{\'e}phane Ross, Geoffrey Gordon, and Drew Bagnell.
\newblock A reduction of imitation learning and structured prediction to no-regret online learning.
\newblock In \emph{Proceedings of the fourteenth international conference on artificial intelligence and statistics}, pages 627--635. JMLR Workshop and Conference Proceedings, 2011.

\bibitem[Ruppert(1988)]{ruppert1988efficient}
David Ruppert.
\newblock Efficient estimations from a slowly convergent robbins-monro process.
\newblock Technical report, Cornell University Operations Research and Industrial Engineering, 1988.

\bibitem[Sandler et~al.(2023)Sandler, Zhmoginov, Vladymyrov, and Miller]{sandler2023training}
Mark Sandler, Andrey Zhmoginov, Max Vladymyrov, and Nolan Miller.
\newblock Training trajectories, mini-batch losses and the curious role of the learning rate.
\newblock \emph{arXiv preprint arXiv:2301.02312}, 2023.

\bibitem[Schmidt et~al.(2017)Schmidt, Le~Roux, and Bach]{schmidt2017minimizing}
Mark Schmidt, Nicolas Le~Roux, and Francis Bach.
\newblock Minimizing finite sums with the stochastic average gradient.
\newblock \emph{Mathematical Programming}, 162:\penalty0 83--112, 2017.

\bibitem[Shao et~al.(2024)Shao, Wang, Zhu, Xu, Song, Bi, Zhang, Zhang, Li, Wu, et~al.]{shao2024deepseekmath}
Zhihong Shao, Peiyi Wang, Qihao Zhu, Runxin Xu, Junxiao Song, Xiao Bi, Haowei Zhang, Mingchuan Zhang, YK~Li, Y~Wu, et~al.
\newblock Deepseekmath: Pushing the limits of mathematical reasoning in open language models, 2024.
\newblock \emph{URL https://arxiv. org/abs/2402.03300}, 2\penalty0 (3):\penalty0 5, 2024.

\bibitem[Smith and Topin(2019)]{smith2019super}
Leslie~N Smith and Nicholay Topin.
\newblock Super-convergence: Very fast training of neural networks using large learning rates.
\newblock In \emph{Artificial intelligence and machine learning for multi-domain operations applications}, volume 11006, pages 369--386. SPIE, 2019.

\bibitem[Team(2024)]{team2024qwen2}
Qwen Team.
\newblock Qwen2 technical report.
\newblock \emph{arXiv preprint arXiv:2407.10671}, 2024.

\bibitem[Van~Brunt(2004)]{Van_Brunt2004}
B~Van~Brunt.
\newblock \emph{The calculus of variations}.
\newblock Universitext. Springer, New York, NY, December 2004.

\bibitem[Villalobos et~al.(2024)Villalobos, Ho, Sevilla, Besiroglu, Heim, and Hobbhahn]{villalobos2024position}
Pablo Villalobos, Anson Ho, Jaime Sevilla, Tamay Besiroglu, Lennart Heim, and Marius Hobbhahn.
\newblock Position: Will we run out of data? limits of llm scaling based on human-generated data.
\newblock In \emph{Forty-first International Conference on Machine Learning}, 2024.

\bibitem[von Werra et~al.(2020)von Werra, Belkada, Tunstall, Beeching, Thrush, Lambert, Huang, Rasul, and Gallouédec]{vonwerra2022trl}
Leandro von Werra, Younes Belkada, Lewis Tunstall, Edward Beeching, Tristan Thrush, Nathan Lambert, Shengyi Huang, Kashif Rasul, and Quentin Gallouédec.
\newblock Trl: Transformer reinforcement learning.
\newblock \url{https://github.com/huggingface/trl}, 2020.

\bibitem[Vyas et~al.(2024)Vyas, Morwani, Zhao, Kwun, Shapira, Brandfonbrener, Janson, and Kakade]{vyas2024soap}
Nikhil Vyas, Depen Morwani, Rosie Zhao, Mujin Kwun, Itai Shapira, David Brandfonbrener, Lucas Janson, and Sham Kakade.
\newblock Soap: Improving and stabilizing shampoo using adam.
\newblock \emph{arXiv preprint arXiv:2409.11321}, 2024.

\bibitem[Wang et~al.(2019)Wang, Pruksachatkun, Nangia, Singh, Michael, Hill, Levy, and Bowman]{wang2019superglue}
Alex Wang, Yada Pruksachatkun, Nikita Nangia, Amanpreet Singh, Julian Michael, Felix Hill, Omer Levy, and Samuel Bowman.
\newblock Superglue: A stickier benchmark for general-purpose language understanding systems.
\newblock \emph{Advances in neural information processing systems}, 32, 2019.

\bibitem[Wang(2024)]{lucidrains_emapytorch}
Phil Wang.
\newblock ema-pytorch: A simple way to keep track of an exponential moving average (ema) version of your pytorch model.
\newblock \url{https://github.com/lucidrains/ema-pytorch}, 2024.
\newblock Accessed: 2025-06-20.

\bibitem[Wei et~al.(2022)Wei, Wang, Schuurmans, Bosma, Xia, Chi, Le, Zhou, et~al.]{wei2022chain}
Jason Wei, Xuezhi Wang, Dale Schuurmans, Maarten Bosma, Fei Xia, Ed~Chi, Quoc~V Le, Denny Zhou, et~al.
\newblock Chain-of-thought prompting elicits reasoning in large language models.
\newblock \emph{Advances in neural information processing systems}, 35:\penalty0 24824--24837, 2022.

\bibitem[Wilks(1938)]{wilks1938large}
Samuel~S Wilks.
\newblock The large-sample distribution of the likelihood ratio for testing composite hypotheses.
\newblock \emph{The annals of mathematical statistics}, 9\penalty0 (1):\penalty0 60--62, 1938.

\bibitem[Wolf et~al.(2019)Wolf, Debut, Sanh, Chaumond, Delangue, Moi, Cistac, Rault, Louf, Funtowicz, et~al.]{wolf2019huggingface}
Thomas Wolf, Lysandre Debut, Victor Sanh, Julien Chaumond, Clement Delangue, Anthony Moi, Pierric Cistac, Tim Rault, R{\'e}mi Louf, Morgan Funtowicz, et~al.
\newblock Huggingface's transformers: State-of-the-art natural language processing.
\newblock \emph{arXiv preprint arXiv:1910.03771}, 2019.

\bibitem[Yao et~al.(2007)Yao, Rosasco, and Caponnetto]{yao2007early}
Yuan Yao, Lorenzo Rosasco, and Andrea Caponnetto.
\newblock On early stopping in gradient descent learning.
\newblock \emph{Constructive approximation}, 26\penalty0 (2):\penalty0 289--315, 2007.

\bibitem[Zhang et~al.(2019)Zhang, Li, Nado, Martens, Sachdeva, Dahl, Shallue, and Grosse]{zhang2019algorithmic}
Guodong Zhang, Lala Li, Zachary Nado, James Martens, Sushant Sachdeva, George Dahl, Chris Shallue, and Roger~B Grosse.
\newblock Which algorithmic choices matter at which batch sizes? insights from a noisy quadratic model.
\newblock \emph{Advances in neural information processing systems}, 32, 2019.

\bibitem[Zhang et~al.(2024)Zhang, Morwani, Vyas, Wu, Zou, Ghai, Foster, and Kakade]{zhang2024does}
Hanlin Zhang, Depen Morwani, Nikhil Vyas, Jingfeng Wu, Difan Zou, Udaya Ghai, Dean Foster, and Sham Kakade.
\newblock How does critical batch size scale in pre-training?
\newblock \emph{arXiv preprint arXiv:2410.21676}, 2024.

\bibitem[Zhang et~al.(2014)Zhang, Xiao, Zhang, and Niu]{zhang2014parameter}
Pu~Zhang, Wei-lin Xiao, Xi-li Zhang, and Pan-qiang Niu.
\newblock Parameter identification for fractional ornstein--uhlenbeck processes based on discrete observation.
\newblock \emph{Economic Modelling}, 36:\penalty0 198--203, 2014.

\end{thebibliography}

\newpage

\appendix

\crefalias{section}{appendix}

\section{Further Details on Empirical Setup}\label[appendix]{app:empirical_setup}

In this section, we provide further details on our empirical setup, including the base training hyperparameters, precise statistics of the dataset we use, and our evaluation setup.

\paragraph{Training Data.}  We use the \tulu~ dataset \citep{lambert2024t}, which ordinarily consists of about 1M sequences as our training data.  We randomly split the data, keeping $99\%$ for training and $1\%$ for validation.  We then filter our training set to ensure that each sequence has at most 4096 tokens as per the \qwen~ tokenizer in order to prevent any memory issues.  This results in 929,949 distinct training sequences, amounting to about 600M tokens.  We then randomly select 200 sequences from the validation set for evaluation during training, which we keep fixed throughout.

\paragraph{Training.}  The default hyperparameters we use for training are summarized in \Cref{tab:training_hyperparams}.  While by default we do not use any learning rate decay, as we find that after stabilization this leads to the best performance, when we do experiment with learning rate decay, we use a linear schedule without warmup, decaying to some constant fraction of the peak learning rate.  We train using the HuggingFace transformers and trl libraries \citep{wolf2019huggingface,vonwerra2022trl}, in particular using the \textsf{SFTTrainer} class.  All training was conducted on 80 Gb NVIDIA A100 GPUs.  When training \gemma~ and \llama~ models, we use their tokenizers but always enforce the chat template used for \qwen~ in an effort to ensure consistency.

\paragraph{Stabilizers.}  We consider four candidate stabilizers: \ema, \oumle, \ouema, and \dema.  The implementation of \bema\  is given in \Cref{alg:oumle}, with $\gamma = 1.0$ and, by default, $\rho = 10$.  We implement \ema\  as a special case of \oumle, but with $\eta = \infty$, which removes the bias correction term.  To implement \ouema, we compute
\begin{align}
    \thetabar_t = \left( 1 - \frac{1}{(1 + \gamma t)^{\eta}} \right)^{-1} \left( \theta_t - \frac{1}{(1 + \gamma t)^\eta} \theta_0 \right)\label{eq:ouema_practical}
\end{align}
and then apply the EMA update rule in \Cref{alg:oumle}, but with $\theta_t$ replaced by $\thetabar_t$.  Finally, we discuss \dema~ in \Cref{sec:app_further_empirical_results}.  In all cases, we set $\gamma = 1.0$ and, unless otherwise specified, we set $\rho = 10$.  Finally, unless otherwise noted, we always assume the frequency $\phi = 400$ in order to reduce the computational cost of evaluation.  Default hyperparameters for the stabilizers are summarized in \Cref{tab:stabilizer_hyperparams}.

\paragraph{Evaluation.}  We evaluate the candidate stabilizers (\ema, \oumle, \ouema, and \dema) on saved checkpoints so as to escape the need to retrain the model multiple times.  By default, in an effort to reduce computation, we update the stabilizer every 400 gradient steps, although we investigate the effect of this choice in \Cref{sec:app_further_empirical_results}.  As described in \Cref{sec:empirics}, we consider train and test losses, consisting of cross-entropy on the current training batch and a fixed set of 200 held-out sequences, respectively.  We also consider \boolq, \gsmk, and \mmluhs~ as benchmarks for language generation and reasoning.  We use vLLM \citep{kwon2023efficient} to generate responses in each case, and we once again emphasize that our evaluation is on generations, requiring the model to not just know the correct answer to a given question, but also to (at least loosely) follow the formatting instructions given in the prompt so that we can parse the final answer.  In an effort to reduce the computational cost of evaluation and in recognition of the relatively small size of the models we consider, we restrict MMLU to \mmluhs, the set of high school topics, which we take to be the `easy' topics in that benchmark.  We use all such topics in order to avoid cherrypicking by subject.  The topics are as follows:  \texttt{high\_school\_biology}, \texttt{high\_school\_chemistry}, \texttt{high\_school\_computer\_science}, \texttt{high\_school\_european\_history}, \texttt{high\_school\_geography}, \\ \texttt{high\_school\_government\_and\_politics}, \texttt{high\_school\_macroeconomics}, \\ \texttt{high\_school\_mathematics}, \texttt{high\_school\_microeconomics}, \texttt{high\_school\_physics}, \\ \texttt{high\_school\_psychology}, \texttt{high\_school\_statistics}, \texttt{high\_school\_us\_history}, and \\ \texttt{high\_school\_world\_history}.
 For the generation evaluations (\boolq, \gsmk, and \mmluhs), we randomly select prompts and, for each prompt, generate 50 responses with temperature 1, computing the average accuracy per prompt in order to reduce the variance of our estimates of model quality.  For \gsmk~ and \mmluhs, we use Chain of Thought prompting \citep{wei2022chain} and for \boolq~ we use the common choice of 5-shot prompting \citep{team2025gemma,grattafiori2024llama}.

 \begin{table}[t]
    \centering
    \caption{Default Training Hyperparameters}
    \label{tab:training_hyperparams}
    \begin{adjustbox}{max width=\textwidth}
        \begin{tabular}{lc}
            \toprule
            \textbf{Hyperparameter} & \textbf{Value} \\
            \midrule
            Model & \qwen \\
            Tokenizer & \qwen \\
            Epochs & 2 \\
            Peak Learning Rate & $3 \times 10^{-5}$ \\
            Effective Batch Size & 256 \\
            Learning Rate Decay & None \\
            Warmup Steps & 0 \\
            dtype & fp16 \\
            Optimizer & Adam \\
            Adam $\beta_1$ & 0.9 \\
            Adam $\beta_2$ & 0.999 \\
            Adam $\epsilon$ & $10^{-8}$ \\
            Gradient Clipping & 1.0 \\
            \bottomrule
        \end{tabular}
    \end{adjustbox}
\end{table}

\begin{remark}\label{rmk:eval}
 We emphasize that our evaluation procedure is on tasks for which the model is not directly trained, in that the only training the model receives is on the \tulu~ dataset, which is a general-purpose dataset intended to improve question-answering, reasoning, and instruction following.  As such, there is no \emph{a priori} reason that performance on \boolq, \gsmk, or \mmluhs~ should necessarily improve after training, although we do observe that it does.  Some reason for this improvement is that the model learns to better follow the formatting instructions in the prompts, allowing answers to be correctly parsed.  Recent discussion of the effect that correct formatting has on benchmark performance has emphasized that a number of approaches that claim to enhance reasoning abilities in LMs actually do so by improving the model's ability to follow formatting instructions \citep{llmrl2025incorrect} and as such it is critical in such evaluations to ensure that the correct benchmark is used.  Note that this is not a problem in our work as we are not claiming that \bema\  improves reasoning abilities, but rather that it improves optimization; because \tulu~ is designed to improve the model's ability to follow formatting instructions the evaluations we conduct indeed measure that which we claim.  To reiterate, the goal of our evaluation suite is not to lift performance on the benchmarks qua benchmark performance, but rather to measure the model's quality when evaluated closed-loop in the precise regime that Gradient Variance Amplification (GVA) is a problem \citep{blockbutterfly2024}.
\end{remark}

\begin{table}[t]
    \centering
    \caption{Default \bema\  Hyperparameters}
    \label{tab:stabilizer_hyperparams}
    \begin{adjustbox}{max width=\textwidth}
        \begin{tabular}{lc}
            \toprule
            \textbf{Hyperparameter} & \textbf{Value} \\
            \midrule
            EMA Power $\kappa$ & 0.5 \\
            Bias Correction Power $\eta$ & 0.2 \\
            Multiplier $\gamma$ & 1.0 \\
            Lag $\rho$ & 10 \\
            Burn-in $\tau$ & 0 \\
            Frequency $\phi$ & 400 \\
            \bottomrule
        \end{tabular}
    \end{adjustbox}
\end{table}

\section{Further Empirical Results}\label[appendix]{sec:app_further_empirical_results}

We now describe in detail the additional empirical results and ablations we conducted that were briefly alluded to, but not extensively discussed, in the main text.  In particular, we conduct a thorough and exhaustive investigation of the sensitivity of \bema\  to its hyperparameters as well as those of training.  We then compare \bema\  to alternative stabilizers, such as \ouema\  and \dema, and conclude by evaluating \bema\  on \gemma~ and \llama~ in order to ensure that our results are not specific to \qwen.

\paragraph{Changing $\eta$ and $\kappa$.}  In \Cref{fig:lowema}, we display the train and test crossentropy losses as well as \boolq~ performance for different values of $\kappa$, which tunes the strength of the EMA intervention, and different values of $\eta$, which tunes the strength of the bias correction.  In general, we find that as $\kappa$ is increased (leaading to more aggressive averaging), the optimization trajectory can handle lower values of $\eta$ (stronger intervention of bias correction), with the maximal $\kappa$ we tried reaching the highest performance on a number of tasks.  A common pathology for small values of $\kappa$ is that the cross entropy losses improves significantly over both \ema\  and vanilla optimization, but the \boolq~ performance suffers, likely due to the misalignment between \tulu~ and \boolq~ resulting in a form of overfitting to the SFT task.  Indeed, for the highest performance of \bema\  on all generations tasks, which is achieved with $\kappa=0.5$, we find that the training and test crossentropy losses are substantially larger than those achieved with smaller $\kappa$, again pointing to the misalignment between SFT task and generation.  It is clear, however, that \bema\  imparts considerable advantage in terms of acceleration relative to \ema\  and vanilla optimization.

\paragraph{Changing the learning rate through learning rate decay.}  In \Cref{fig:minlr_multiples,fig:minlr_multiples_losses,fig:minlr_multiples_benchmarks}, we investigate the effect of learning rate decay on \bema\  and \ema.  In the latter two figures, we plot the optimal throughout training losses in crossentropy (for both train and test sets) as well as performance on the considered generations tasks.  We see that the optimal performance accross the board occurs \emph{without any learning rate decay but with stabilization via} \bema.  That said, the effect learning rate decay has on the optimization trajectories without stabilizing, with \ema, and with \bema\  can all be observed in \Cref{fig:minlr_multiples} and appears to be present, but small.

\paragraph{Effect of \bema\  hyperparameters.} In addition to the $\kappa$ and $\eta$ hyperparameters of \bema\  described above, three other choices can potentially affect the performance of \bema\ : (1) the choice of the burn-in time $\tau$; (2) the choice fo the lag $\rho$; and (3) the choice of update frequency $\phi$.  In \Cref{fig:updateafter_heatmaps}, we demonstrate that, at least in our setup, choosing $\tau = 0$ (no burn-in) leads to by far the best performance, with waiting 500 or 1000 steps leading to substantial degradation.  We conjecture that this is a general phenomenon when starting with pre-trained models and aligns with earlier work suggesting that post-training approximately occurs in a convex setting \citep{malladi2023kernel}.  Beyond $\tau$, we see in \Cref{fig:scalinglag} that the choice of lag $\rho$ has minimal effect on the stabilization, which is unsurprising considering that after sufficiently many steps, the lag does not meaningfully affect the update itself.

Of these three hyperparameters, the update frequency $\phi$ is by far the most signifcant.  In \Cref{fig:updatefreqs,fig:updatefreqs_heatmaps} we investigate the effect that updating significantly more frequently has on \ema\  and \bema.  We find that for small values of $\phi$ (very frequent updates), the convergence of \bema\  is considerably accelerated, leading to much lower train and test losses.  The phenomenon whereby the model overfits to the SFT task and performance (after sufficient training) declines on \boolq~ is significantly magnified by this acceleration as well.  In all cases, however, we continue to see significant acceleration benefits of \bema\  relative to \ema\  and vanilla optimization.

\paragraph{Effect of batch size.} In \Cref{fig:batchsize}, we demonstrate that \bema\  continues to provide significant acceleration after the batch size is doubled.  Indeed, one might expect \ema\  to provide less benefit in this case due to the reduction of stochasticity in the gradients, but the factor of 2 increase to an effective batch size of 512 does not appear to impact the performance overmuch and we continue to see gains from \bema\  relative to \ema\  and vanilla optimization.

\paragraph{Comparison to alternative stabilizers.} In \Cref{fig:ouema}, we display plots analogous to those of \Cref{fig:lowema}, but for \ouema\  instead of \bema.  We vary $\kappa$ from removing all averaging up to $\kappa = 0.5$ and consider a number of values of $\eta$, comparing the performance of \ouema\  to that of vanilla optimization, \ema, and \bema\  with tuned $\eta$ value.  We find that \bema\  signifcantly outperforms \ouema\  accross the board, and \ouema\  tends to outperform \ema\  with respect to acceleration.

In addition to comparing \bema\  to \ouema, we also consider the Double Exponential Moving Average (\dema), which updates according to the following rule:
\begin{align}
    \thetadema_{t} &= 2 \cdot \thetaema_t - \thetaematwo_t \\
    \thetaema_{t} &= (1 - \beta_t) \cdot \thetaema_{t-1} + \beta_t \cdot \theta_t \\
    \thetaematwo_{t} &= (1 - \beta_t) \cdot \thetaematwo_{t-1} + \beta_t \cdot \thetaema_{t},
\end{align}
i.e., $\text{\dema} = 2 \cdot \text{\ema} - \text{\ema}\left( \text{\ema} \right)$.  This stabilizer comes out of the finance literature \citep{mulloy1994smoothing} and has recently been applied to training neural networks as an alternative to \ema\  \citep{chen2023bidirectional}.  In \Cref{fig:dema}, we compare \dema~ to \bema\  and \ouema\  on the quadratic optimization problem of \Cref{fig:Fig-2}, the crossentropy losses, and the generations losses \boolq, \gsmk, and \mmluhs.  In the quadratic case, we see that \dema~ initially improves on \ema, but then eventually matches the performance thereof, and is uncompetitive relative to \bema, as the theory predicts.  In the case of crossentropy, \dema~ improves on \ouema\  and \ema, but has inferior training loss to \bema, and less acceleration than the same in terms of test loss.  Finally, \bema\  continues to outperform \dema~ on the generations tasks, leading to substantial acceleration and sometimes superior peak performance.

\paragraph{Performance on \gemma~ and \llama.} Finally, in order to ensure that our results are not specific to \qwen, we also evaluate \bema\  on \gemma~ and \llama~ in \Cref{fig:gemma,fig:llama}.  We find that \bema\  continues to provide significant acceleration relative to \ema\  on train and test crossentropy losses in both models, as well as in the generations tasks in the default setup with $\kappa = 0.5$.  On the other hand, in several of these examples, we find that vanilla optimization without stabilization actually outperforms both \ema\  and \bema, especially with \llama; in \gemma~ with \boolq, however, we continue to see gains.  Further investigation revealed that \gemma~ and especially \llama~ continue to have problems following instructions, leading to wrong answers by default, even after finetuning on \tulu; thus while these results are in general encouraging for \bema, a more complete evaluation on these models with a different suite of tasks that is more commensurate to their capabilities is necessary to firm up these conclusions, which we leave for future work.  We conclude by noting that even without averaging, i.e., setting $\kappa = 0$, leads to significant improvements in \boolq~ performance when using \bema\  over vanilla training without stabilization, especially for \gemma.

\begin{figure}[t]   
	\centering
    \subfigure[]{  
		\includegraphics[width=0.29\textwidth]{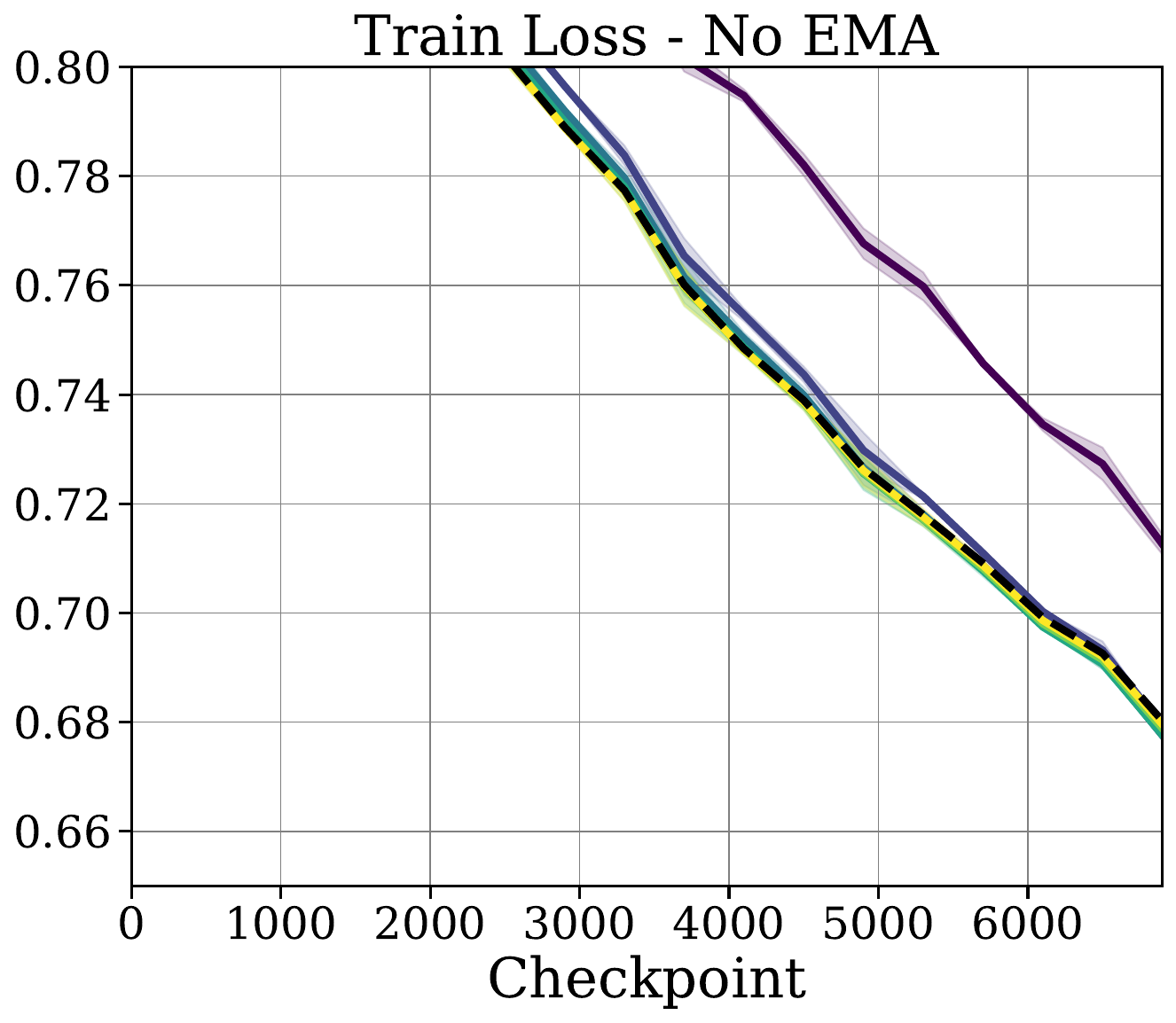}
		\label{sfig:noema_trainloss} 
	} \hfill \subfigure[]{
        \includegraphics[width=0.29\textwidth]{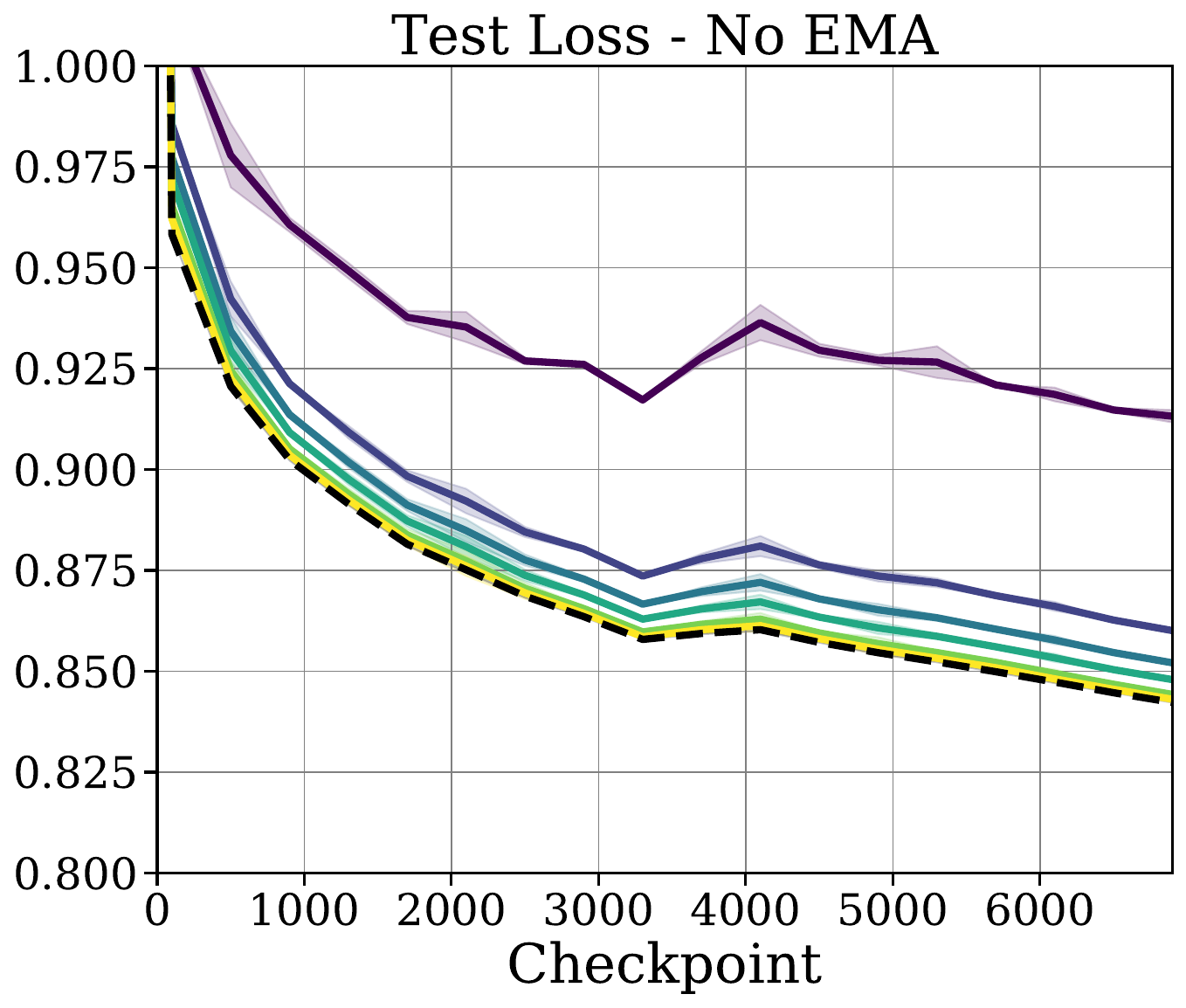} 
        \label{sfig:noema_testloss}
    } \hfill \subfigure[]{ 
        \includegraphics[width=0.29\textwidth]{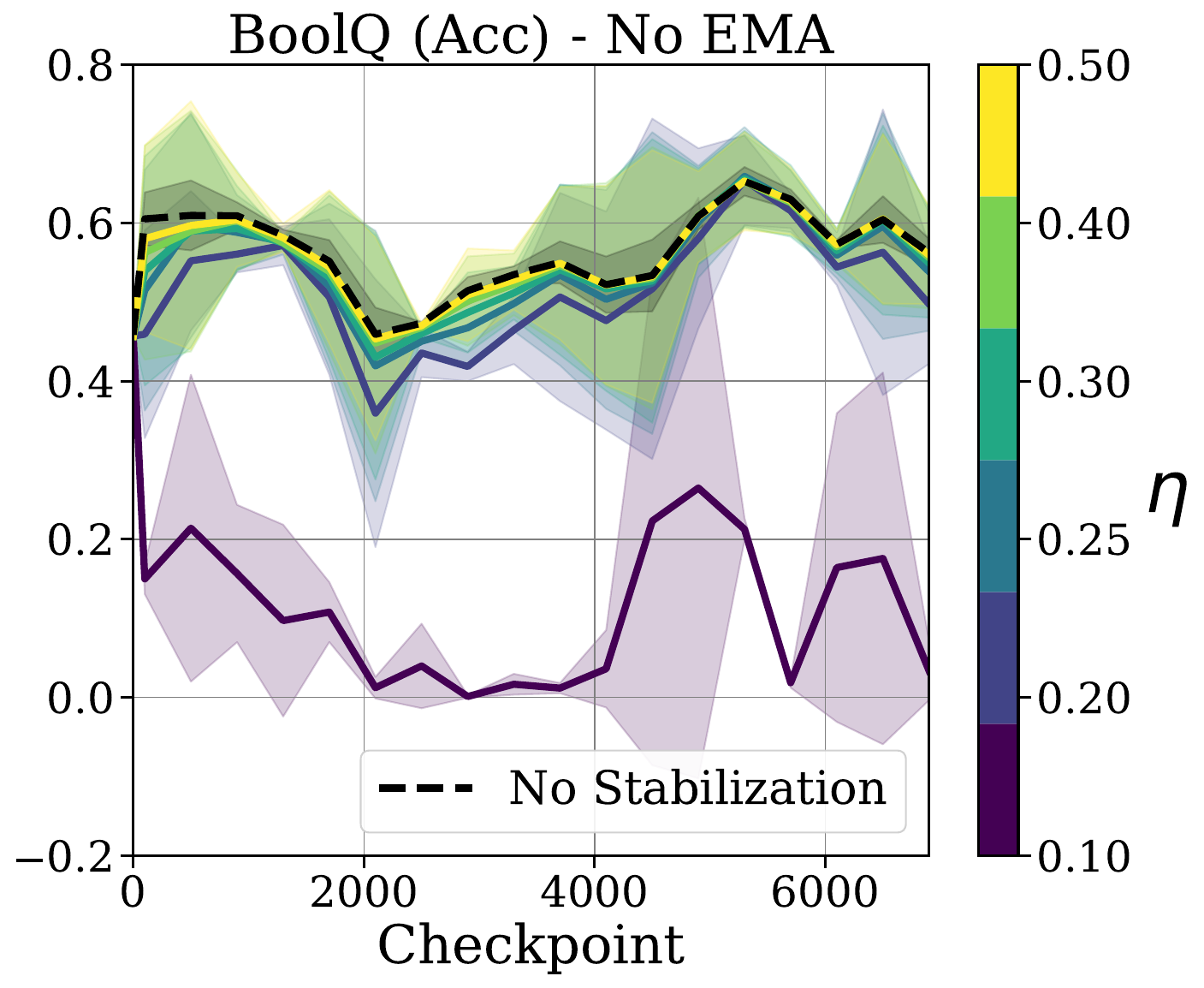} 
        \label{sfig:noema_boolq}
    } \\
   \subfigure[]{  
		\includegraphics[width=0.29\textwidth]{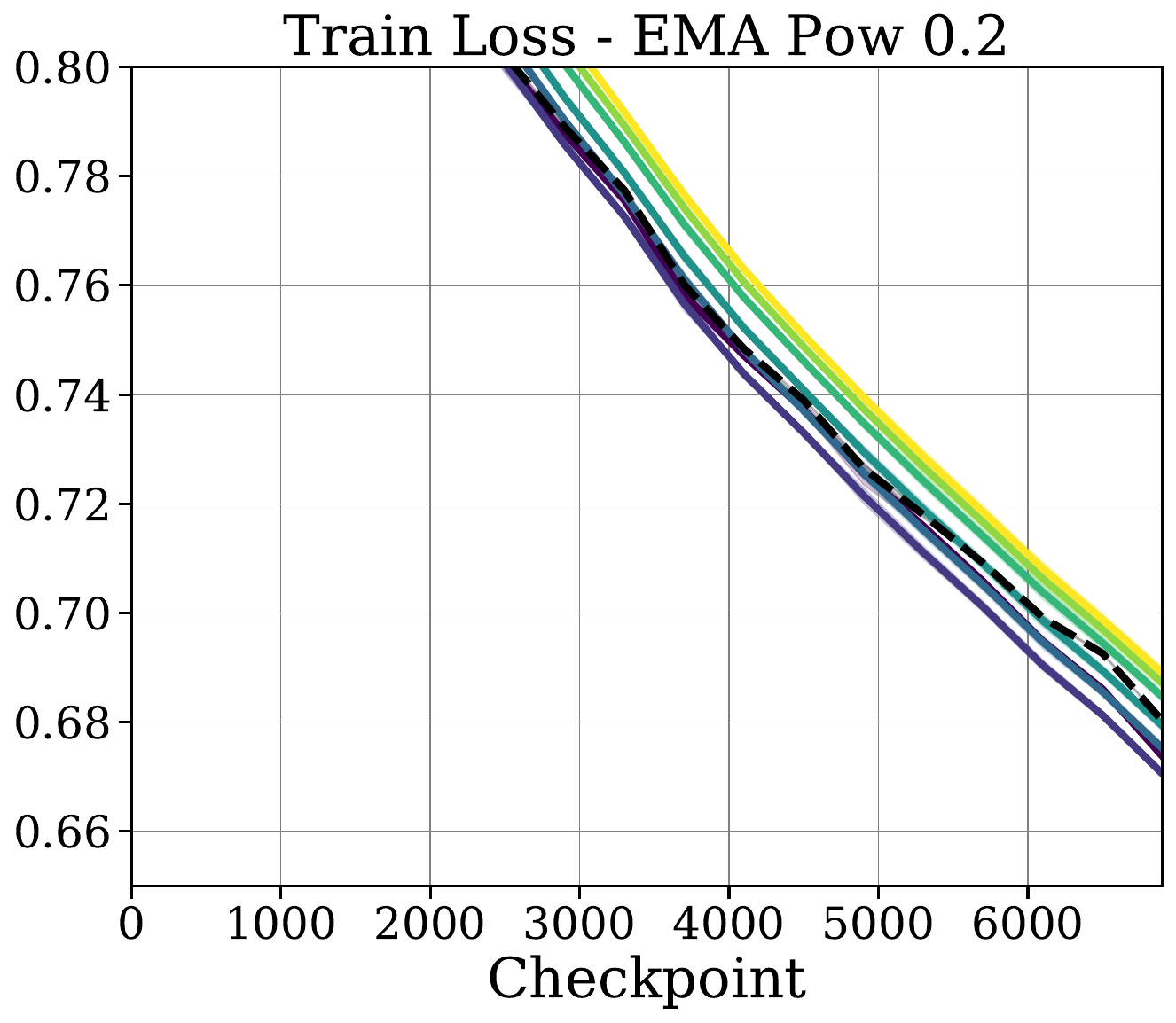}
		\label{sfig:lowema2_trainloss} 
	} \hfill \subfigure[]{
        \includegraphics[width=0.29\textwidth]{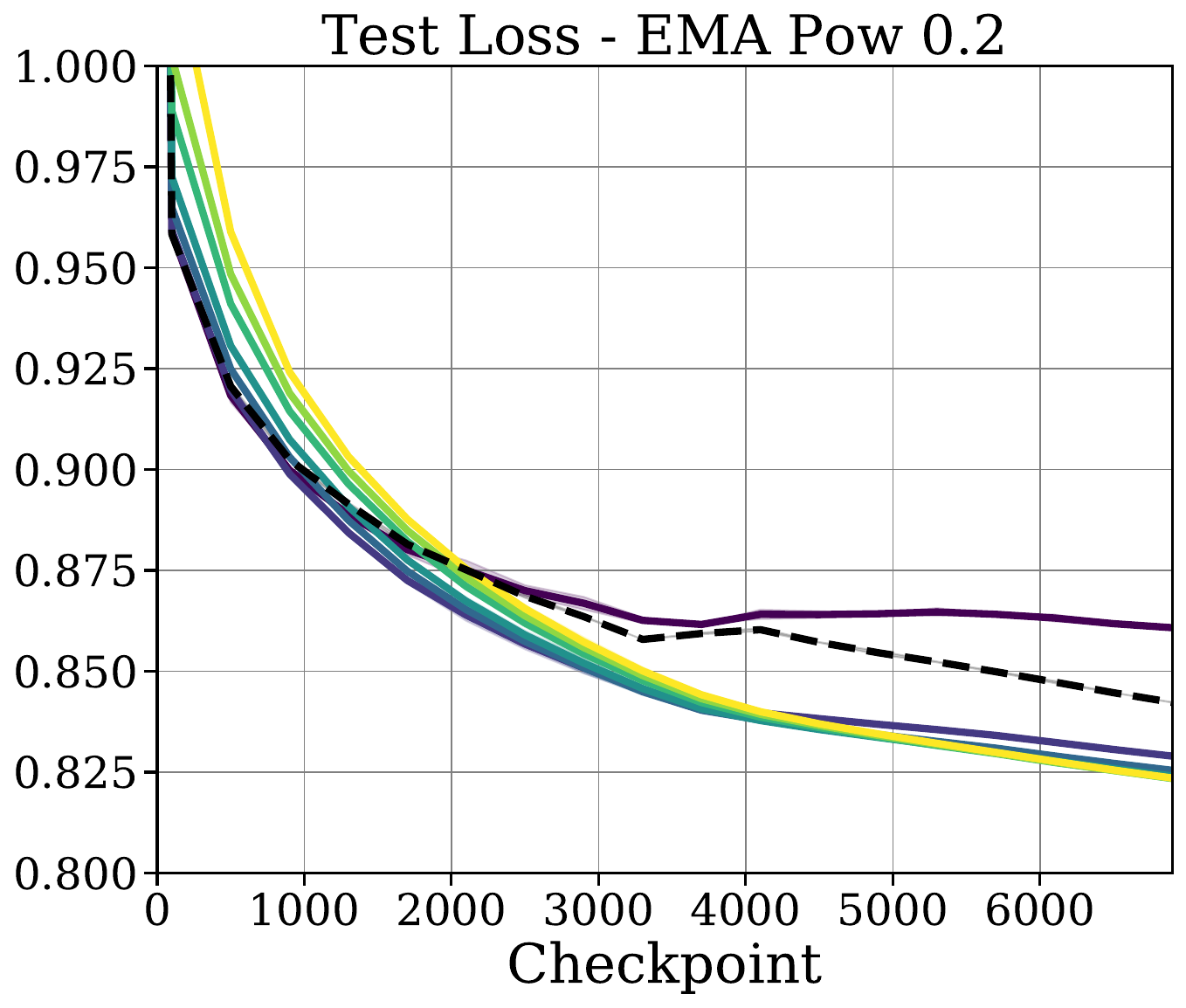} 
        \label{sfig:lowema2_testloss}
    } \hfill \subfigure[]{ 
        \includegraphics[width=0.29\textwidth]{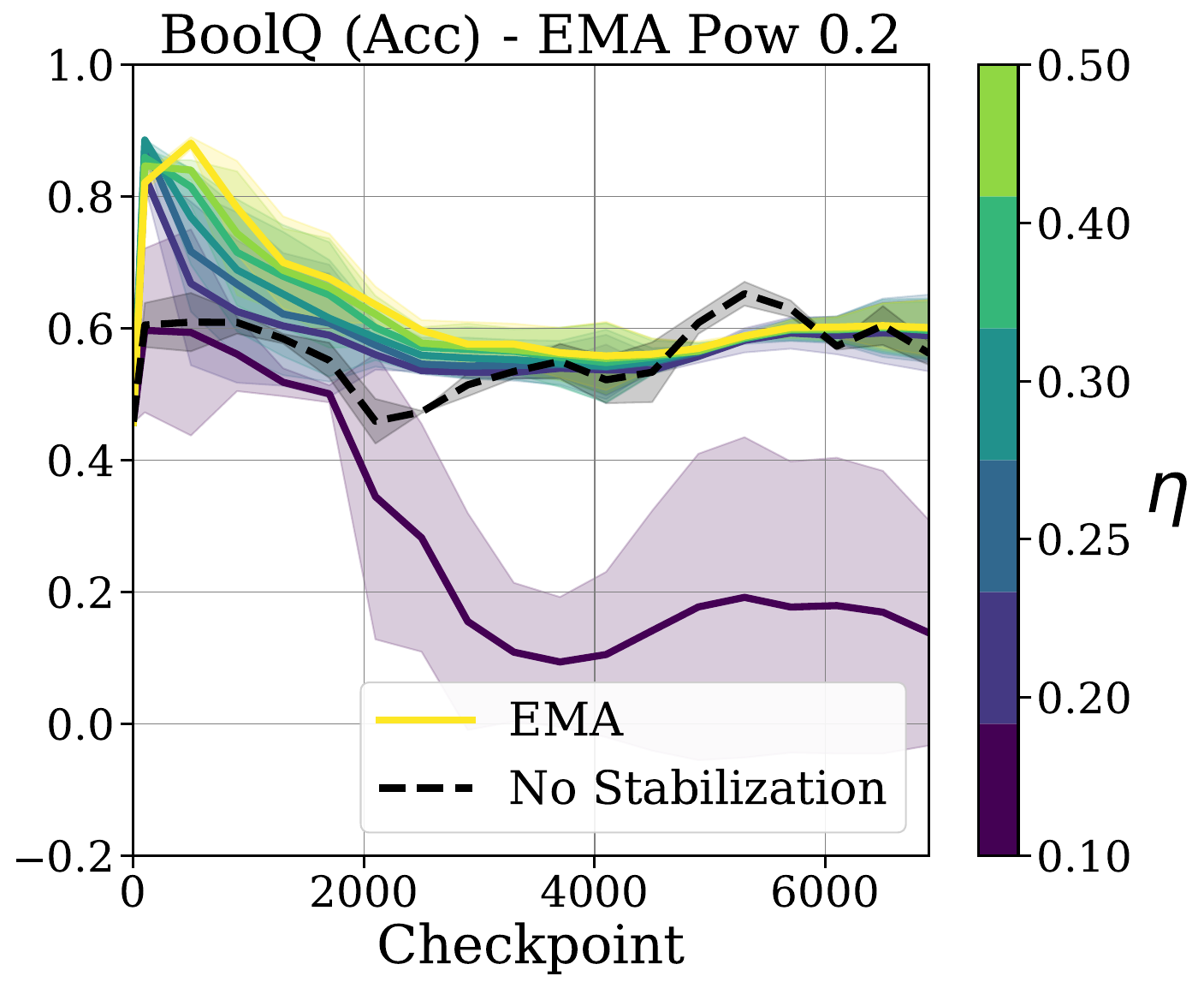} 
        \label{sfig:lowema2_boolq}
    } \\
    \subfigure[]{  
		\includegraphics[width=0.29\textwidth]{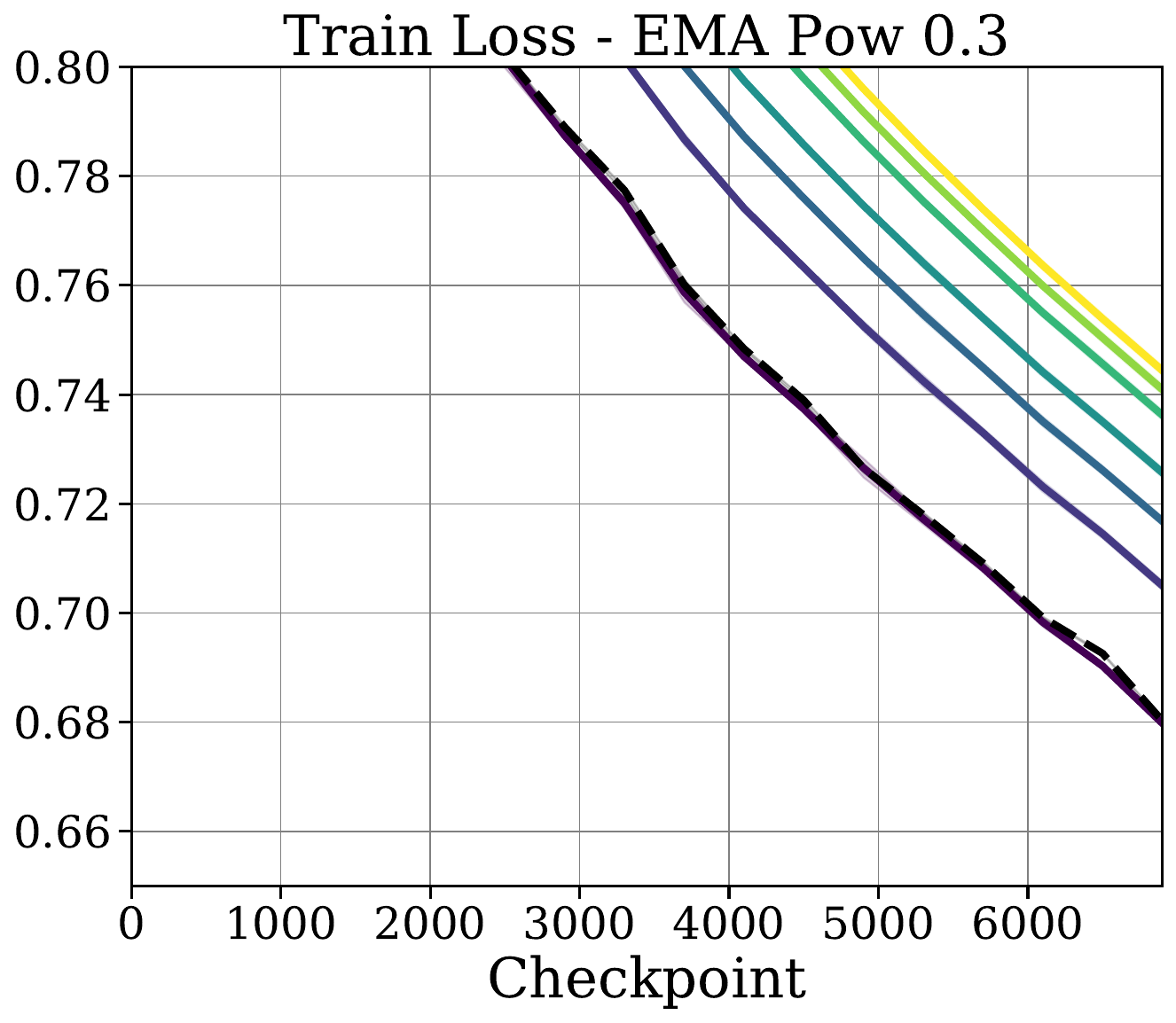}
		\label{sfig:lowema3_trainloss} 
	} \hfill \subfigure[]{
        \includegraphics[width=0.29\textwidth]{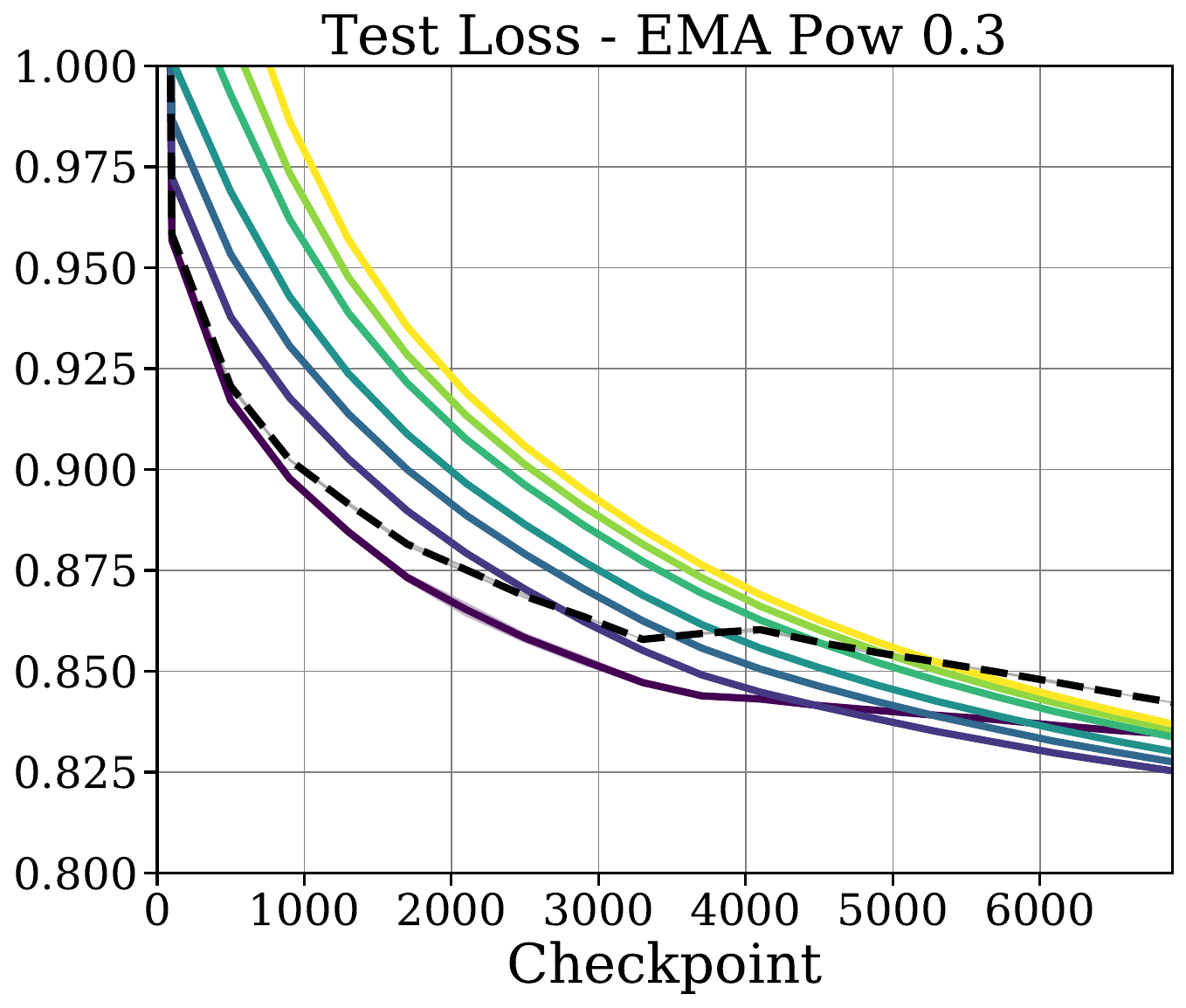} 
        \label{sfig:lowema3_testloss}
    } \hfill \subfigure[]{ 
        \includegraphics[width=0.29\textwidth]{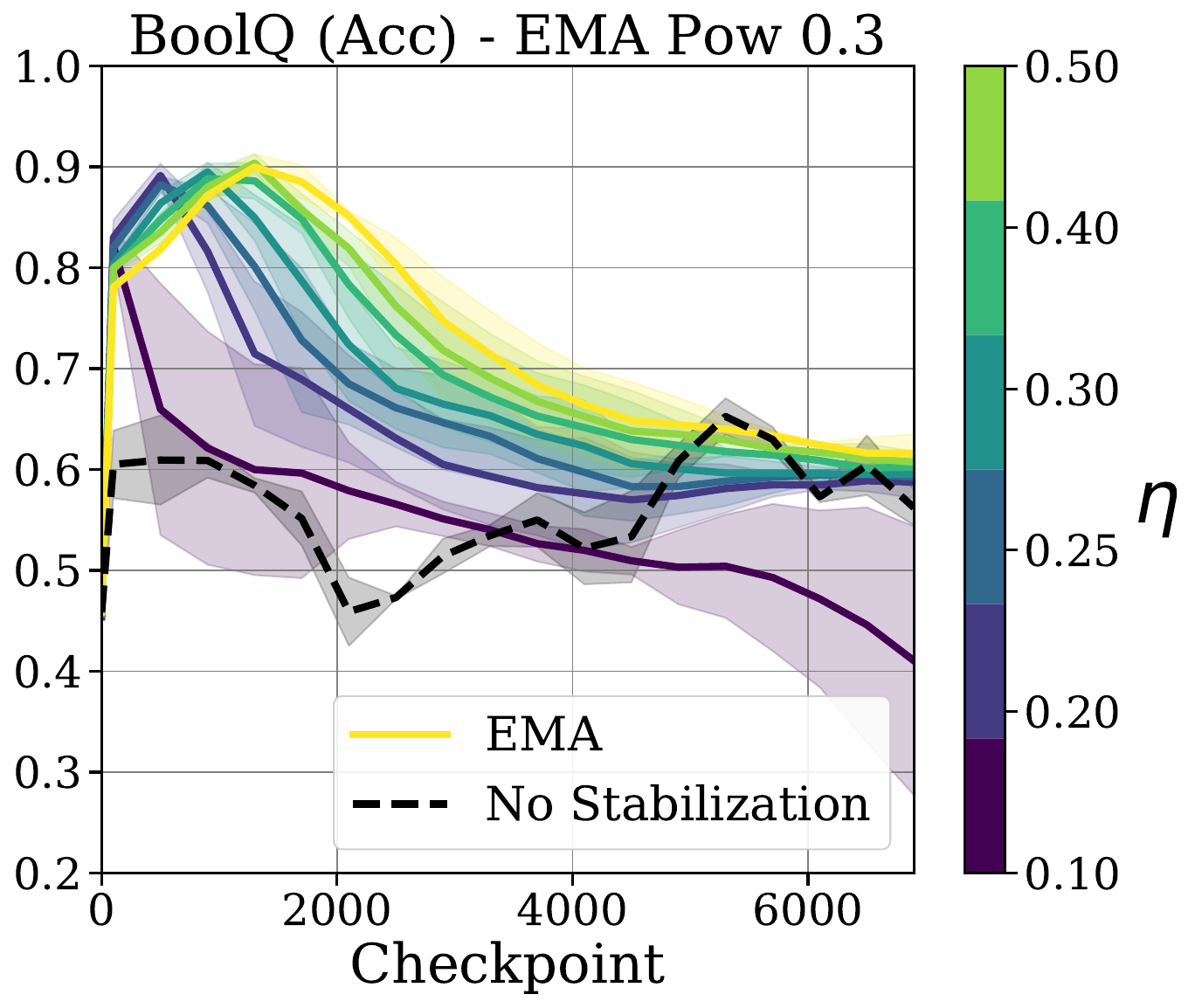} 
        \label{sfig:lowema3_boolq}
    } \\
    \subfigure[]{  
		\includegraphics[width=0.29\textwidth]{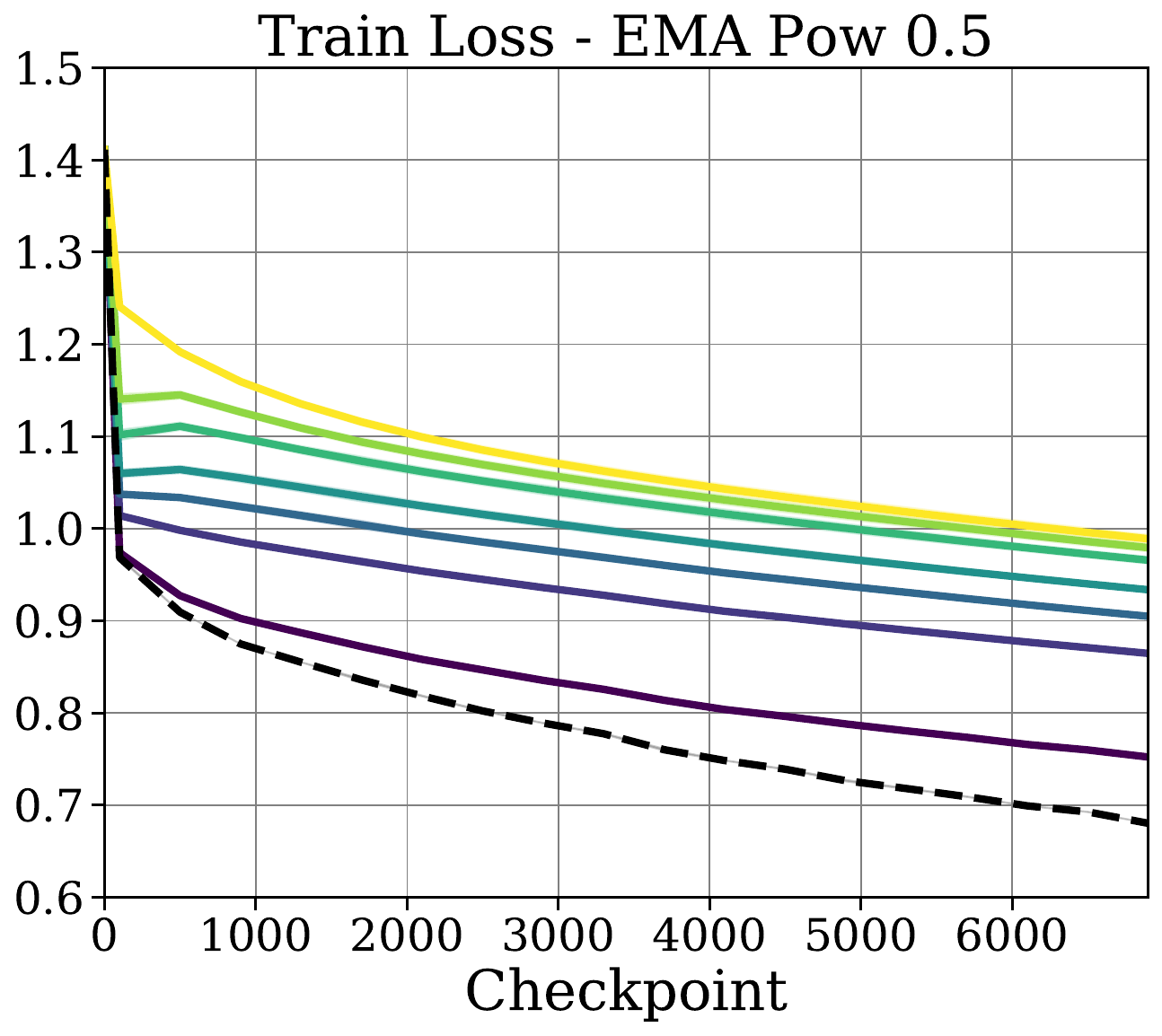}
		\label{sfig:lowema5_trainloss} 
	} \hfill \subfigure[]{
        \includegraphics[width=0.29\textwidth]{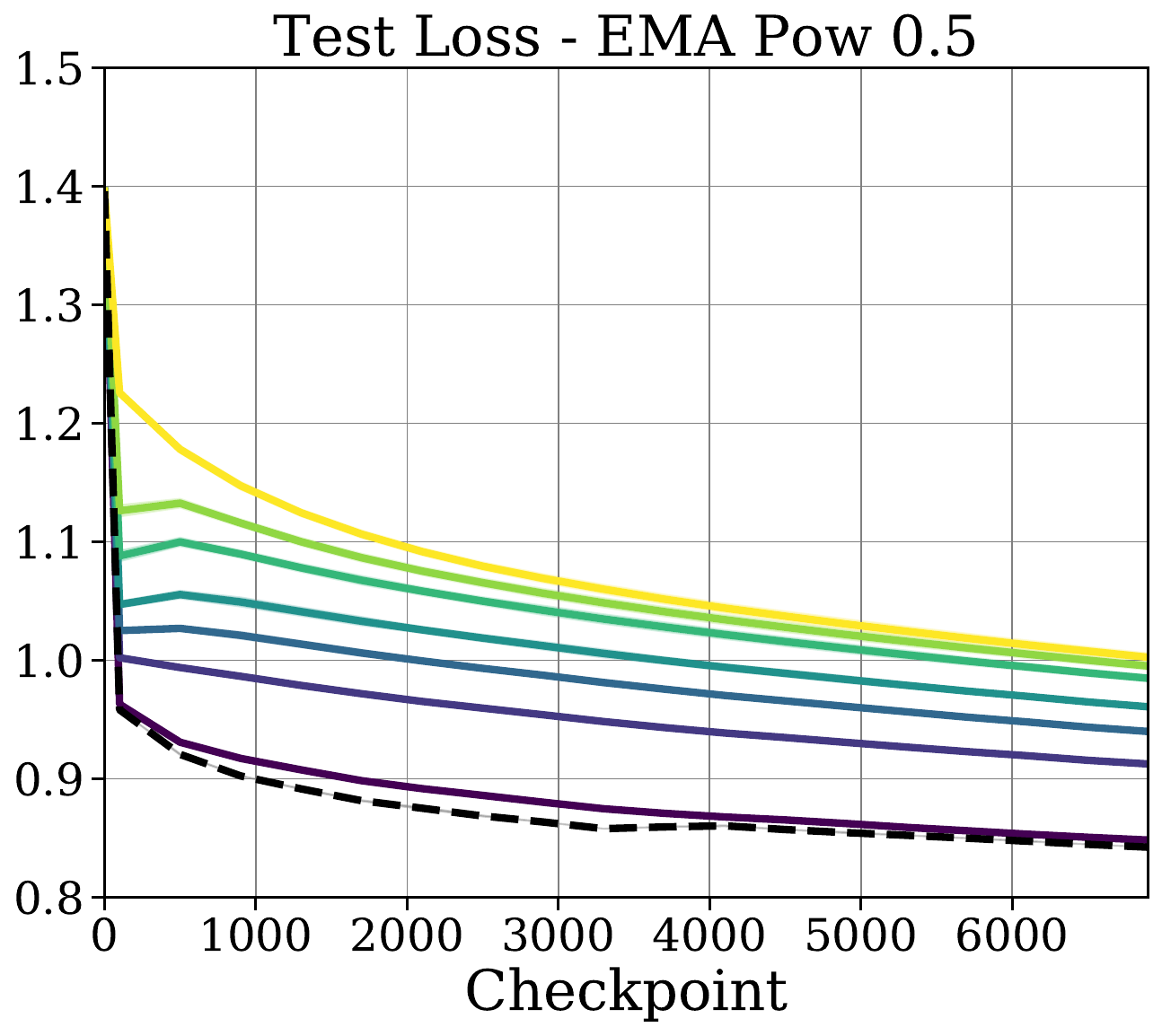} 
        \label{sfig:lowema5_testloss}
    } \hfill \subfigure[]{ 
        \includegraphics[width=0.29\textwidth]{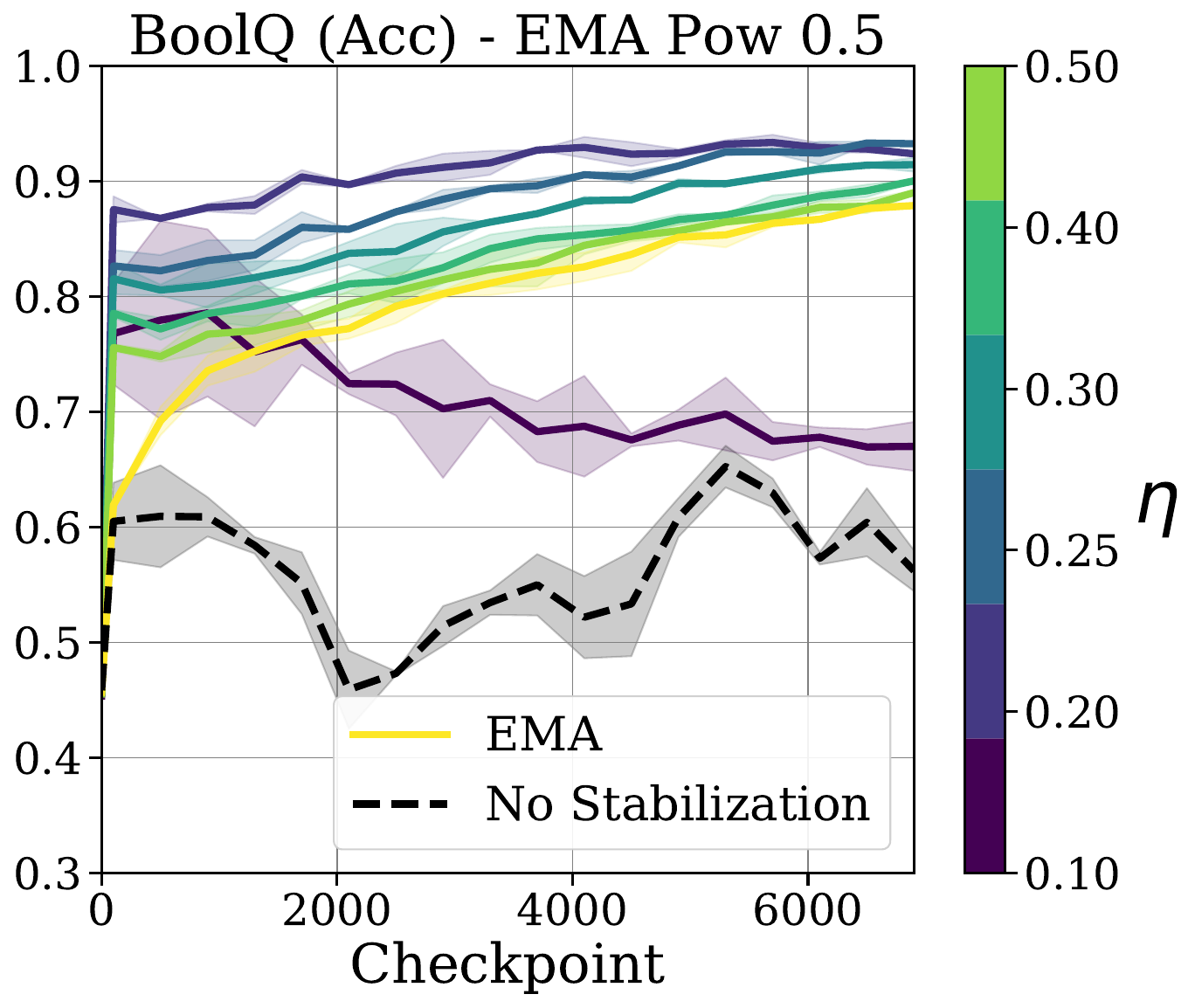} 
        \label{sfig:lowema5_boolq}
    } \\
	\caption{
        Performance of \bema\  with $\kappa = 0.0$ (no EMA) and other values $\kappa$ with respect to \textbf{(first column)} train loss, \textbf{(second column)} test loss, and \textbf{(third column)} \boolq\ accuracy.  \bema\  performance generally increases with $\kappa$ as training allows for lower values of $\eta$, leading to a stronger intervention of \bema\ over \ema.
    }  
	\label{fig:lowema} 
\end{figure}

\begin{figure}[t]
    \centering
    \subfigure[]{
        \includegraphics[width=0.29\textwidth]{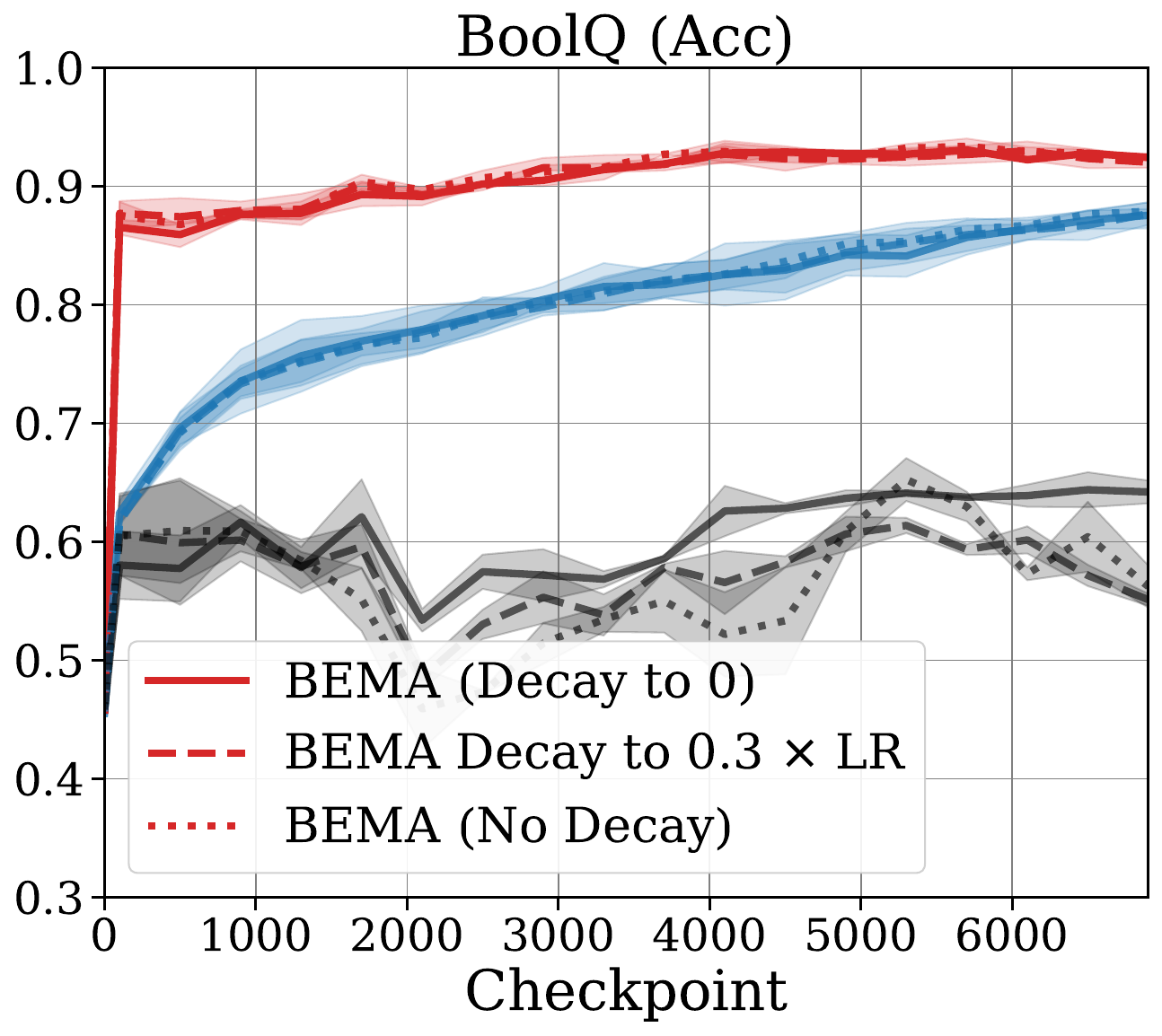}
        \label{sfig:boolq_minlr_mult}
    } \hfill
    \subfigure[]{
        \includegraphics[width=0.29\textwidth]{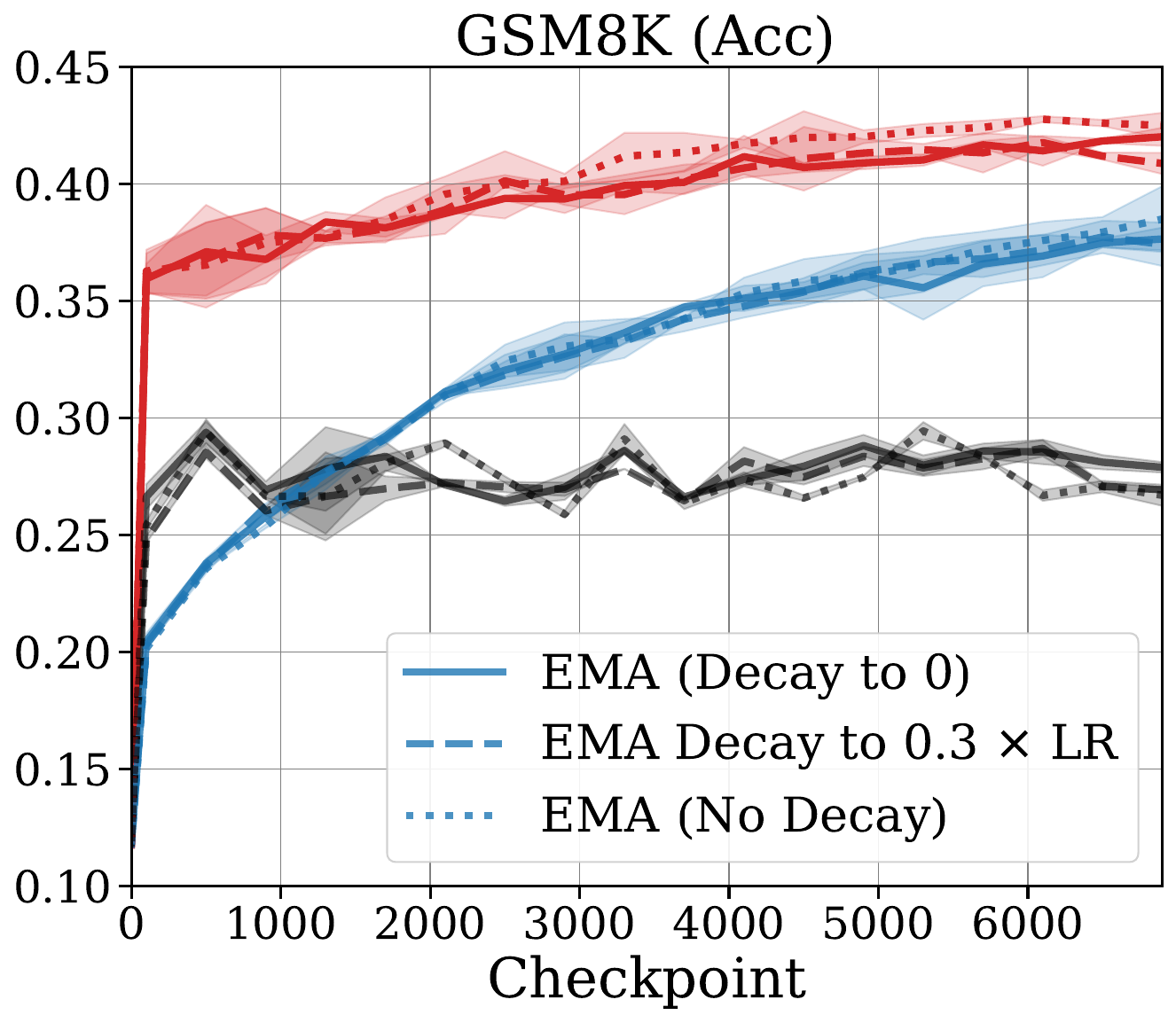}
        \label{sfig:gsm8k_minlr_mult}
    } \hfill
    \subfigure[]{
        \includegraphics[width=0.29\textwidth]{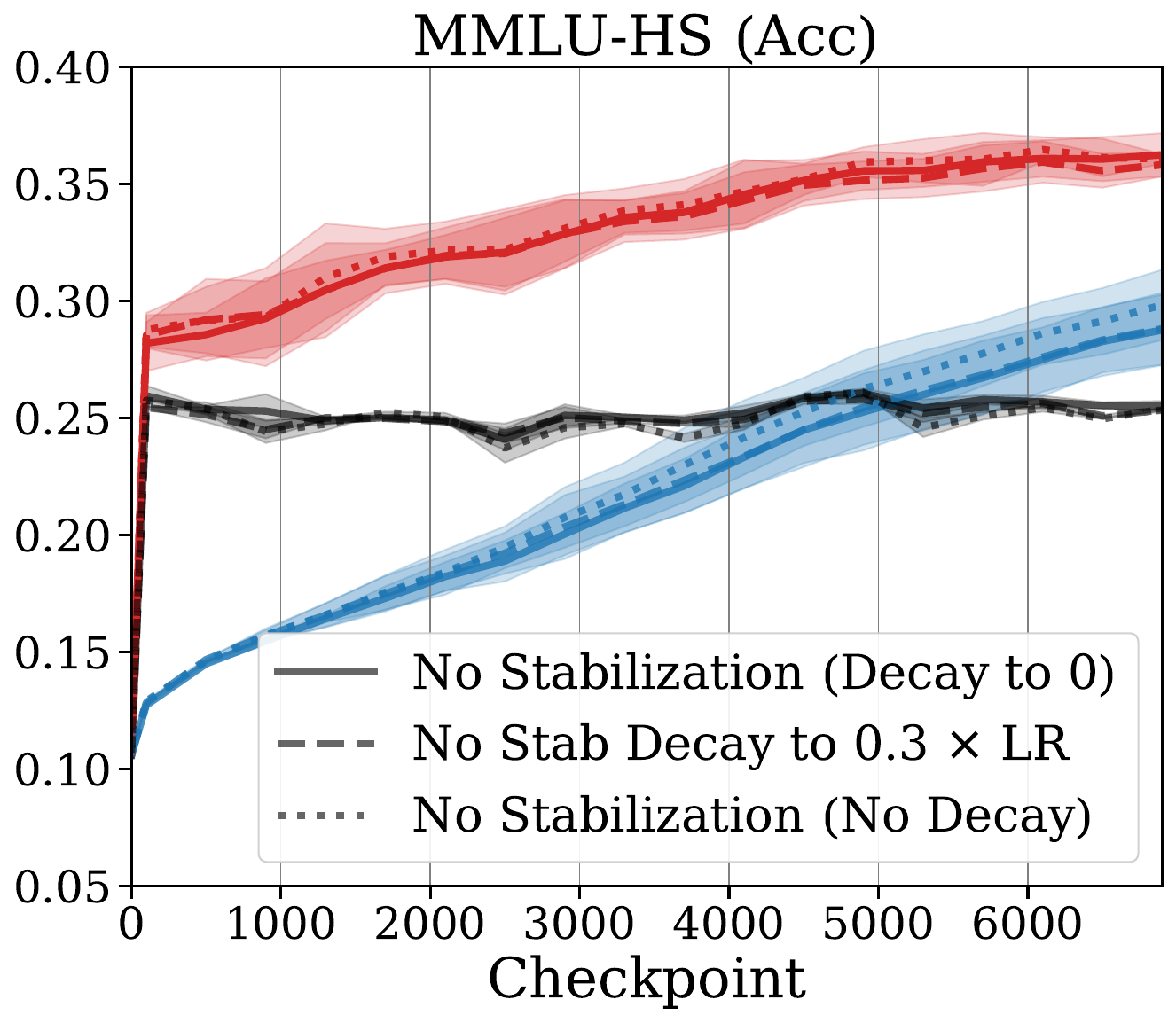}
        \label{sfig:mmluhs_minlr_mult}
    } 
    \caption{Demonstration of the effect of learning rate decay on training with stabilization, both with \ema\  and \oumle.  Evaluations on \textbf{(a)} \boolq, \textbf{(b)} \gsmk, and \textbf{(c)} \mmluhs~ suggest that \bema\  robustly improves on \ema\  performance for a range of learning rates.}
    \label{fig:minlr_multiples}
\end{figure}

\begin{figure}[t]
    \centering
    \subfigure{
        \includegraphics[width=\textwidth]{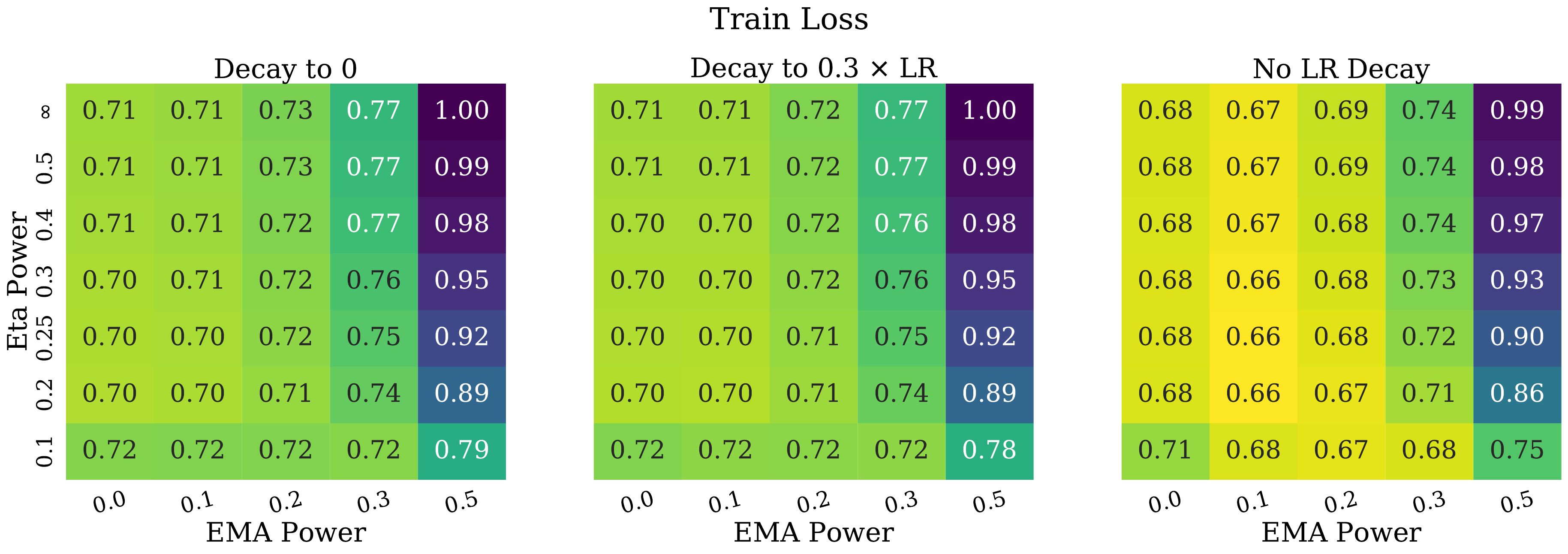}
        \label{sfig:minlr_train_loss}
    } \\
    \subfigure{
        \includegraphics[width=\textwidth]{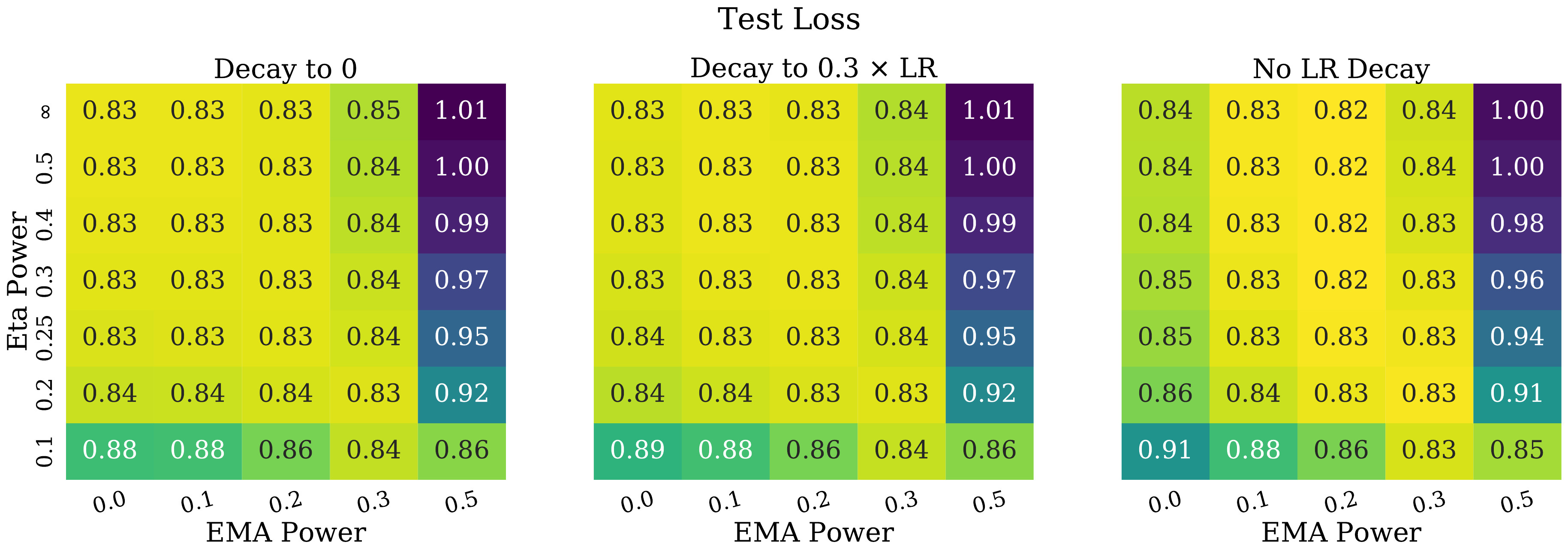}
        \label{sfig:minlr_test_loss}
    }
    \caption{Effect of $\kappa$ and $\eta$ on best over training trajectory train \textbf{(top)} and test \textbf{(bottom)} loss for \bema\  for decay to 0 \textbf{(left)}, decay to 0.3 times peak learning rate \textbf{(middle)}, and no decay \textbf{(right)}.}
    \label{fig:minlr_multiples_losses}
\end{figure}

\begin{figure}[t]
    \centering
    \subfigure{
        \includegraphics[width=\textwidth]{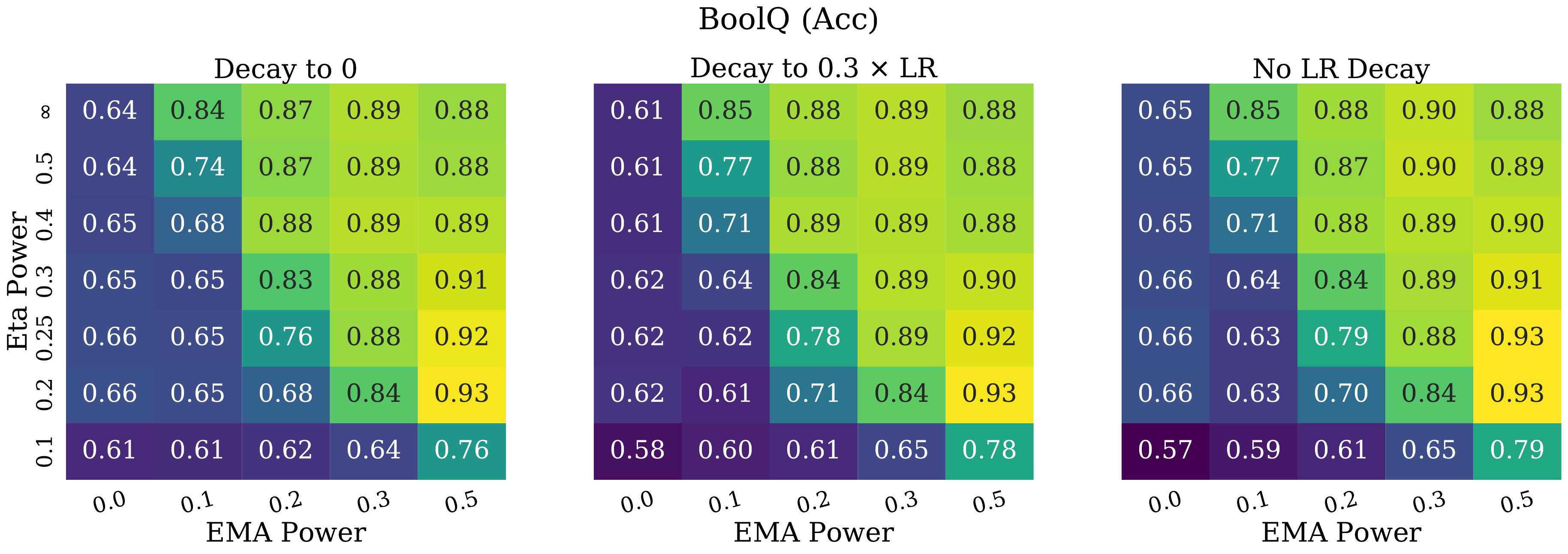}
        \label{sfig:minlr_boolq}
    } \\
    \subfigure{
        \includegraphics[width=\textwidth]{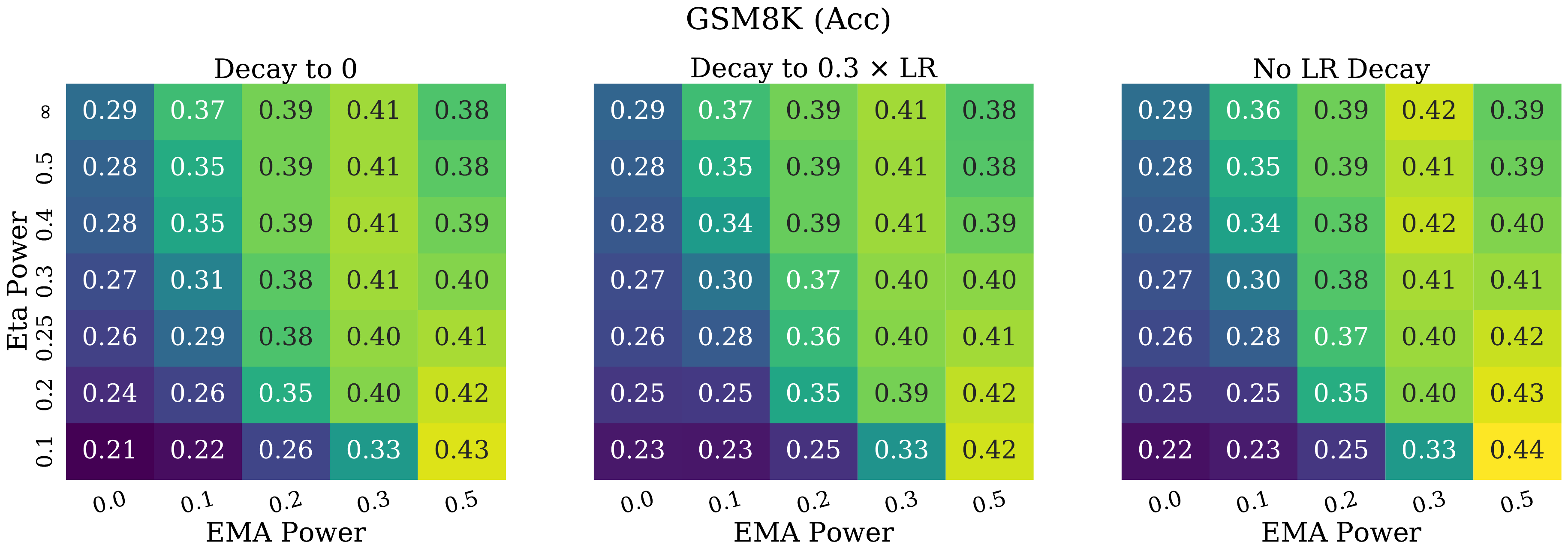}
        \label{sfig:minlr_gsmk}
    } \\
    \subfigure{
        \includegraphics[width=\textwidth]{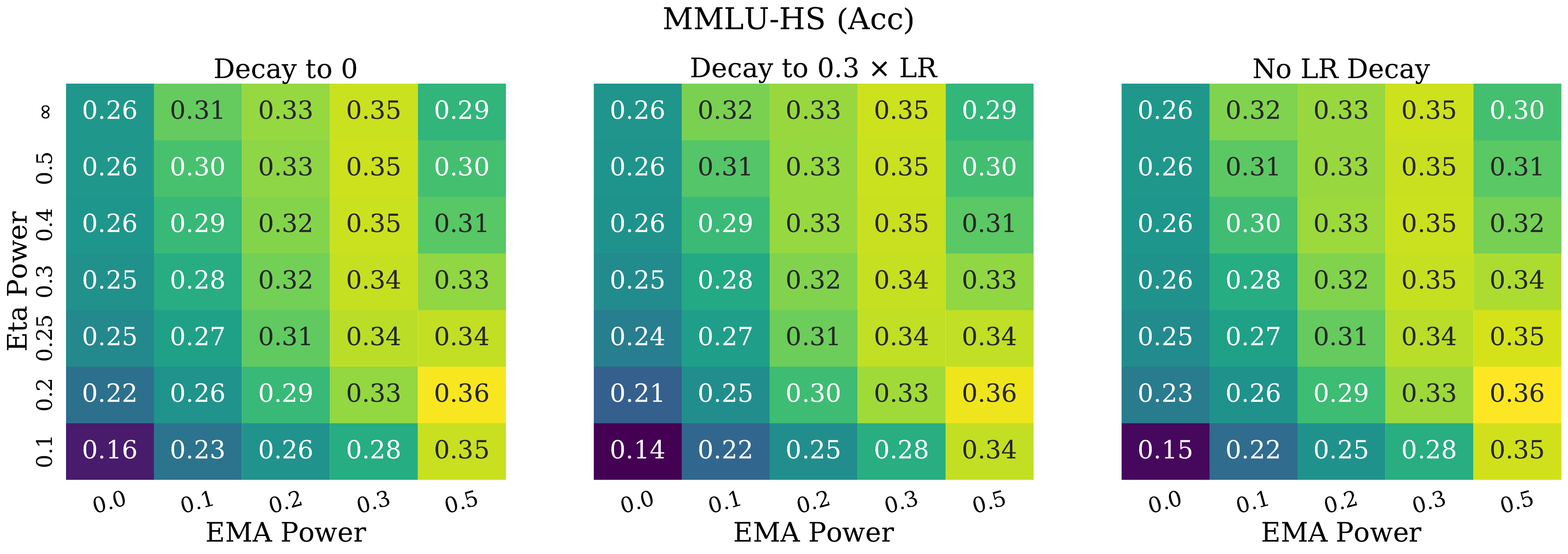}
        \label{sfig:minlr_mmluhs}
    }
    \caption{Effect of $\kappa$ and $\eta$ on optimal throughout training performance on \boolq~ \textbf{(top)}, \gsmk~ \textbf{(middle)}, and \mmluhs~ \textbf{(bottom)} for \bema\  for decay to 0 \textbf{(left)}, decay to 0.3 times peak learning rate \textbf{(middle)}, and no decay \textbf{(right)}.  Compared to pure \ema\  ($\eta = \infty$), \bema\  not only accelerates convergence, but can lead to better performance in the long run.}
    \label{fig:minlr_multiples_benchmarks}
\end{figure}

\begin{figure}[t]
    \centering
    \subfigure{
        \includegraphics[width=\textwidth]{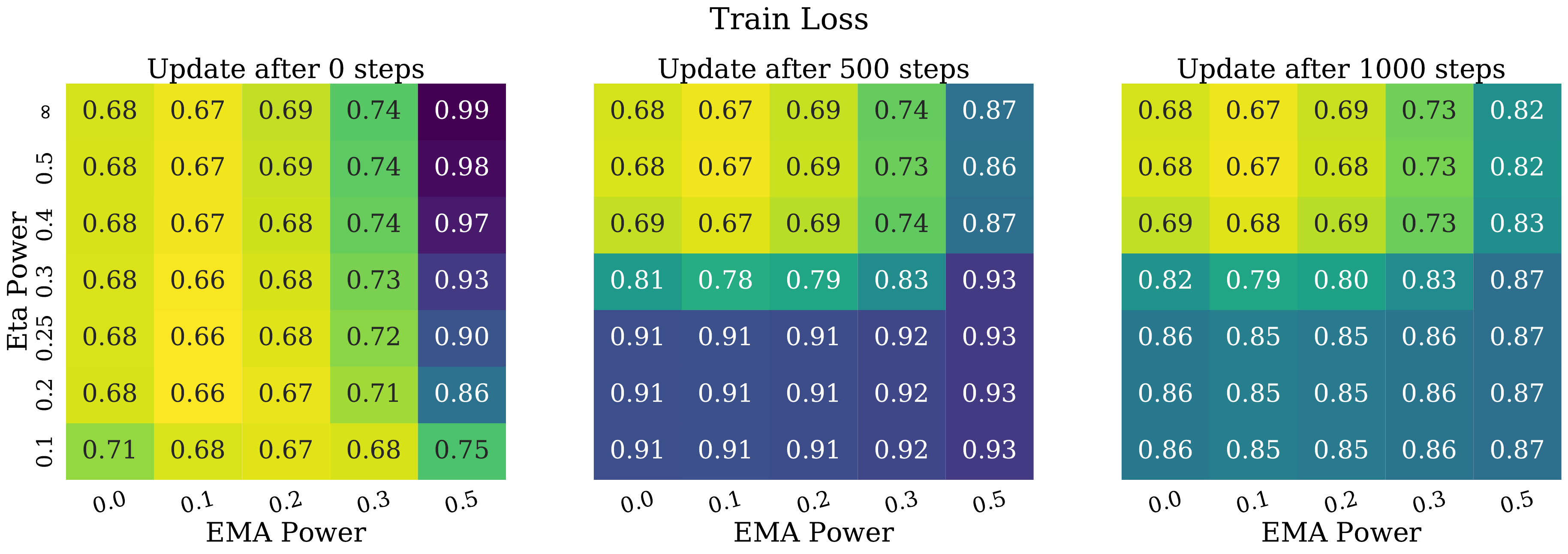}
        \label{sfig:updateafter_train_loss}
    } \\
    \subfigure{
        \includegraphics[width=\textwidth]{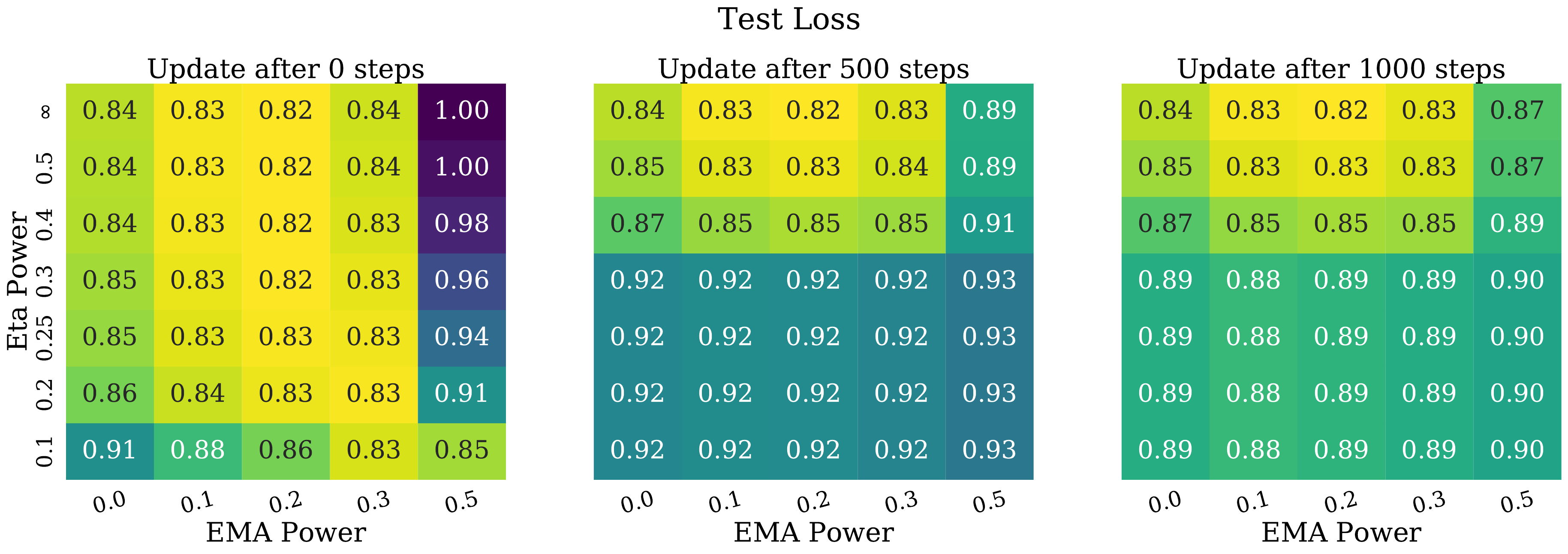}
        \label{sfig:updateafter_test_loss}
    } \\
    \subfigure{
        \includegraphics[width=\textwidth]{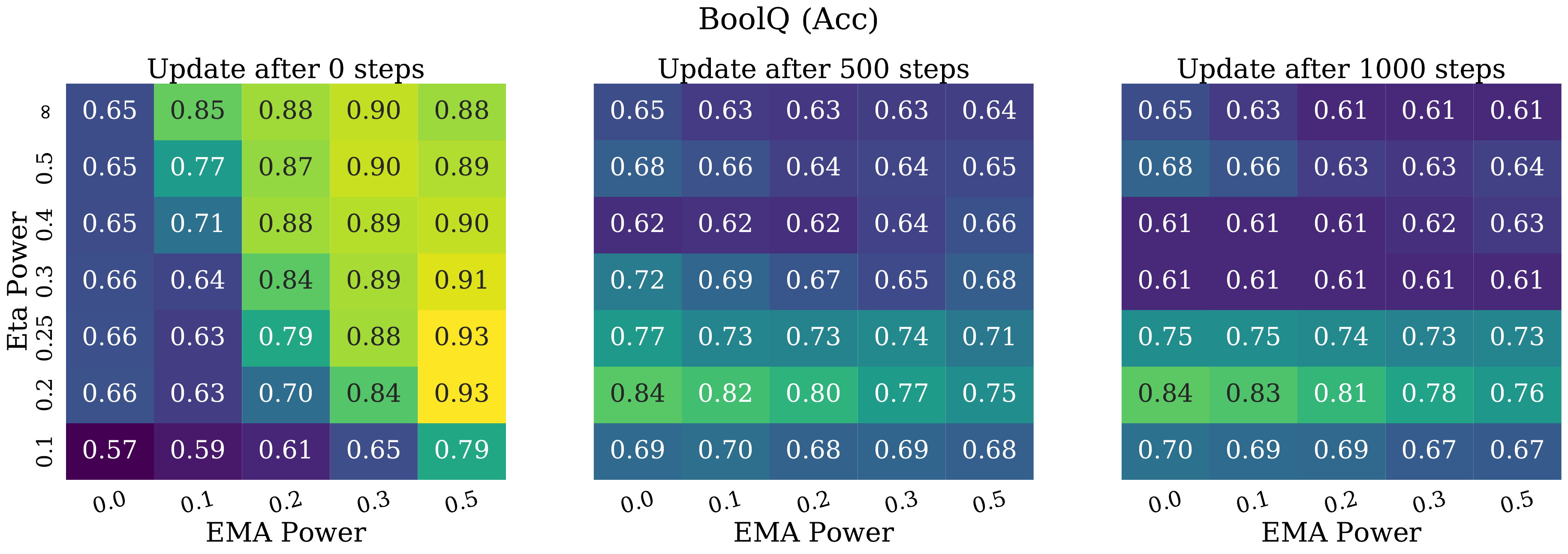}
        \label{sfig:updateafter_boolq}
    }
    \caption{Effect of the choice of $\theta_0$ for different values of $\kappa$ and $\eta$ on optimal throughought training values of train loss \textbf{(top)}, test loss \textbf{(middle)}, and \boolq~ \textbf{(bottom)} for \oumle.  In general, choosing $\theta_0$ \textbf{(left)} to be the weights of the pre-trained model and immediately applying \bema\  leads to the best performance as opposed to waiting 500 \textbf{(middle)} or 1000 \textbf{(right)} steps before stabilizing.  Compared to pure \ema\  ($\eta = \infty$), which is the top row of each heatmap, \bema\  can lead to improved performance.
    }
    \label{fig:updateafter_heatmaps}
\end{figure}

\begin{figure}[t]
    \centering
    \subfigure[]{
        \includegraphics[width=0.29\textwidth]{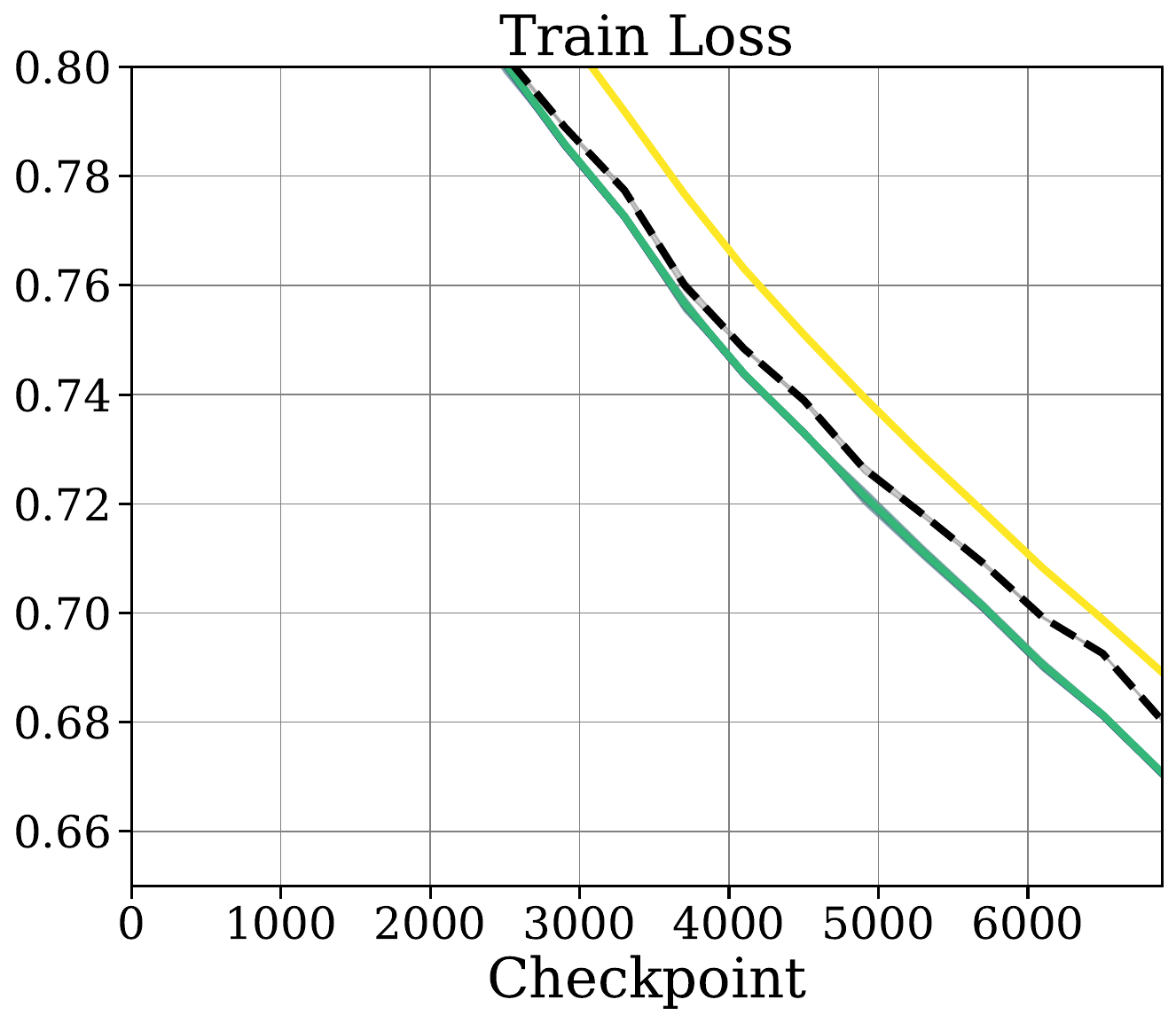}
        \label{sfig:scalinglag_train_loss}
    } \hfill
    \subfigure[]{
        \includegraphics[width=0.29\textwidth]{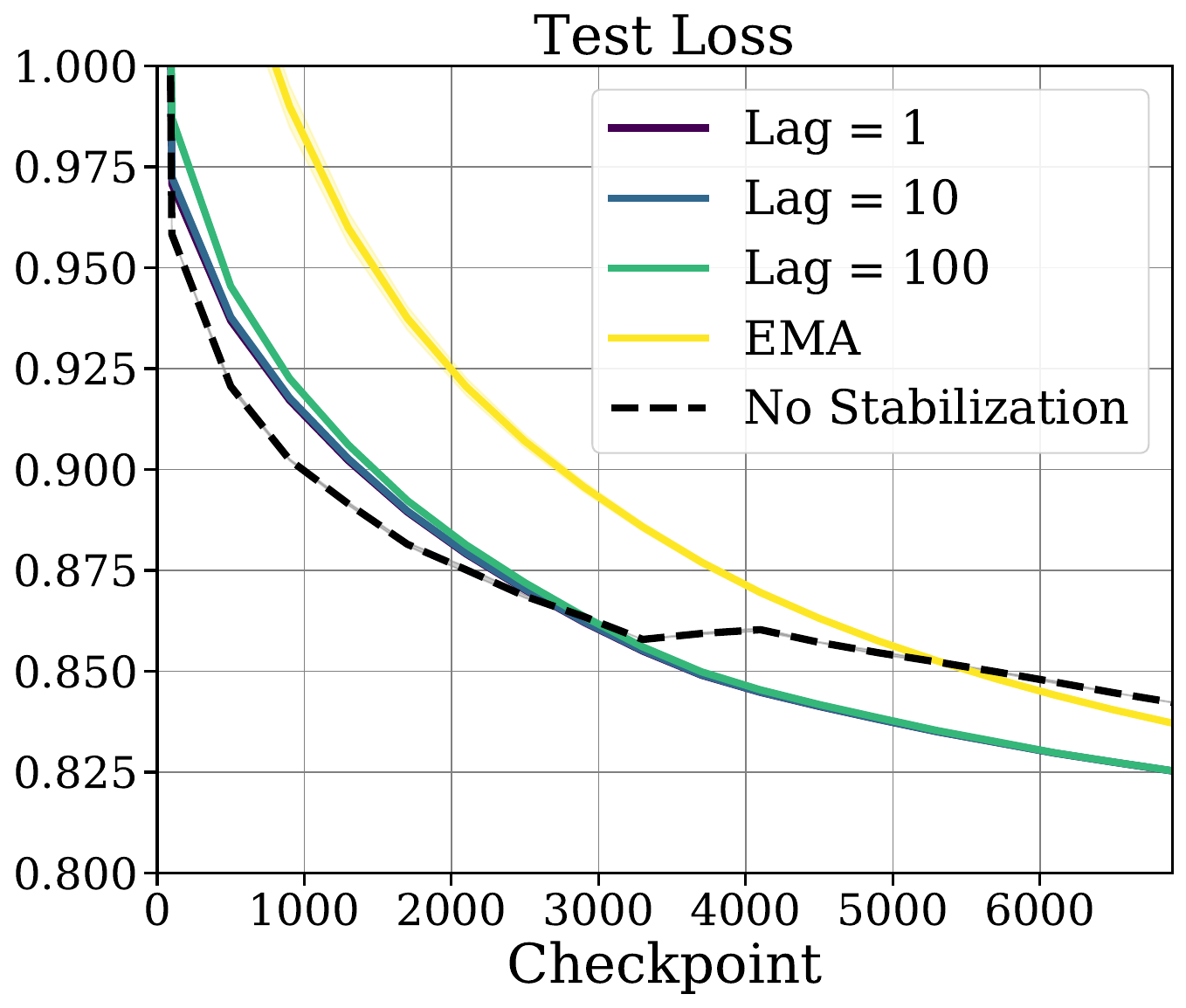}
        \label{sfig:scalinglag_test_loss}
    } \hfill
    \subfigure[]{
        \includegraphics[width=0.29\textwidth]{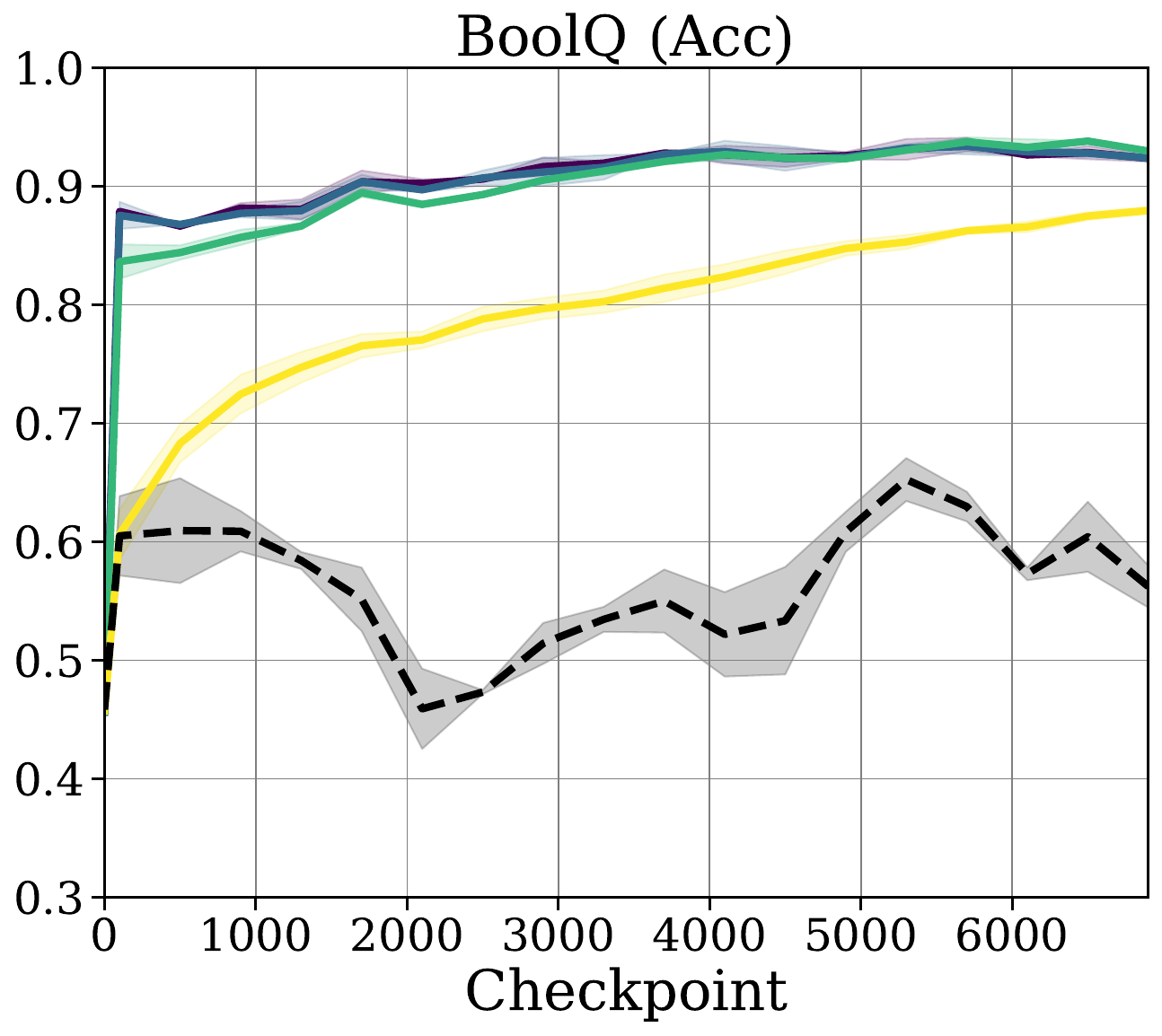}
        \label{sfig:scalinglag_boolq}
    } \\
    \subfigure[]{
        \includegraphics[width=\textwidth]{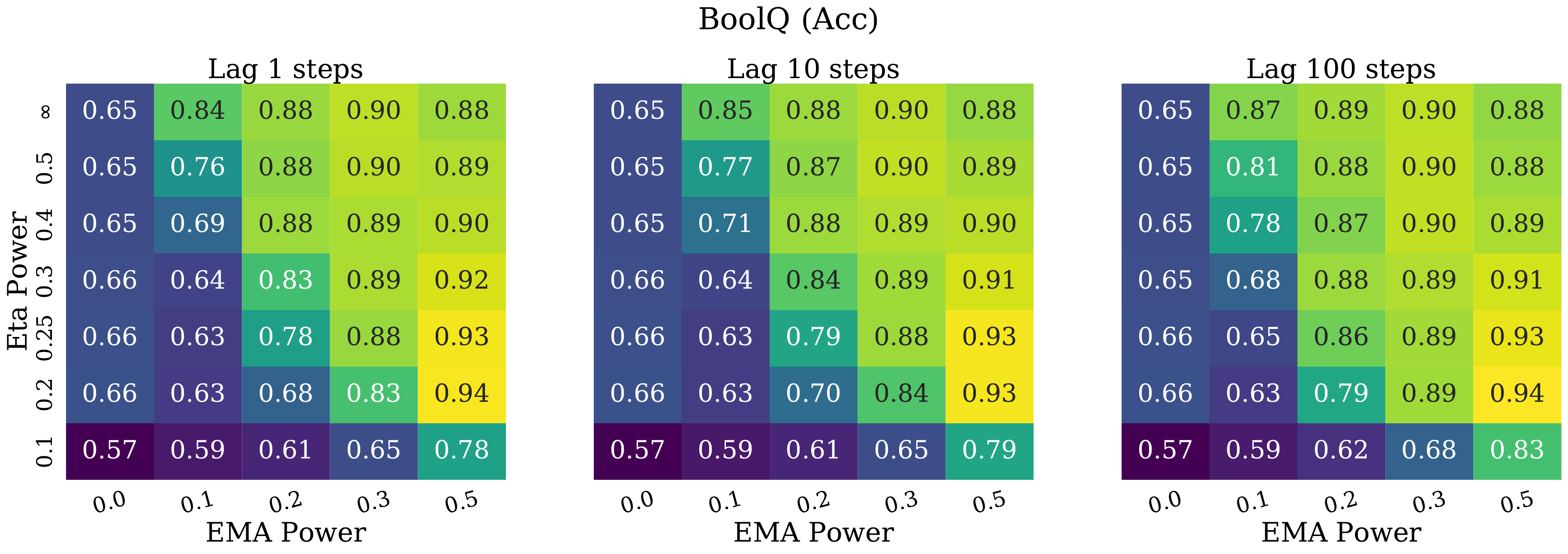}
        \label{sfig:scalinglag_boolq_heatmap}
    }
    \caption{Effect of the choice of lag $\rho$ on training for train loss \textbf{(a)}, test loss \textbf{(b)}, and \boolq~ \textbf{(c)}.  We also compare the optimal performance throughout training for different values of $\rho$, $\eta$, and $\kappa$ in \textbf{(d)}.  In general, we see minimal effect of the choice of $\rho$ on performance for all values of $\eta$ and $\kappa$.
    }
    \label{fig:scalinglag}
\end{figure}

\begin{figure}[t]
    \centering
    \subfigure[]{
        \includegraphics[width=0.29\textwidth]{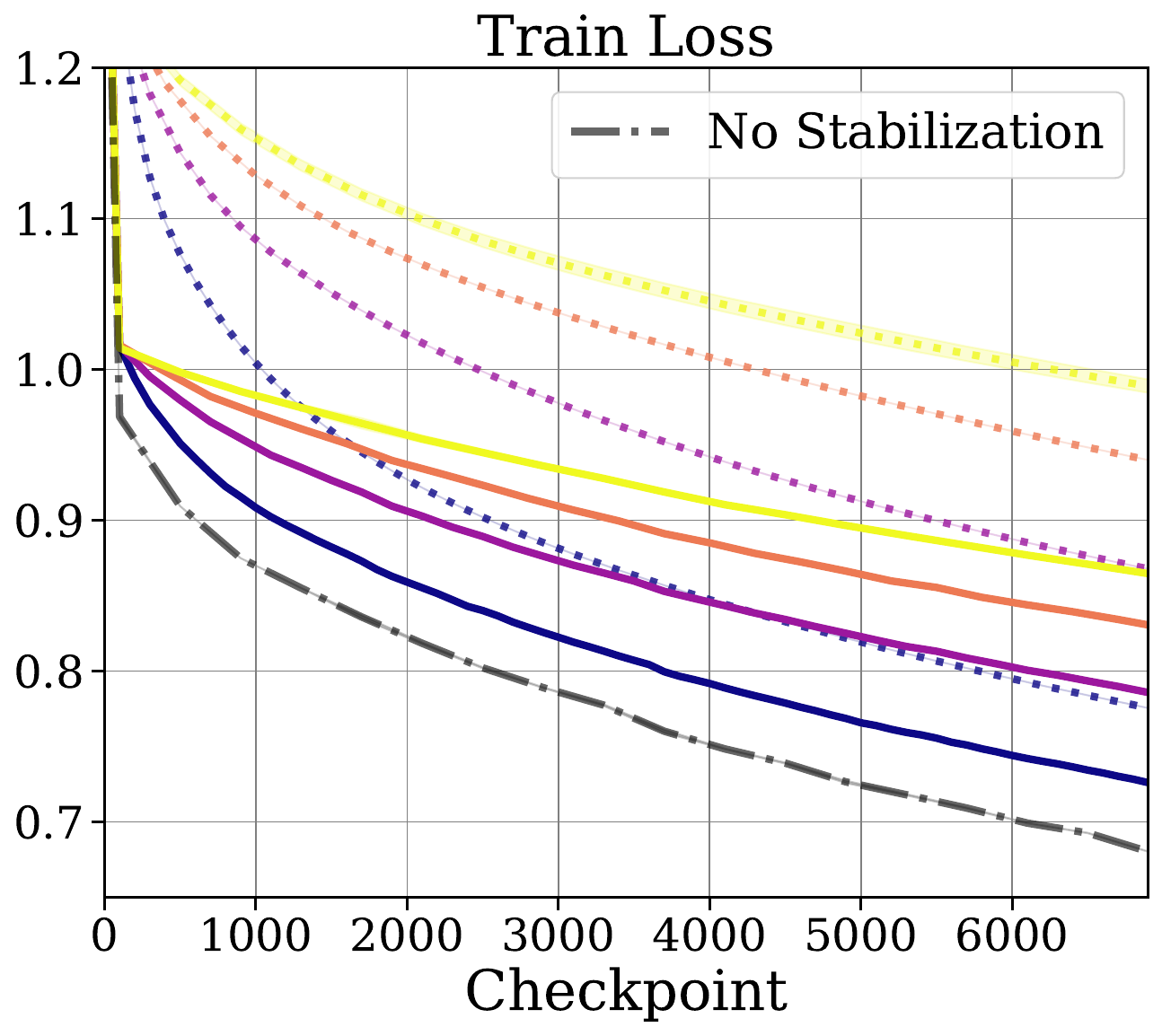}
        \label{sfig:updatefreqs_train_loss}
    } \hfill
    \subfigure[]{
        \includegraphics[width=0.29\textwidth]{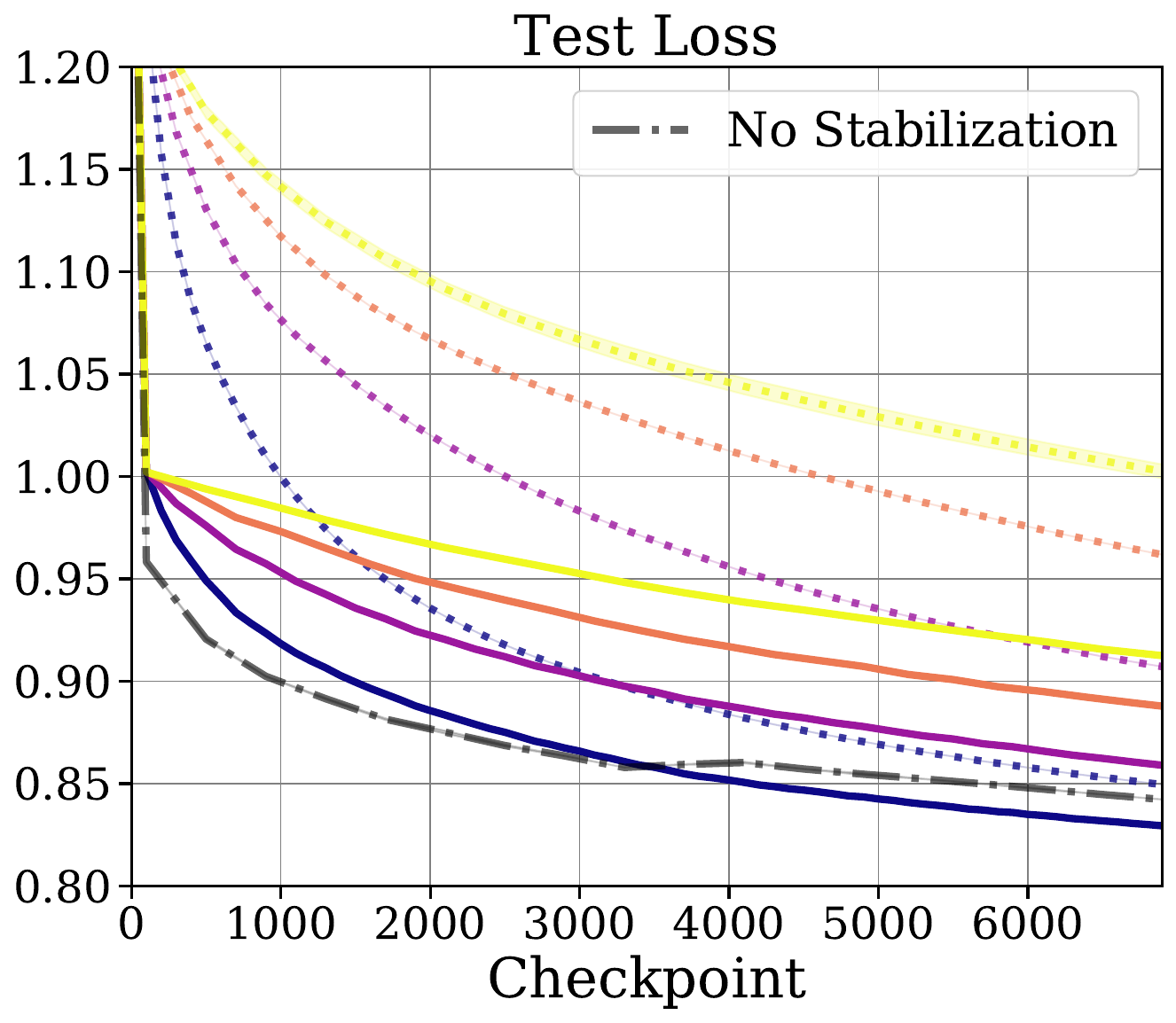}
        \label{sfig:updatefreqs_test_loss}
    } \hfill
    \subfigure[]{
        \includegraphics[width=0.29\textwidth]{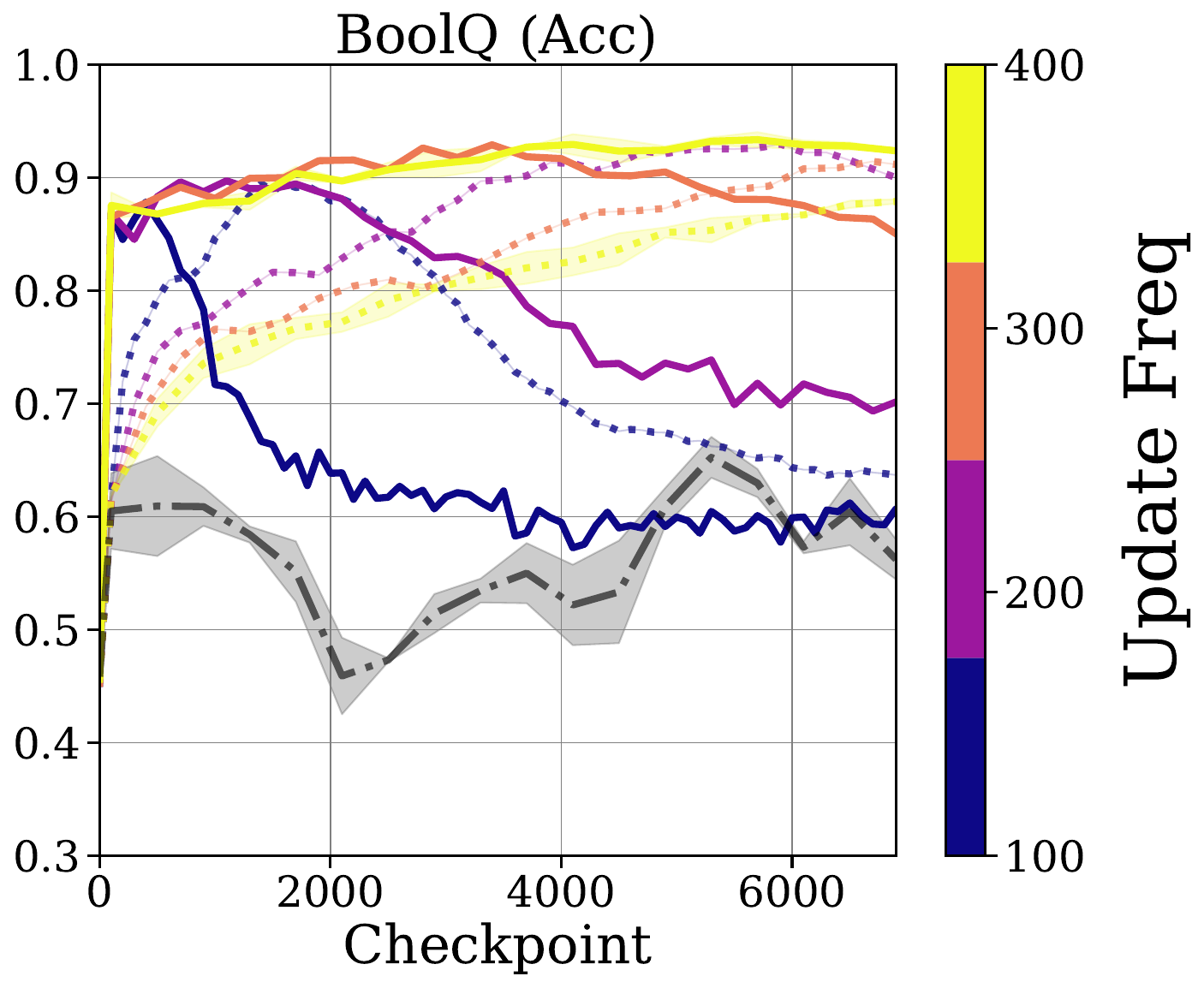}
        \label{sfig:updatefreqs_boolq}
    } \\
    \subfigure[]{
        \includegraphics[width=0.29\textwidth]{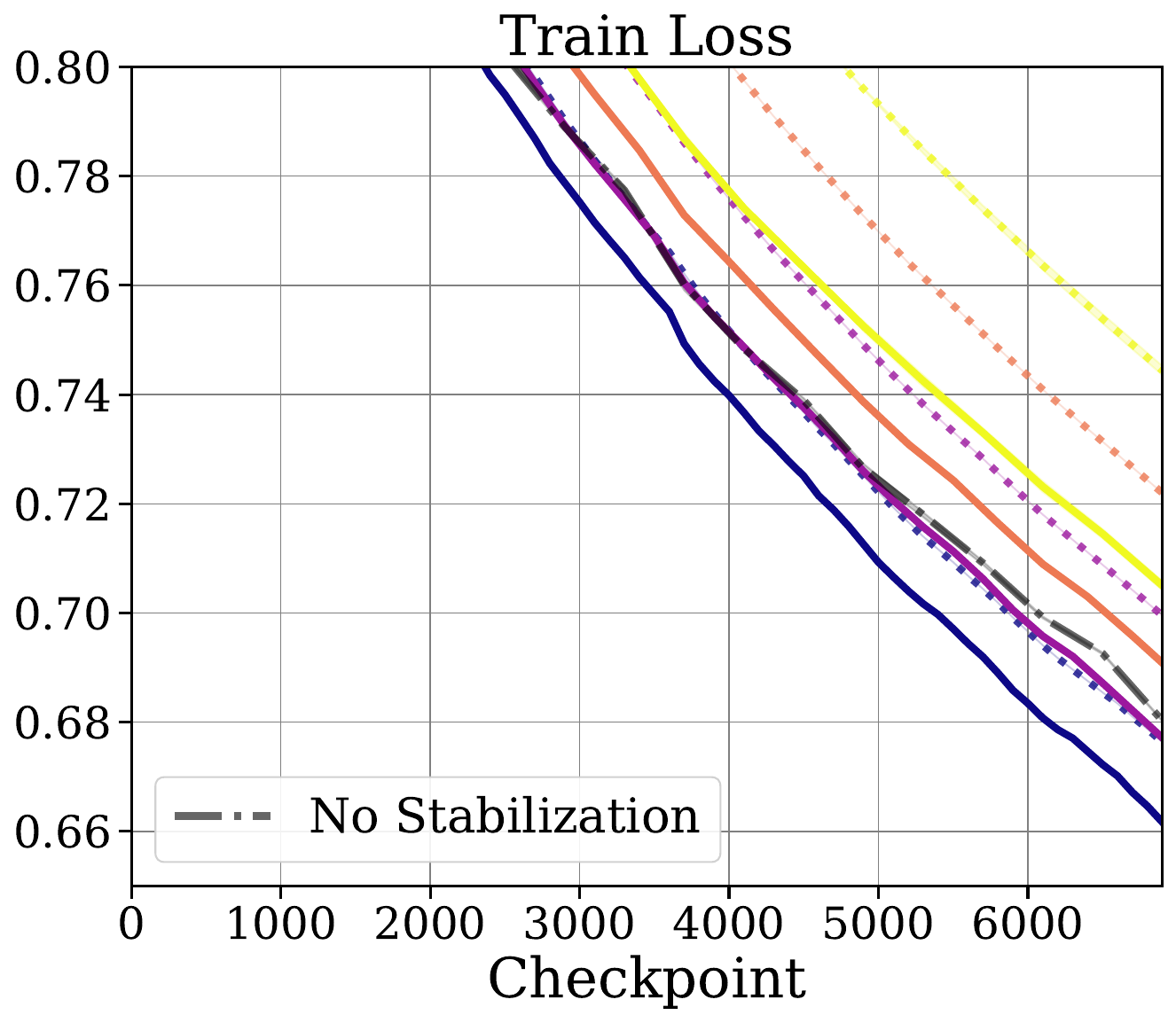}
        \label{sfig:updatefreqs_train_loss3}
    } \hfill
    \subfigure[]{
        \includegraphics[width=0.29\textwidth]{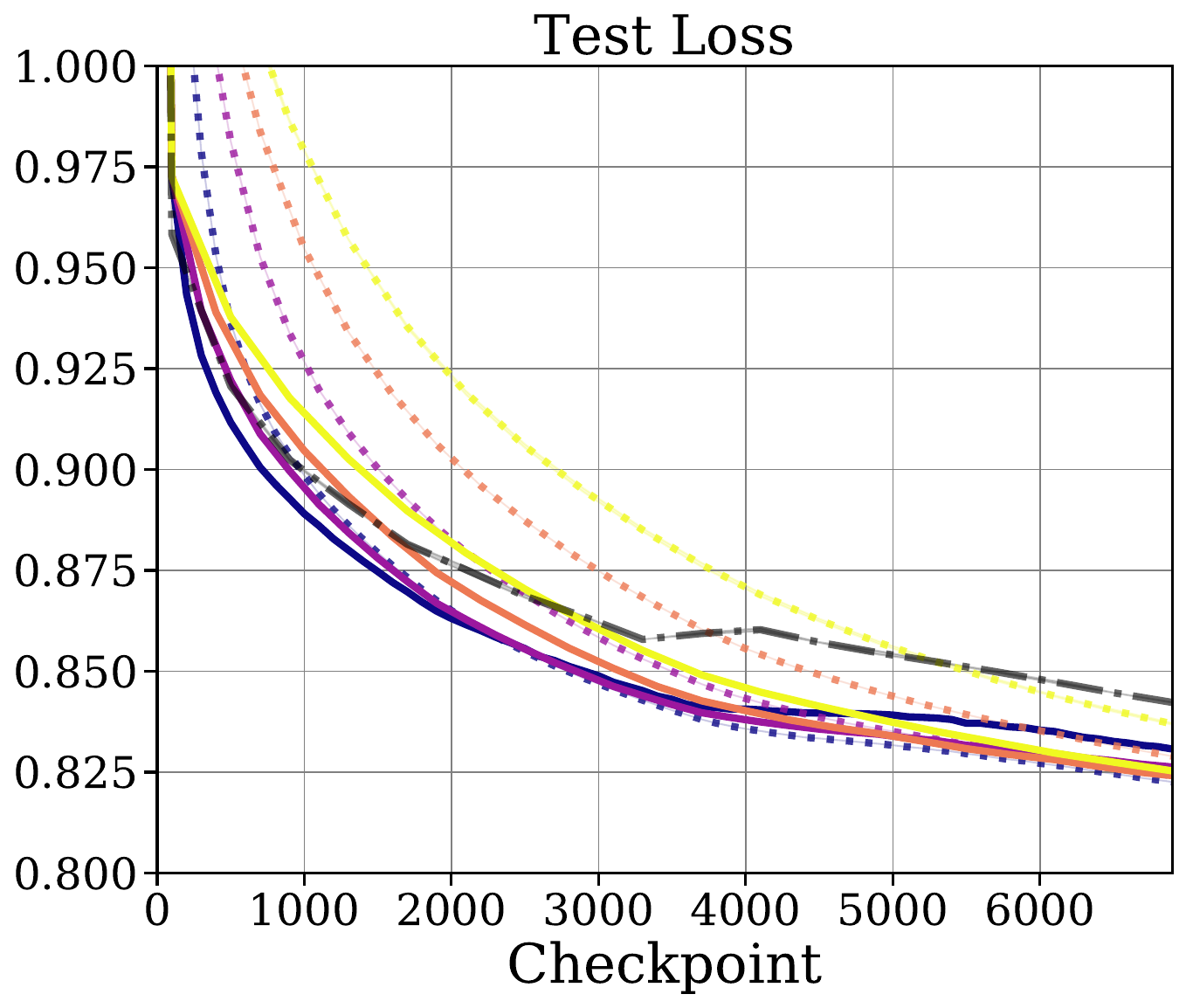}
        \label{sfig:updatefreqs_test_loss3}
    } \hfill
    \subfigure[]{
        \includegraphics[width=0.29\textwidth]{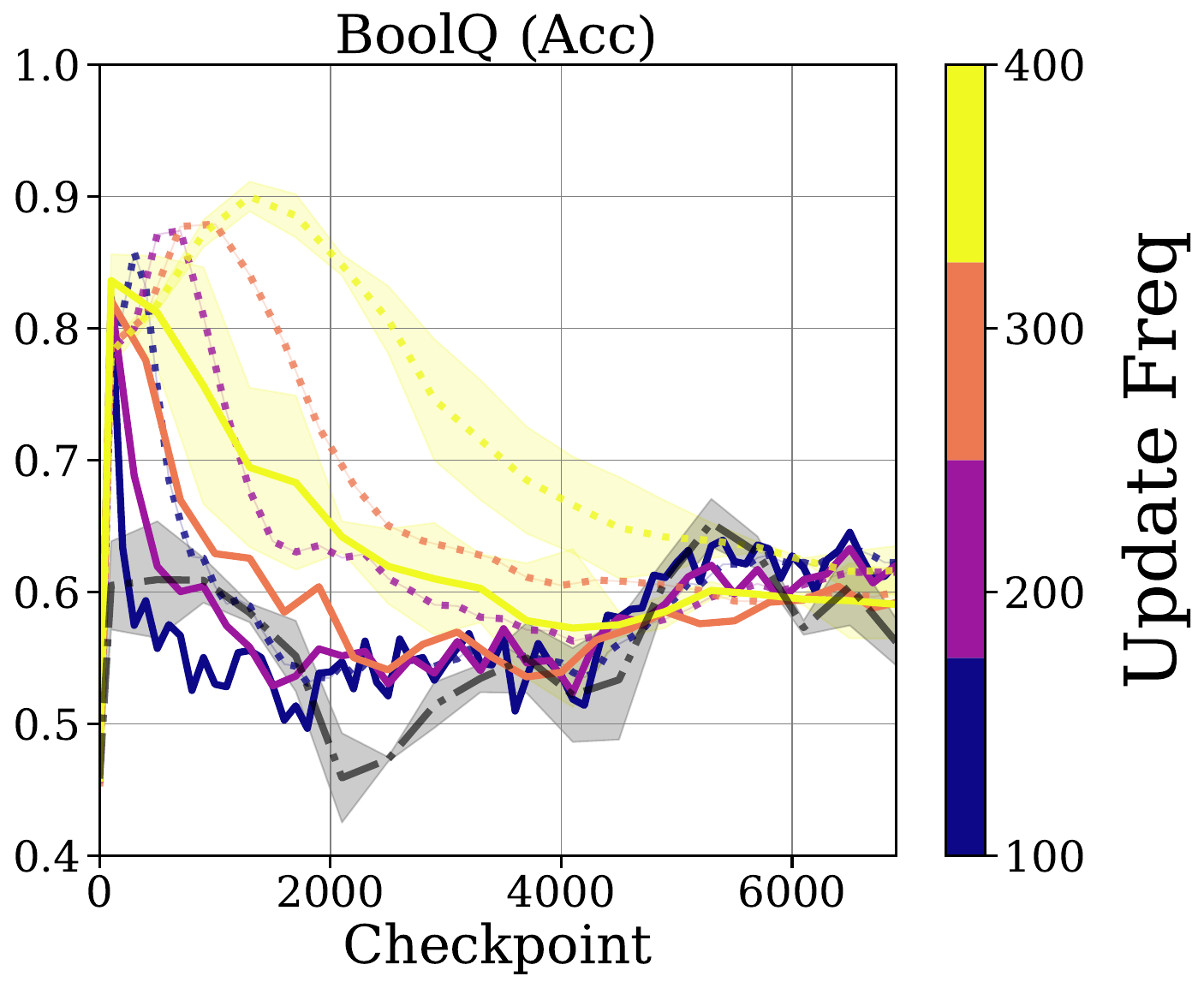}
        \label{sfig:updatefreqs_boolq3}
    }
    \caption{
        Effect of the chocie of update frequency $\phi$ on train loss \textbf{(a,d)}, test loss \textbf{(b,e)}, and \boolq~ performance \textbf{(c,f)} for $\kappa = 0.5$ \textbf{(top)} and $\kappa = 0.3$ \textbf{(bottom)}.  Here we plot both \bema\  (solid lines) and \ema\  (dashed lines) and the color corresponds to $\phi$.  Updating with increasing frequency tends to substantially increase convergence speed, leading to significant improvements in train and test losses.  This benefit trades off against (a) compute time in that more frequent updates slow the wall clock time of training and (b) potential overfitting, as can be observed in the \boolq~ performance for $\phi = 100$ and, to a lesser extent, $\phi = 200$.
    }
    \label{fig:updatefreqs}
\end{figure}

\begin{figure}[t]
    \centering
    \subfigure{
        \includegraphics[width=\textwidth]{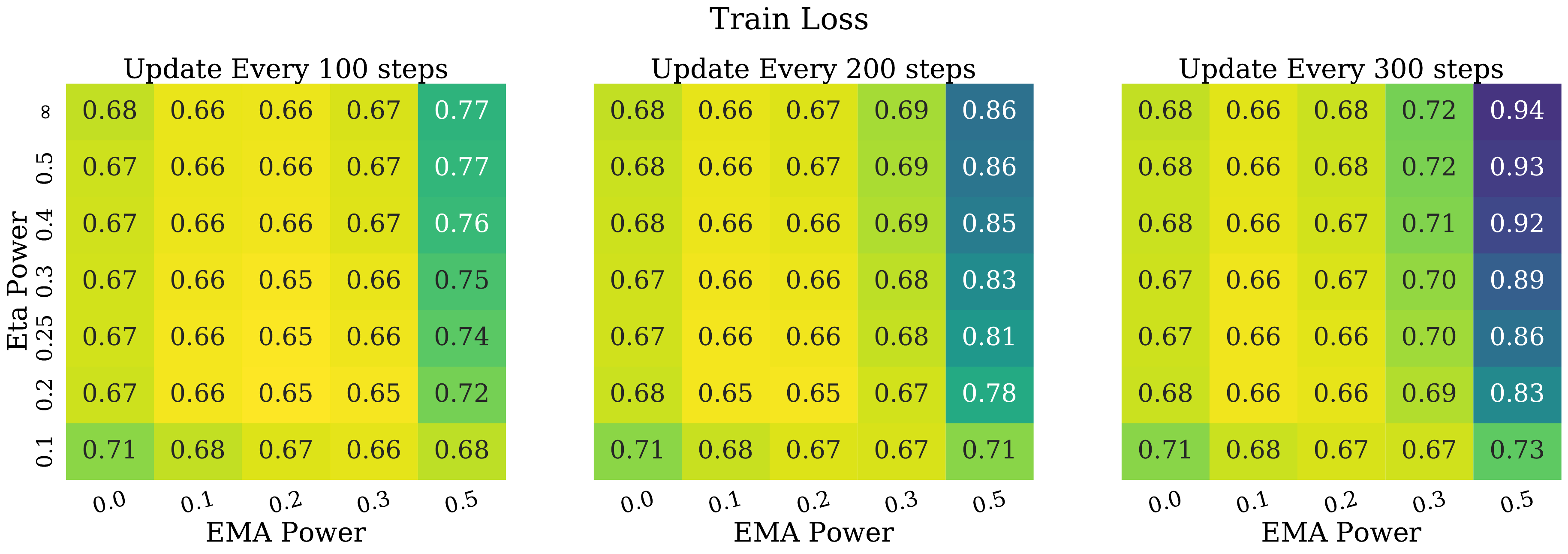}
        \label{sfig:updatefreqs_train_loss_heatmap}
    }  \\
    \subfigure{
        \includegraphics[width=\textwidth]{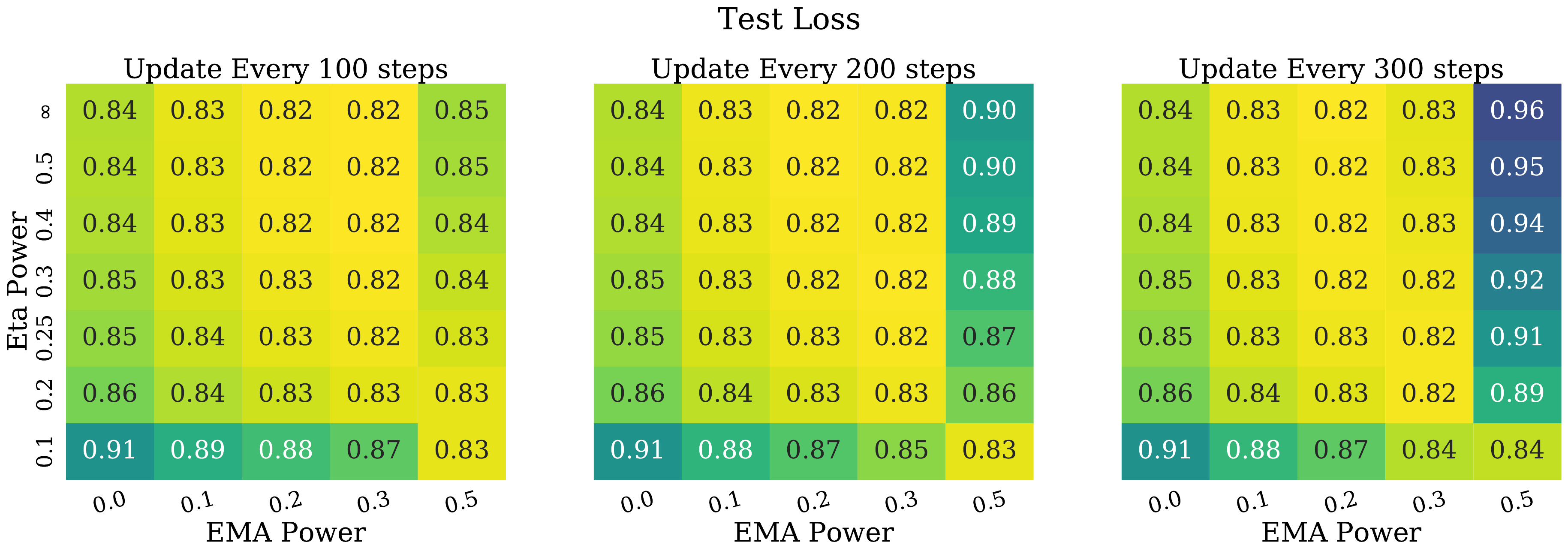}
        \label{sfig:updatefreqs_test_loss_heatmap}
    }  \\
    \subfigure{
        \includegraphics[width=\textwidth]{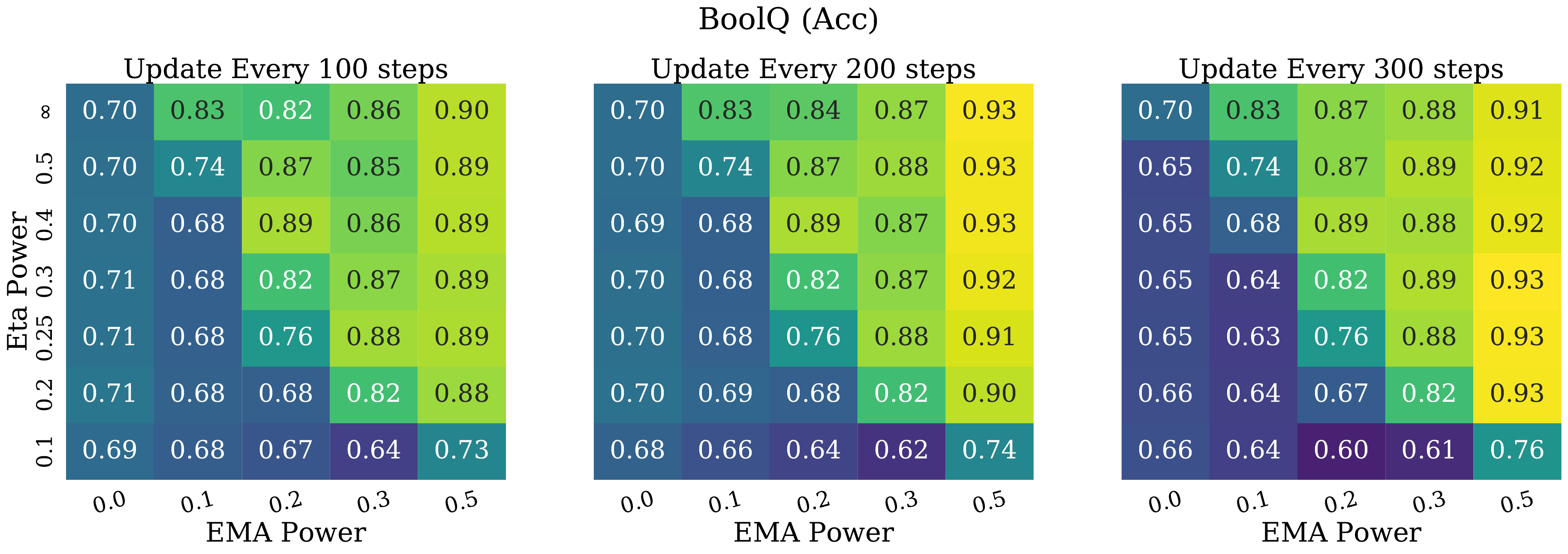}
        \label{sfig:updatefreqs_boolq_heatmap}
    }
    \caption{
        Effect of the choice of update frequency $\phi$ on optimal throughout training trajectory crossentropy loss on train \textbf{(top)} and test \textbf{(middle)} sets as well as performance on \boolq ~ \textbf{(bottom)} for a variety of choices of $\kappa$ and $\eta$.  Compare to \Cref{fig:minlr_multiples_losses,fig:minlr_multiples_benchmarks} for the default choice of $\phi = 400$.
    }
    \label{fig:updatefreqs_heatmaps}
\end{figure}

\begin{figure}[t]
    \centering
    \subfigure[]{
        \includegraphics[width=0.29\textwidth]{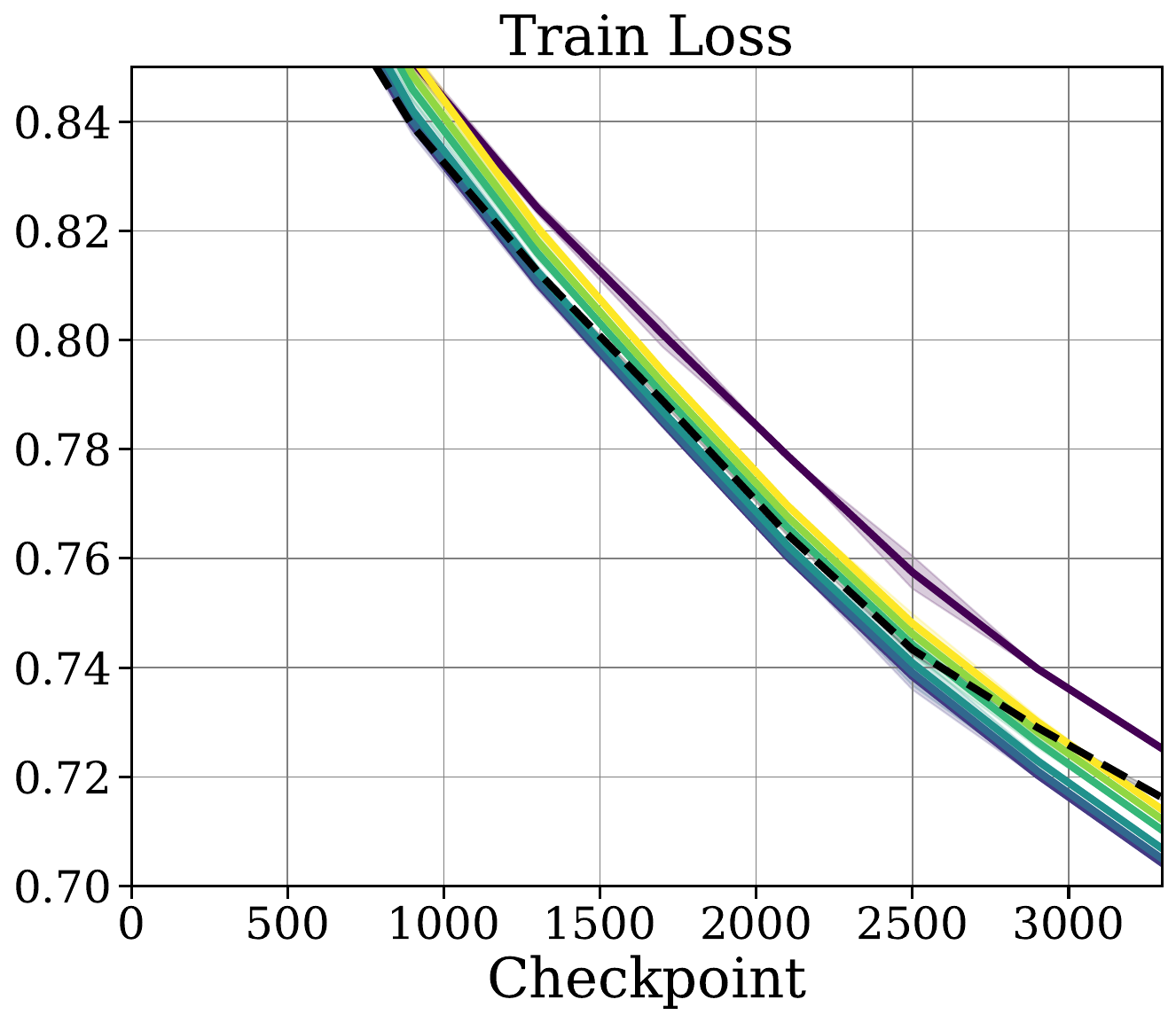}
        \label{sfig:batchsize_train_loss}
    } \hfill
    \subfigure[]{
        \includegraphics[width=0.29\textwidth]{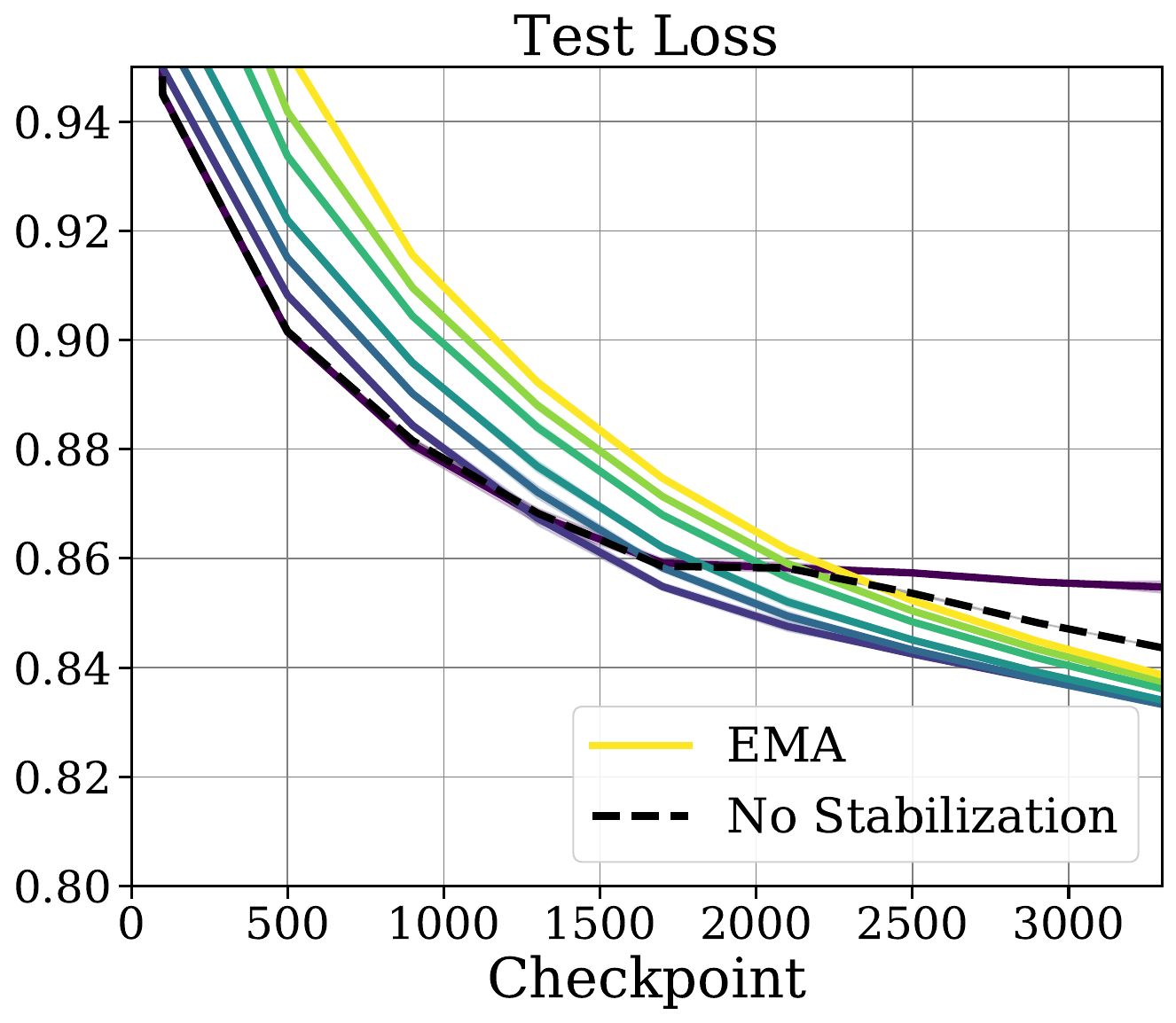}
        \label{sfig:batchsize_test_loss}
    } \hfill
    \subfigure[]{
        \includegraphics[width=0.29\textwidth]{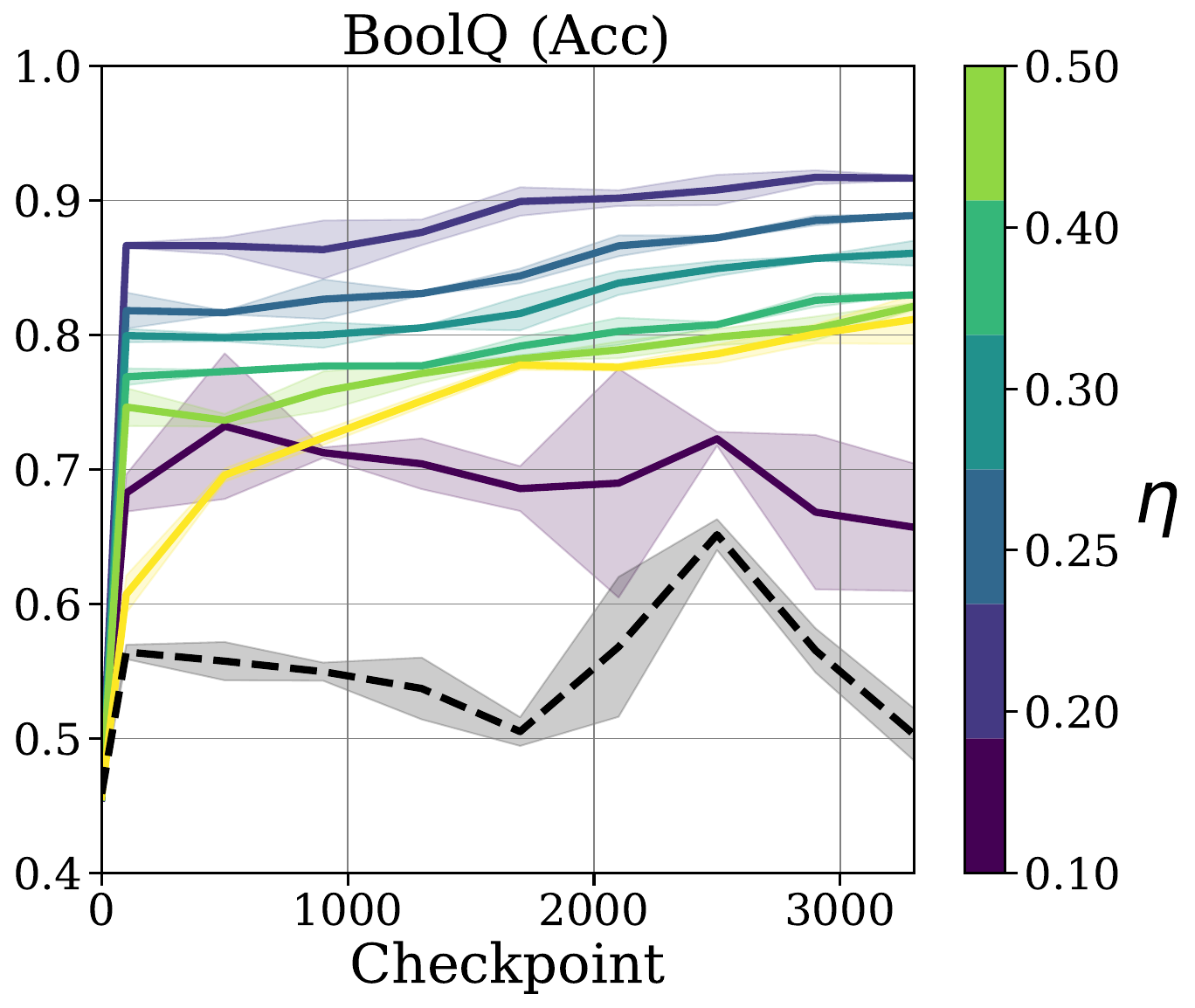}
        \label{sfig:batchsize_boolq}
    } 
    \caption{Demonstration of the robustness of \bema\  performance improvements to the choice of batch size; here training is conducted with an effective batch size of 512 and train loss \textbf{(a)}, test losses \textbf{(b)}, and \boolq~ performance \textbf{(c)} are shown.  We continue to see considerable performance improvements over \ema.}
    \label{fig:batchsize}
\end{figure}

\begin{figure}[t]
    \centering
    \subfigure[]{
        \includegraphics[width=0.29\textwidth]{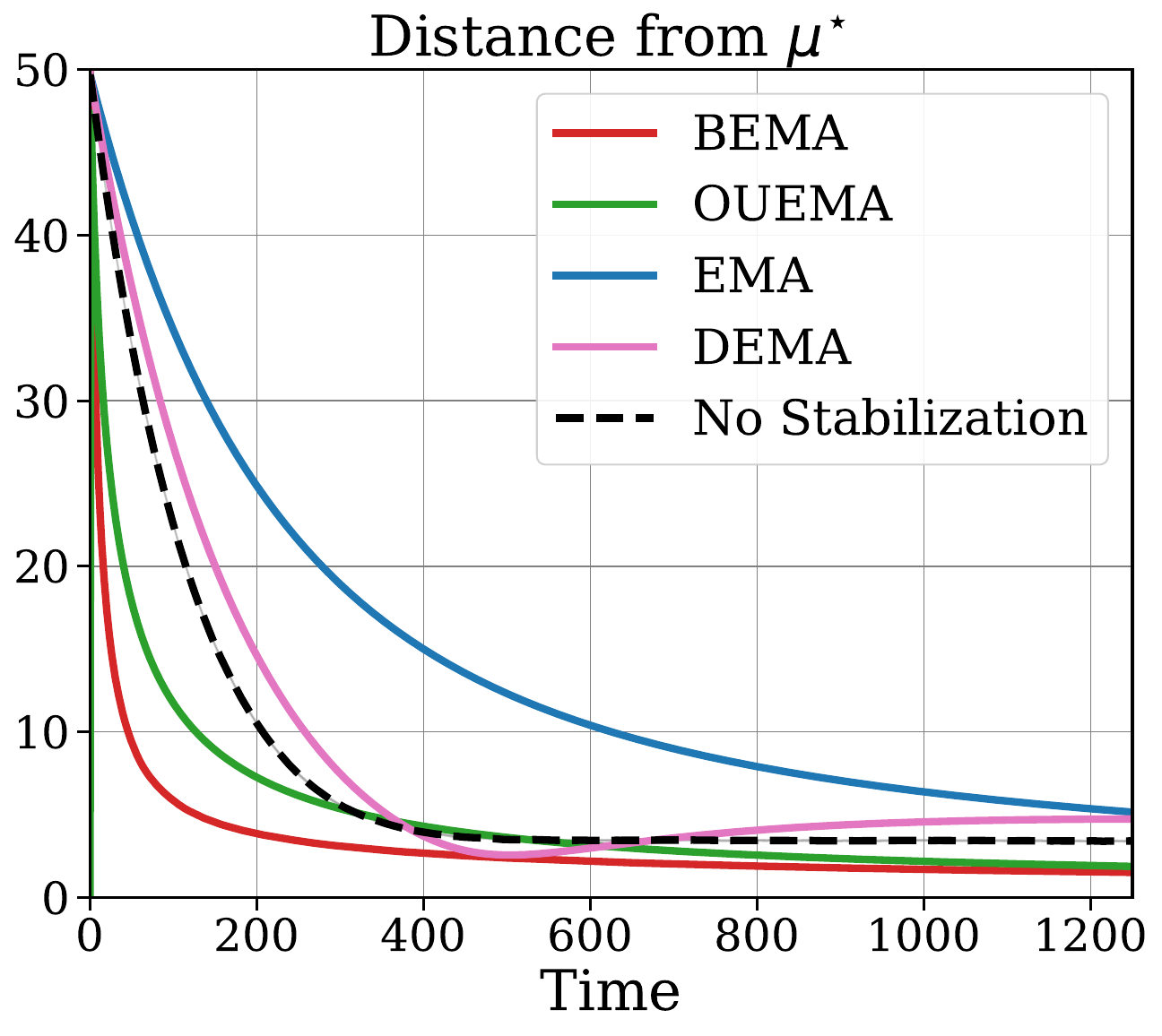}
        \label{sfig:quadratic_dema}
    } \hfill
    \subfigure[]{
        \includegraphics[width=0.29\textwidth]{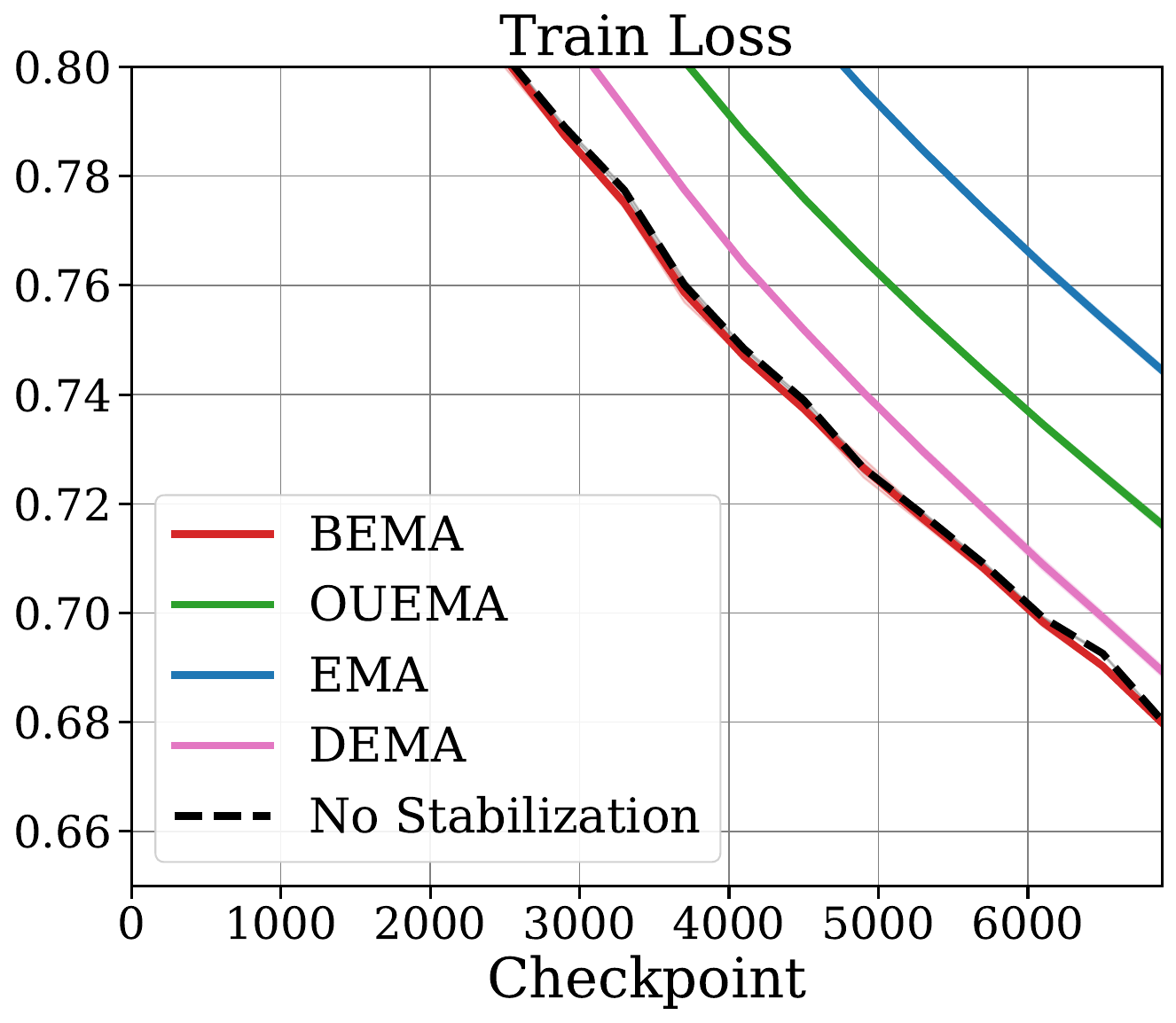}
        \label{sfig:dema_train_loss}
    } \hfill
    \subfigure[]{
        \includegraphics[width=0.29\textwidth]{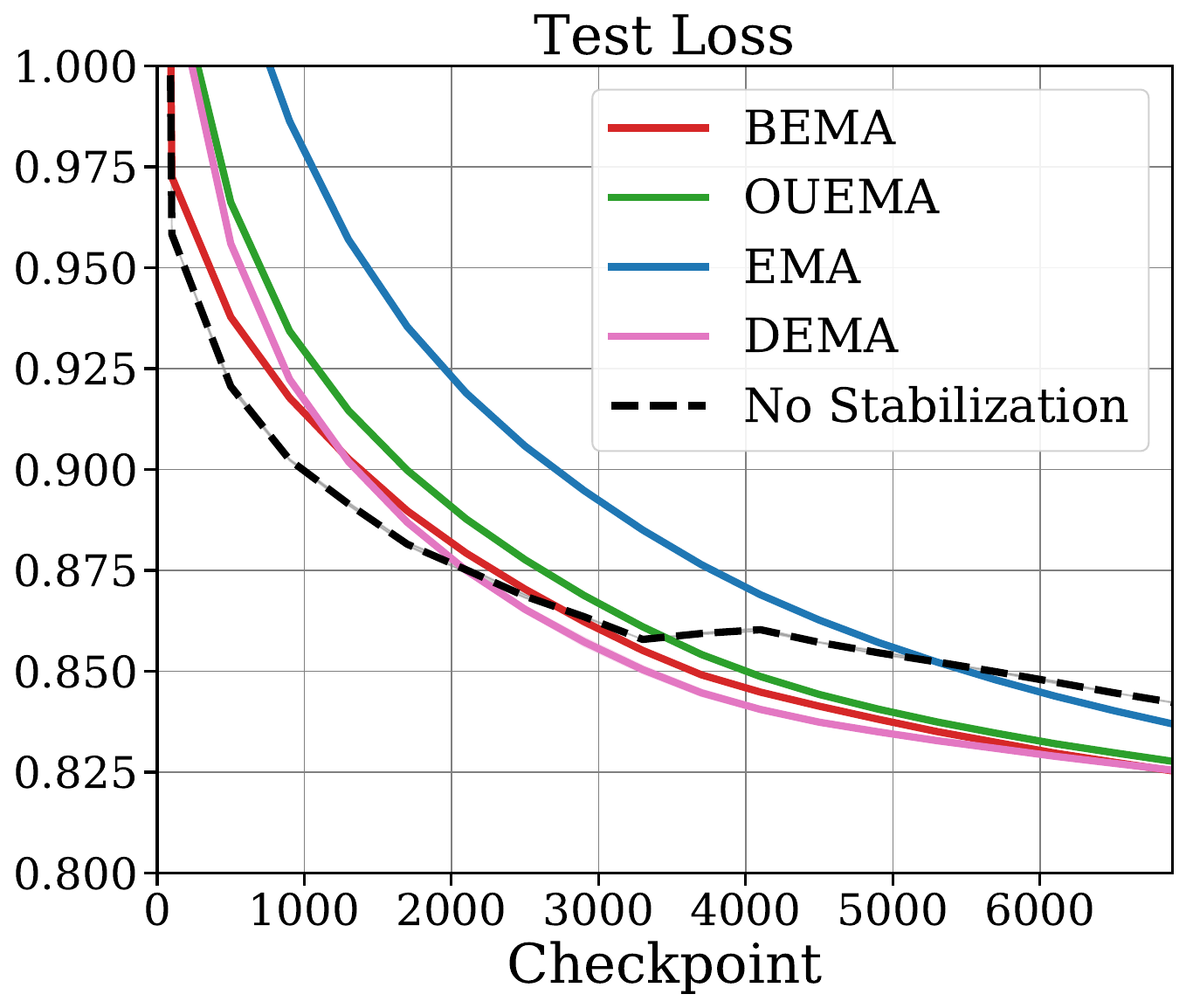}
        \label{sfig:dema_test_loss}
    } \\
    \subfigure[]{
        \includegraphics[width=0.29\textwidth]{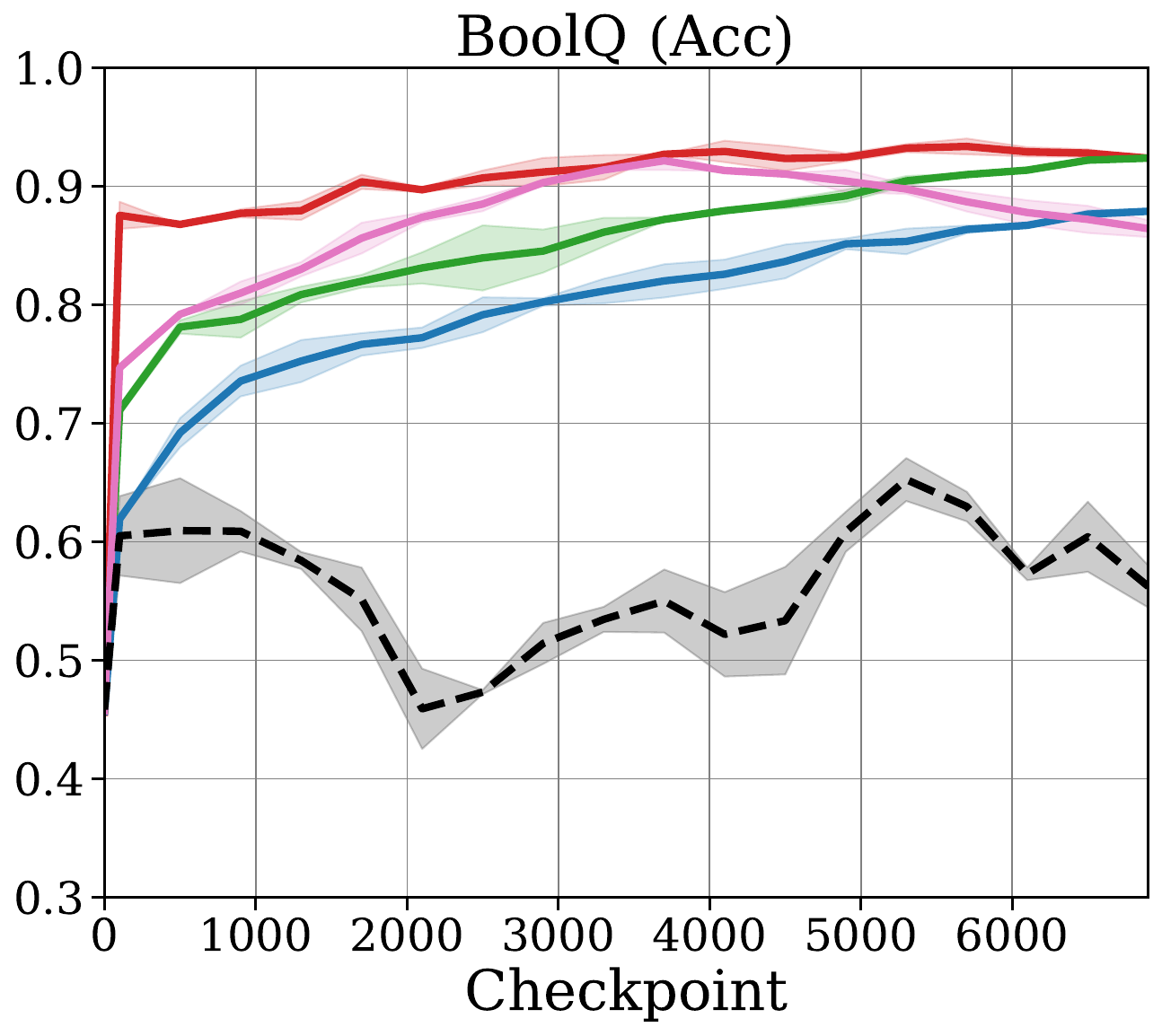}
        \label{sfig:dema_boolq}
    } \hfill
    \subfigure[]{
        \includegraphics[width=0.29\textwidth]{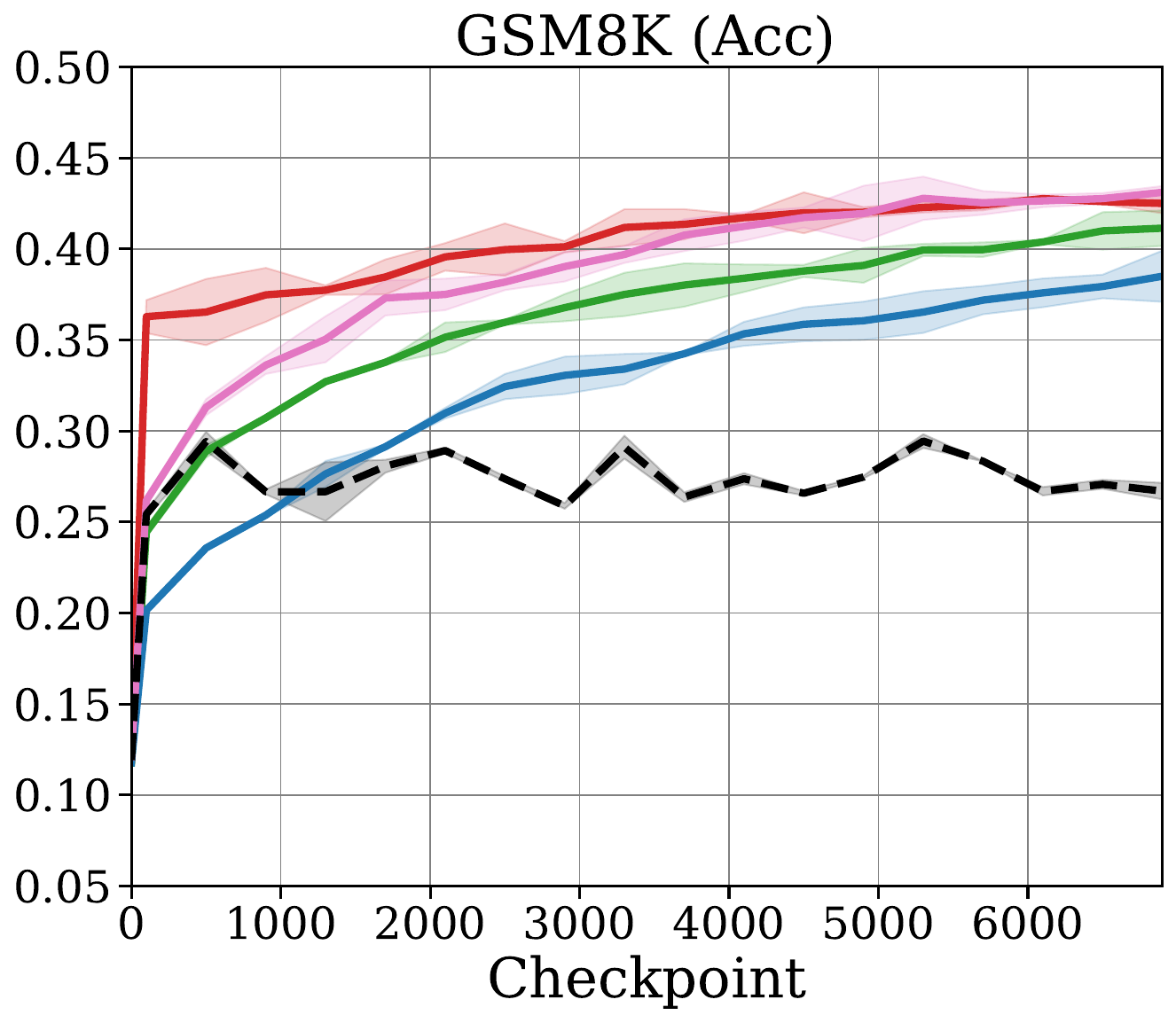}
        \label{sfig:dema_gsm8k}
    } \hfill
    \subfigure[]{
        \includegraphics[width=0.29\textwidth]{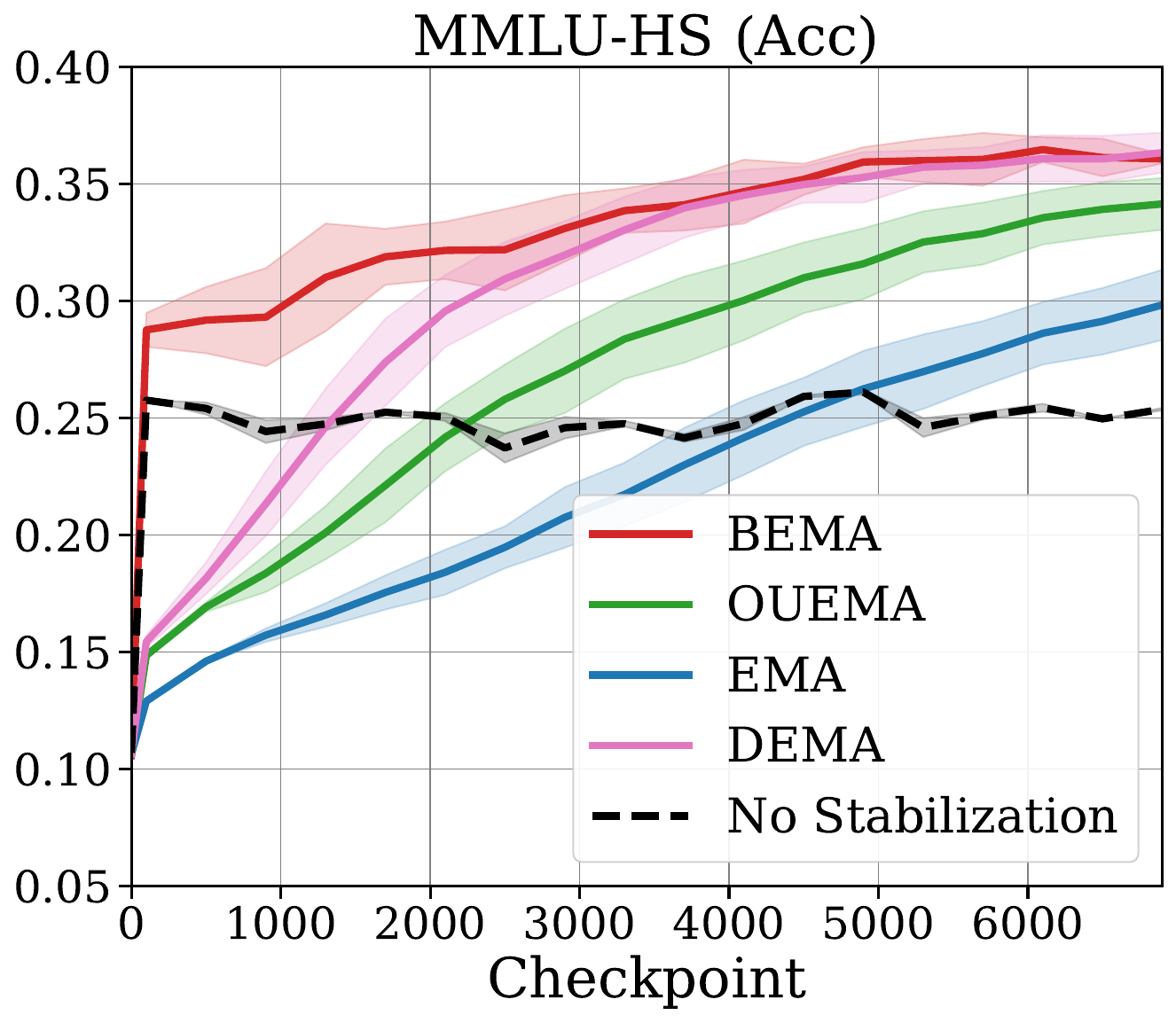}
        \label{sfig:dema_mmluhs}
    }
    \caption{Comparison of \bema\  to alternative stabilizer \dema.  \textbf{(a)} Demonstration of the effect of \dema~ on a quadratic loss landscape.  \textbf{(b)} Train loss, \textbf{(c)} test loss, \textbf{(d)} \boolq~ performance, \textbf{(e)} \gsmk~ performance, and \textbf{(f)} \mmluhs~ performance for \dema~ compared to \oumle, \ouema, and \ema.  In general, \dema~ improves on \ouema, which improves on \ema, but neither matches the acceleration and performance of \bema\  on the generation benchmarks.   
    }
    \label{fig:dema}
\end{figure}

\begin{figure}[t]   
	\centering
    \subfigure[]{  
		\includegraphics[width=0.29\textwidth]{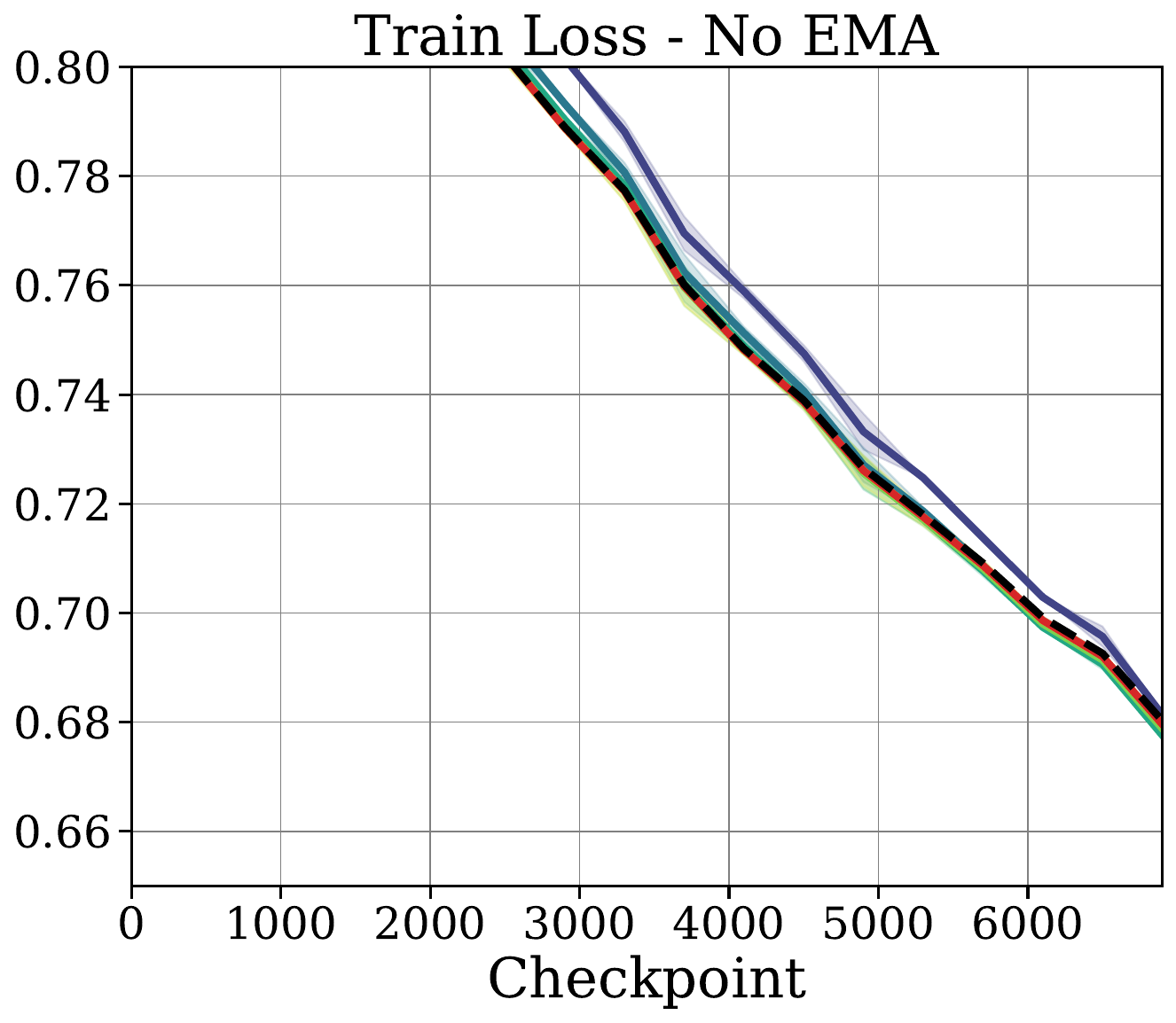}
		\label{sfig:ouema_noema_trainloss} 
	} \hfill \subfigure[]{
        \includegraphics[width=0.29\textwidth]{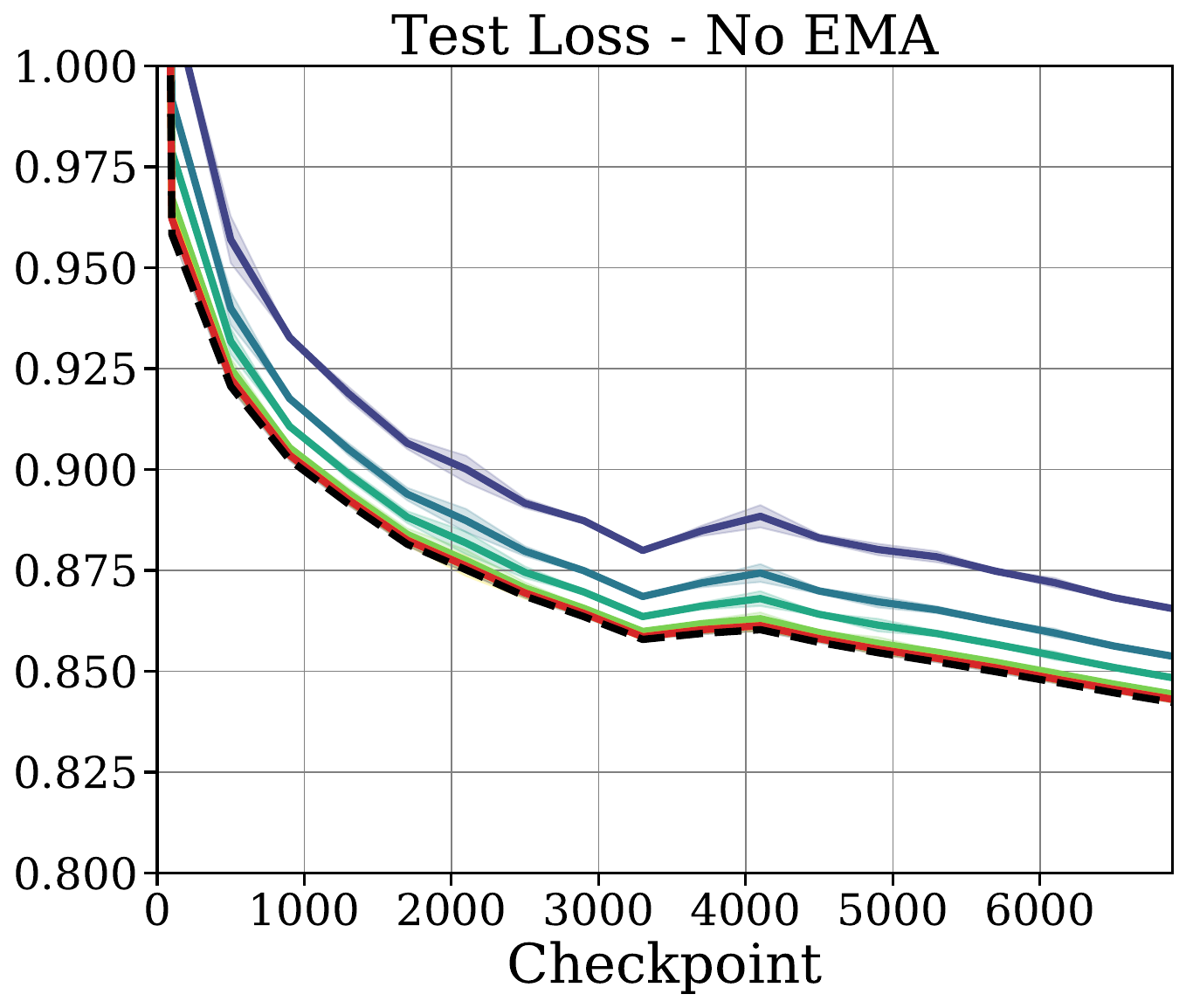}
        \label{sfig:ouema_noema_testloss}
    } \hfill \subfigure[]{ 
        \includegraphics[width=0.29\textwidth]{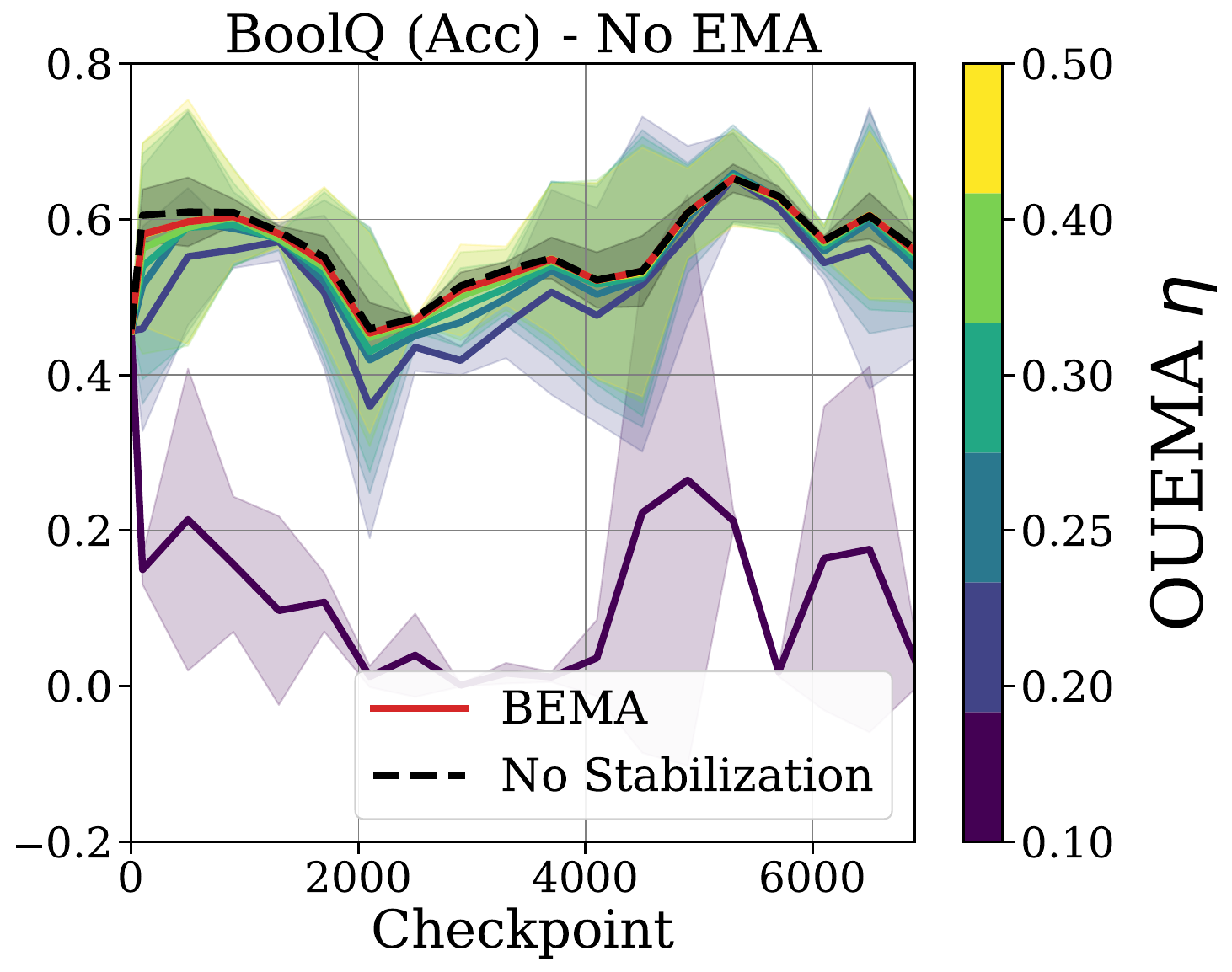} 
        \label{sfig:noema_boolq_dema}
    } \\
   \subfigure[]{  
		\includegraphics[width=0.29\textwidth]{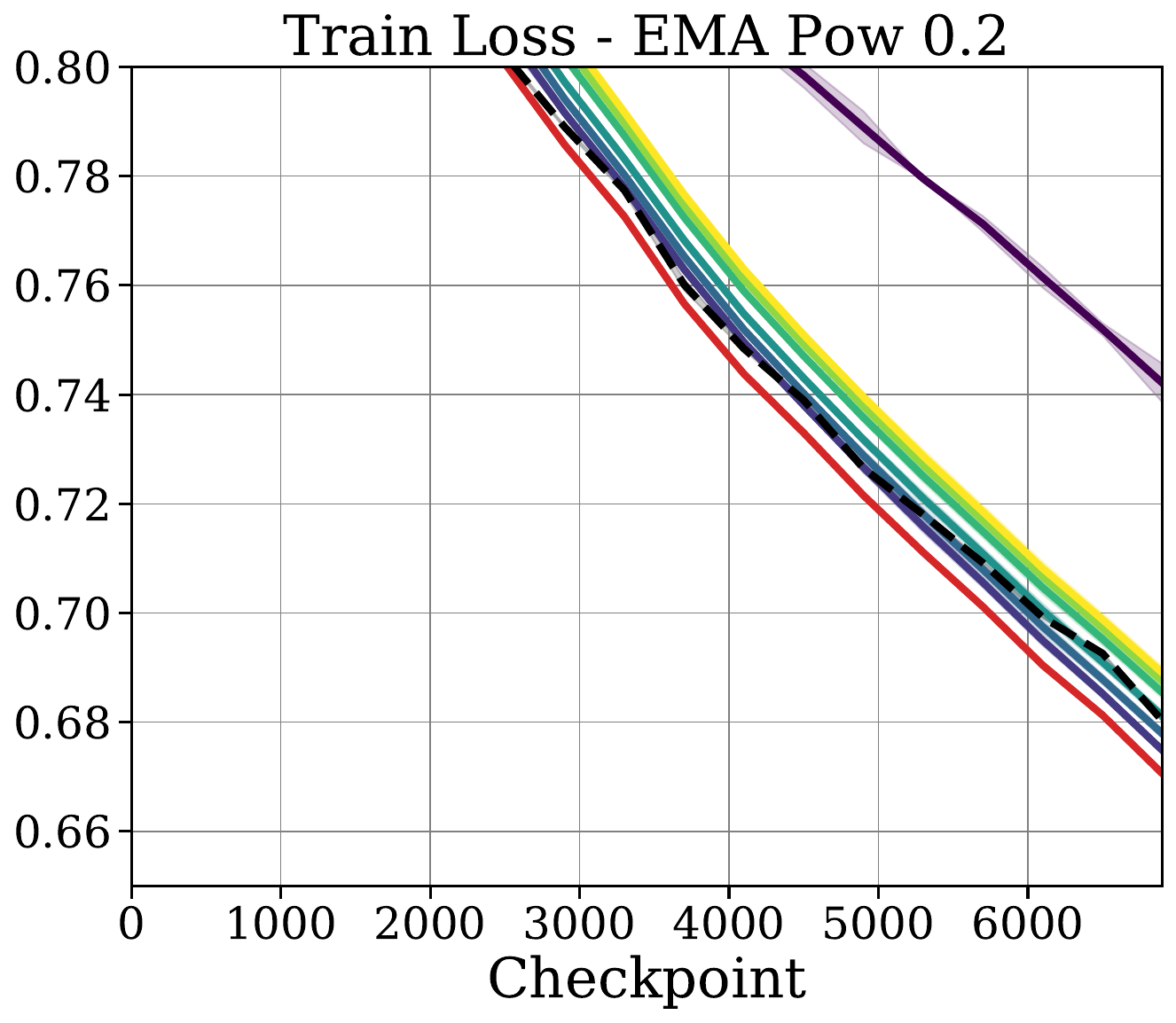}
		\label{sfig:ouema2_trainloss} 
	} \hfill \subfigure[]{
        \includegraphics[width=0.29\textwidth]{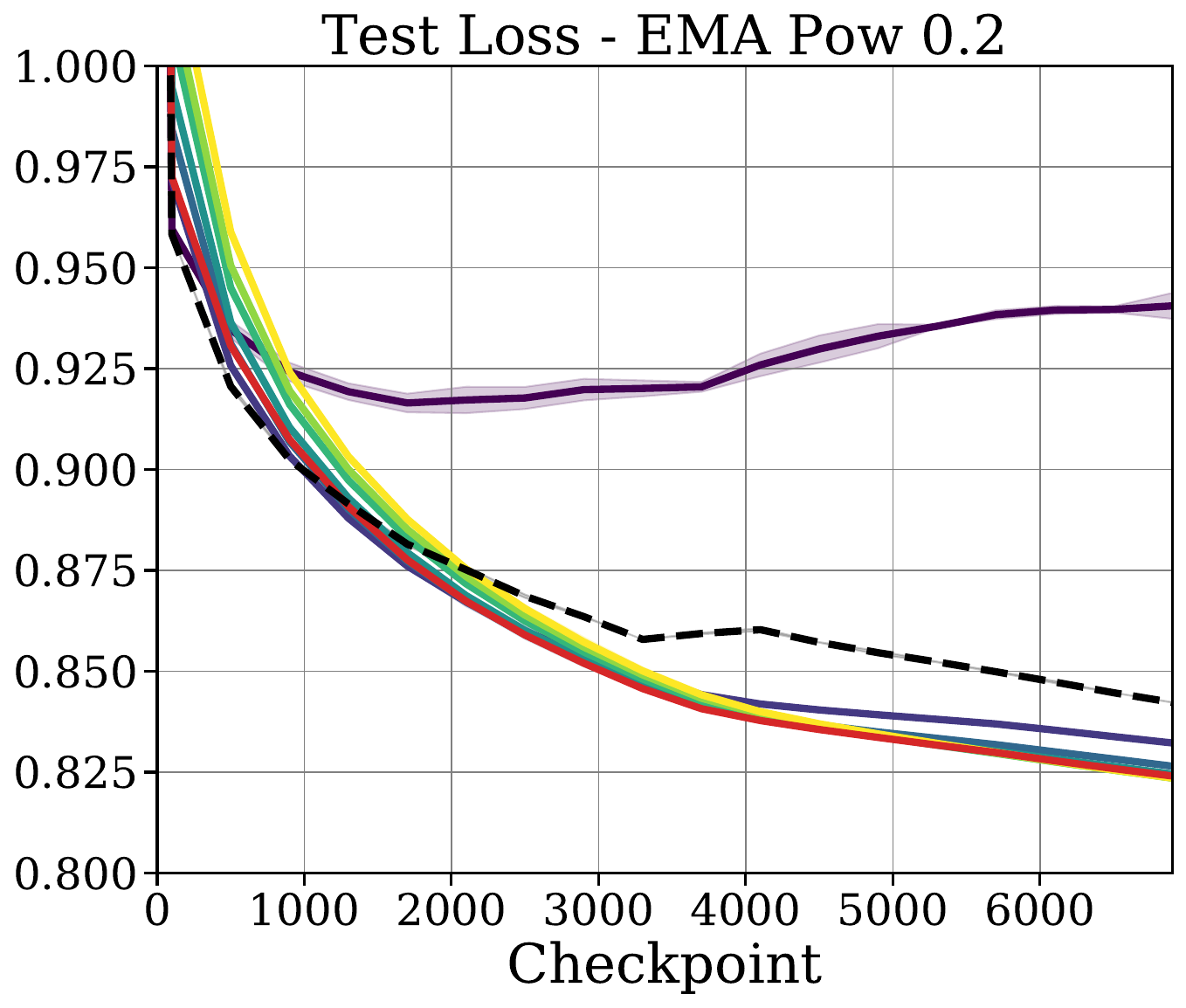} 
        \label{sfig:ouema2_testloss}
    } \hfill \subfigure[]{ 
        \includegraphics[width=0.29\textwidth]{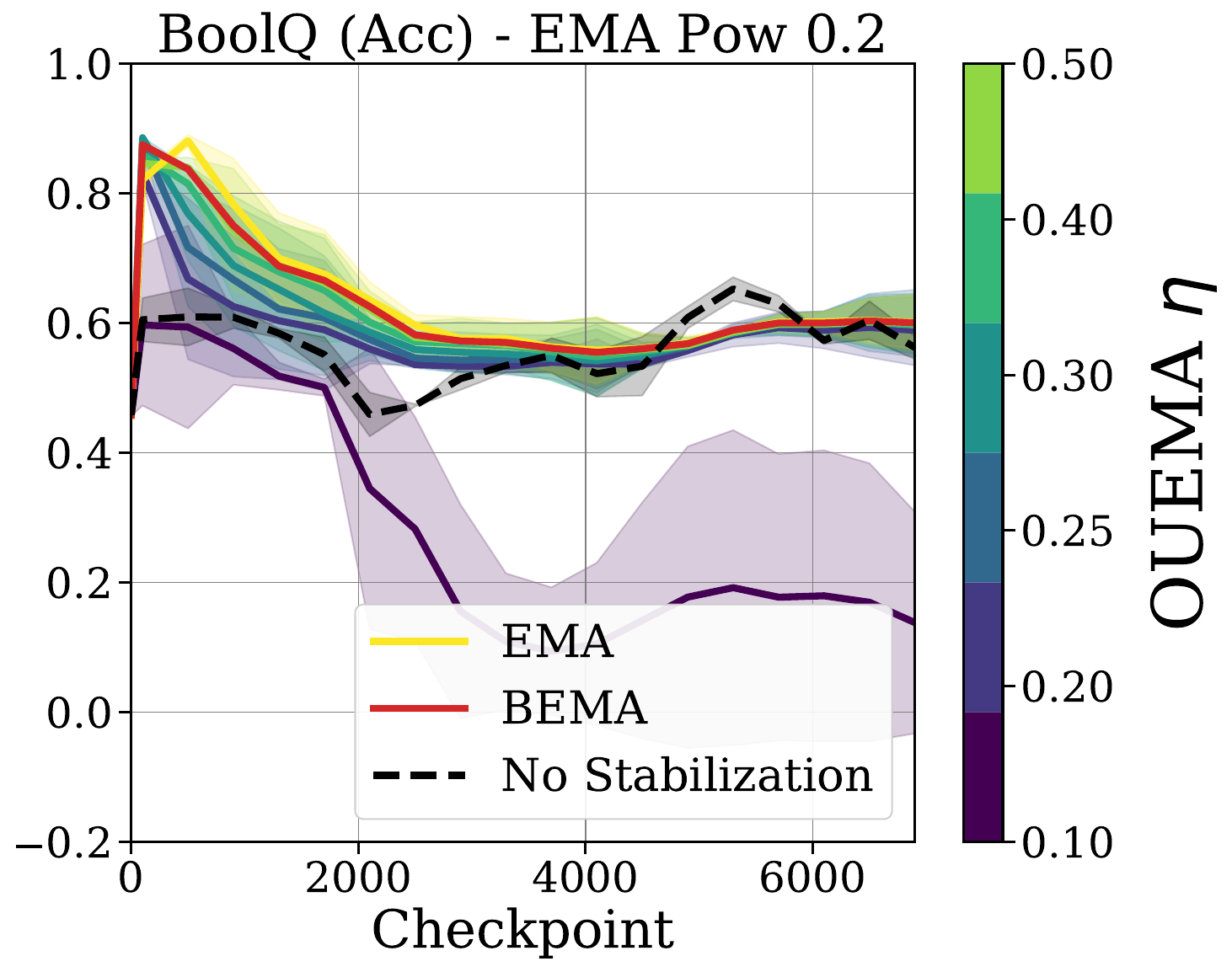} 
        \label{sfig:ouema2_boolq}
    } \\
    \subfigure[]{  
		\includegraphics[width=0.29\textwidth]{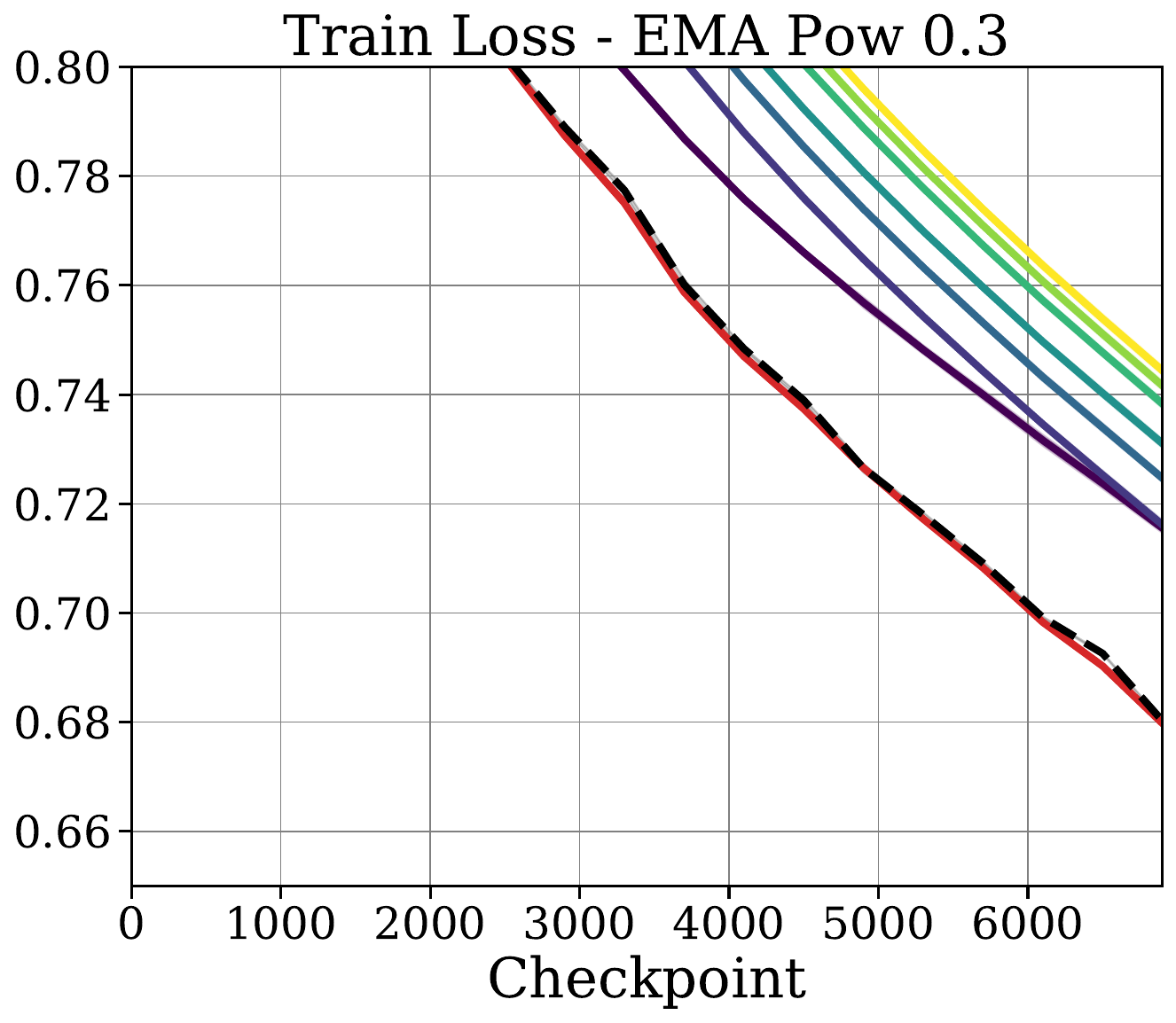}
		\label{sfig:ouema3_trainloss} 
	} \hfill \subfigure[]{
        \includegraphics[width=0.29\textwidth]{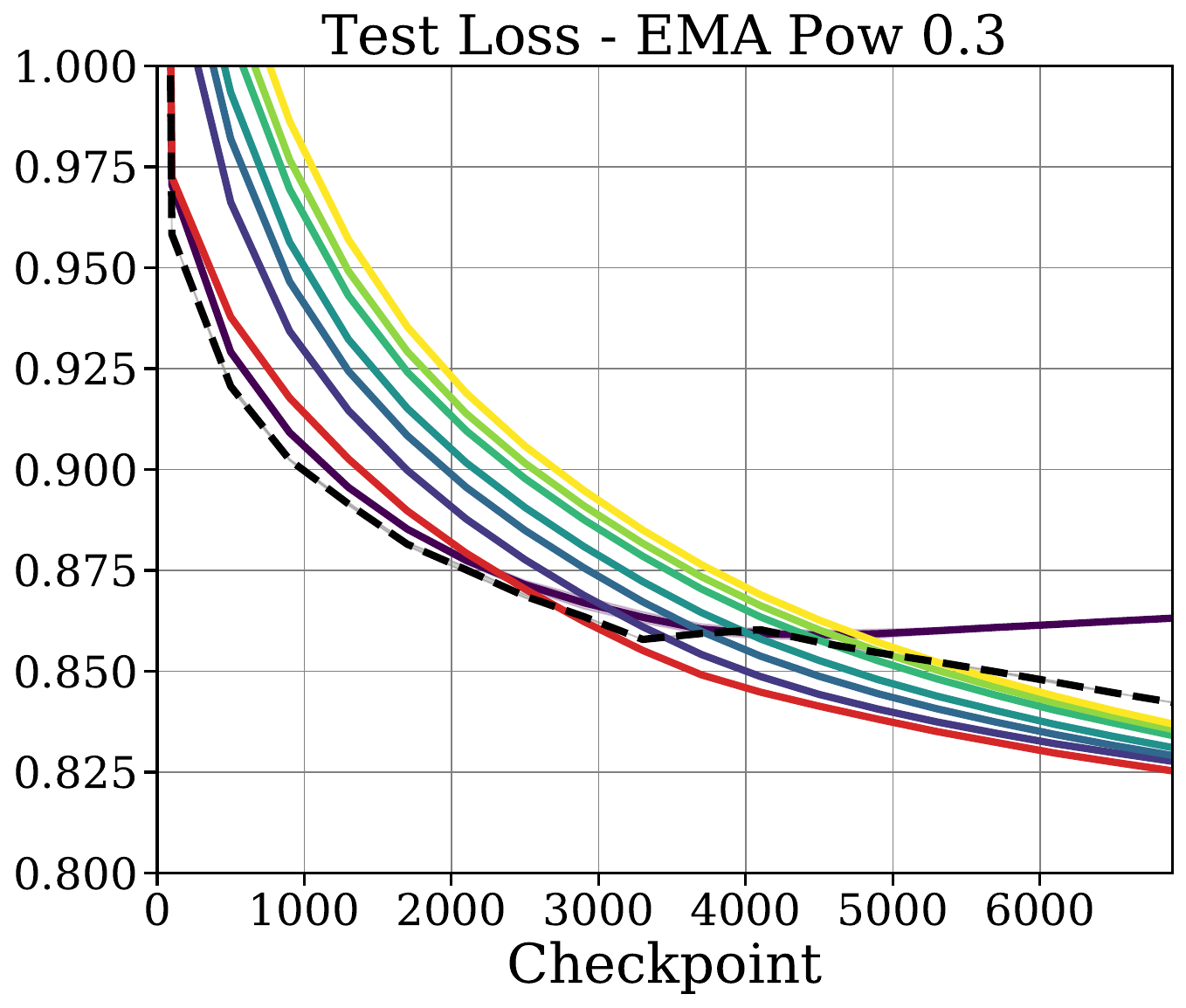} 
        \label{sfig:ouema3_testloss}
    } \hfill \subfigure[]{ 
        \includegraphics[width=0.29\textwidth]{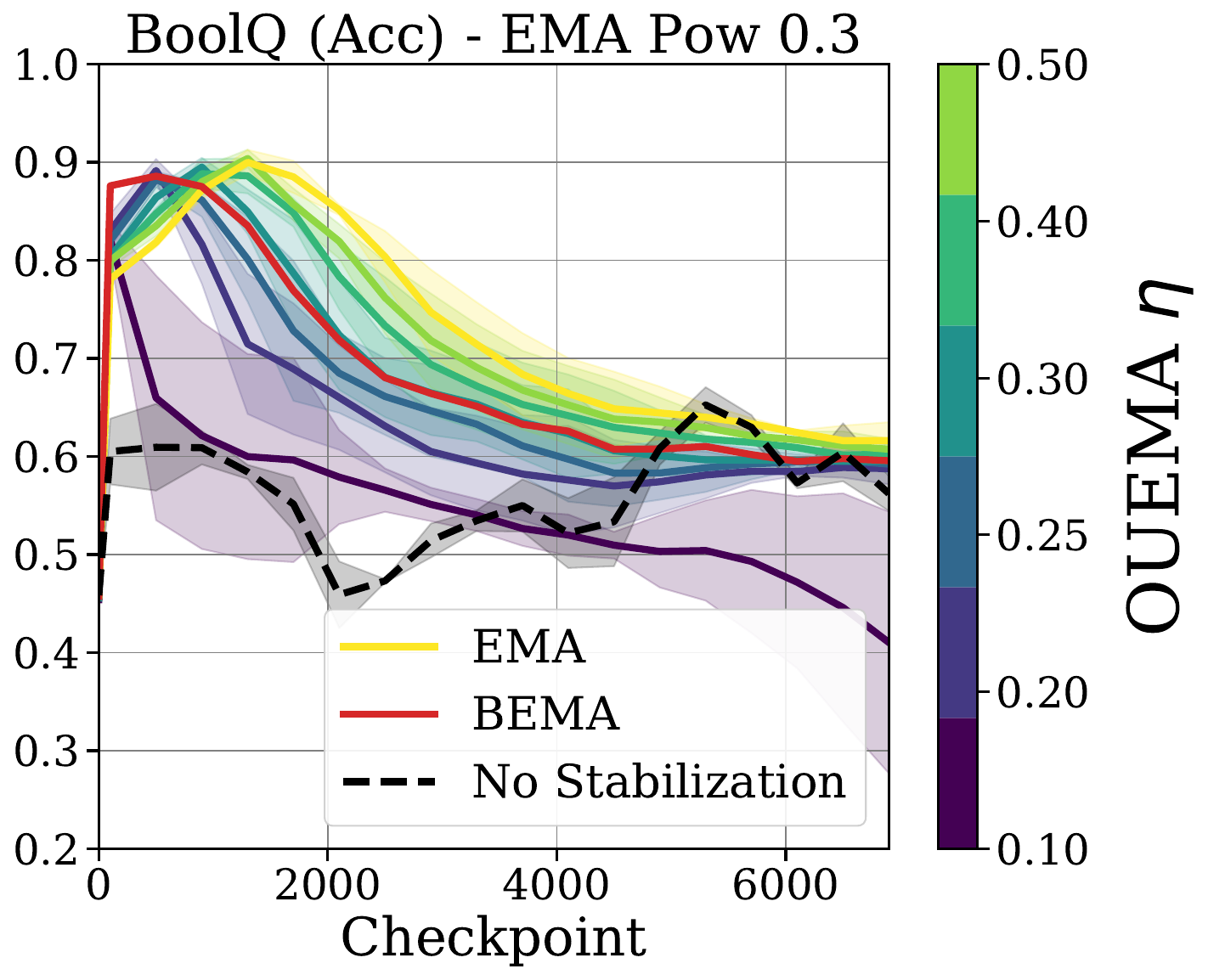} 
        \label{sfig:ouema3_boolq}
    } \\
    \subfigure[]{  
		\includegraphics[width=0.29\textwidth]{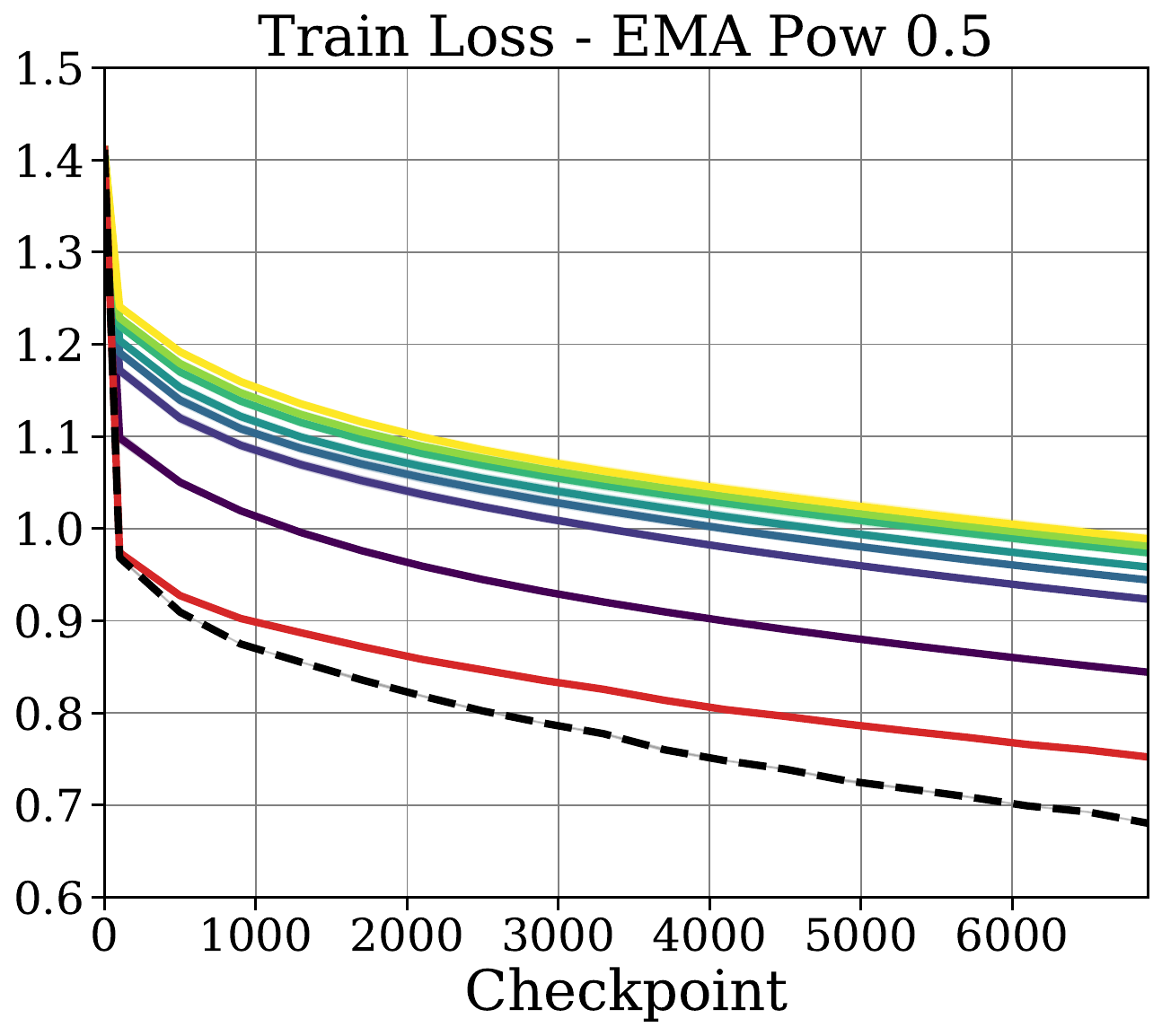}
		\label{sfig:ouema5_trainloss} 
	} \hfill \subfigure[]{
        \includegraphics[width=0.29\textwidth]{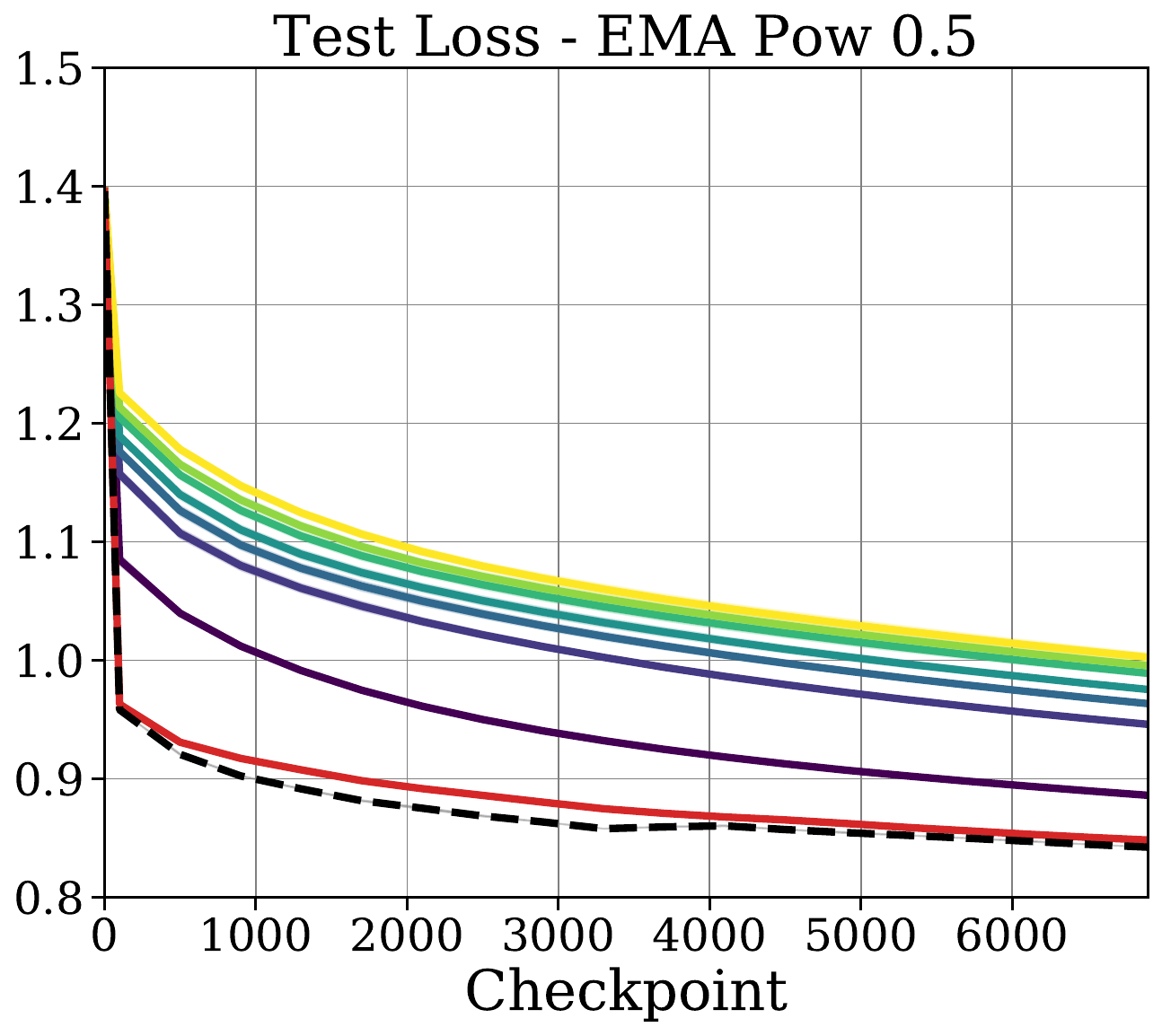} 
        \label{sfig:ouema5_testloss}
    } \hfill \subfigure[]{ 
        \includegraphics[width=0.29\textwidth]{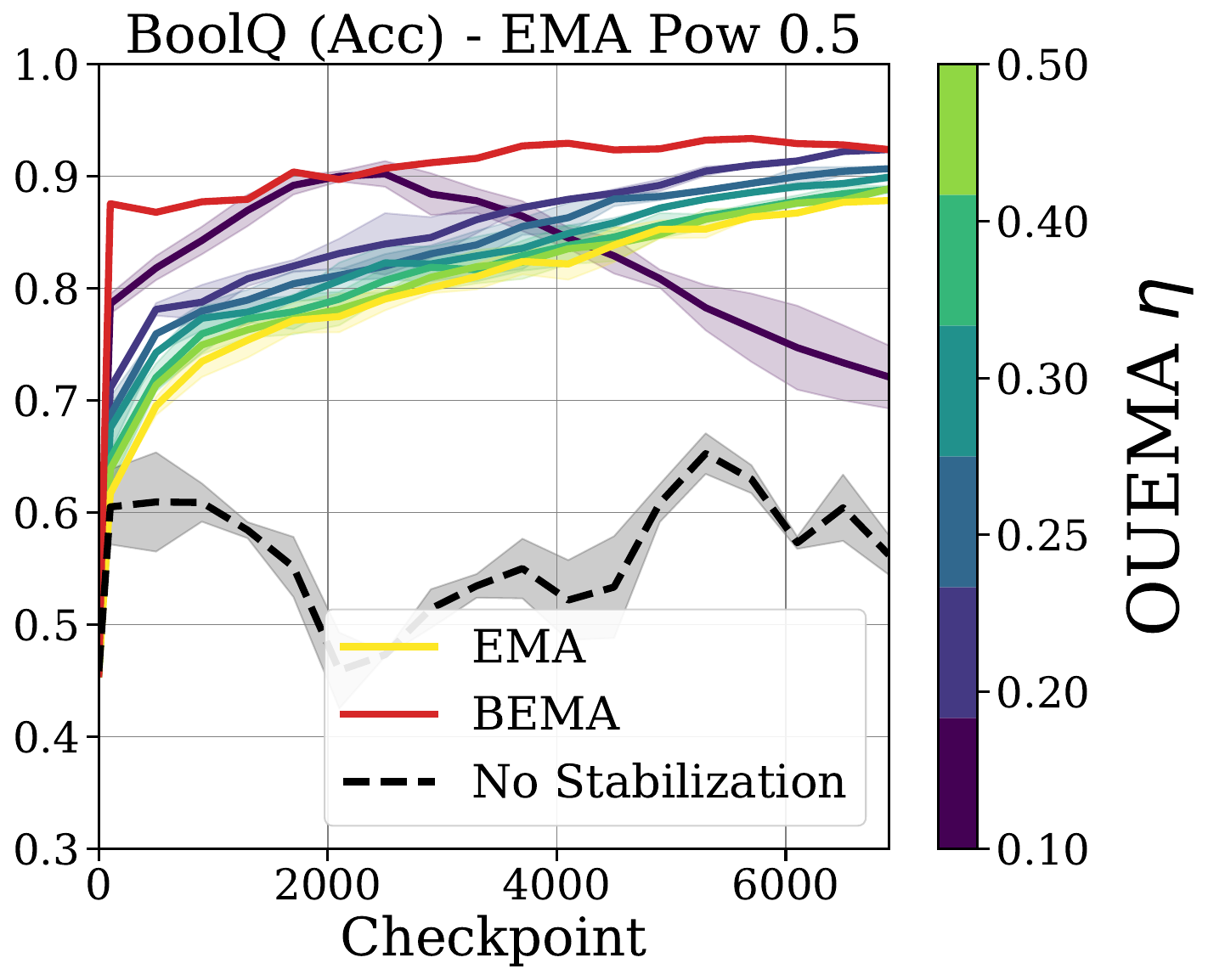} 
        \label{sfig:ouema5_boolq}
    } \\
	\caption{
        Effect of \ouema\  for different values of $\kappa$ from $\kappa = 0.0$ (no EMA \textbf{top}) to $\kappa = 0.5$ (strongest EMA \textbf{bottom}).  We compare to vanilla optimization (No stabilization, dashed), \ema\ (yellow), and \bema\ for the best choice of $\eta$ (red). \bema\ is generally superior to \ouema\  and \ema.
    }  
	\label{fig:ouema} 
\end{figure}

\begin{figure}
    \centering
    \subfigure[]{
        \includegraphics[width=0.29\textwidth]{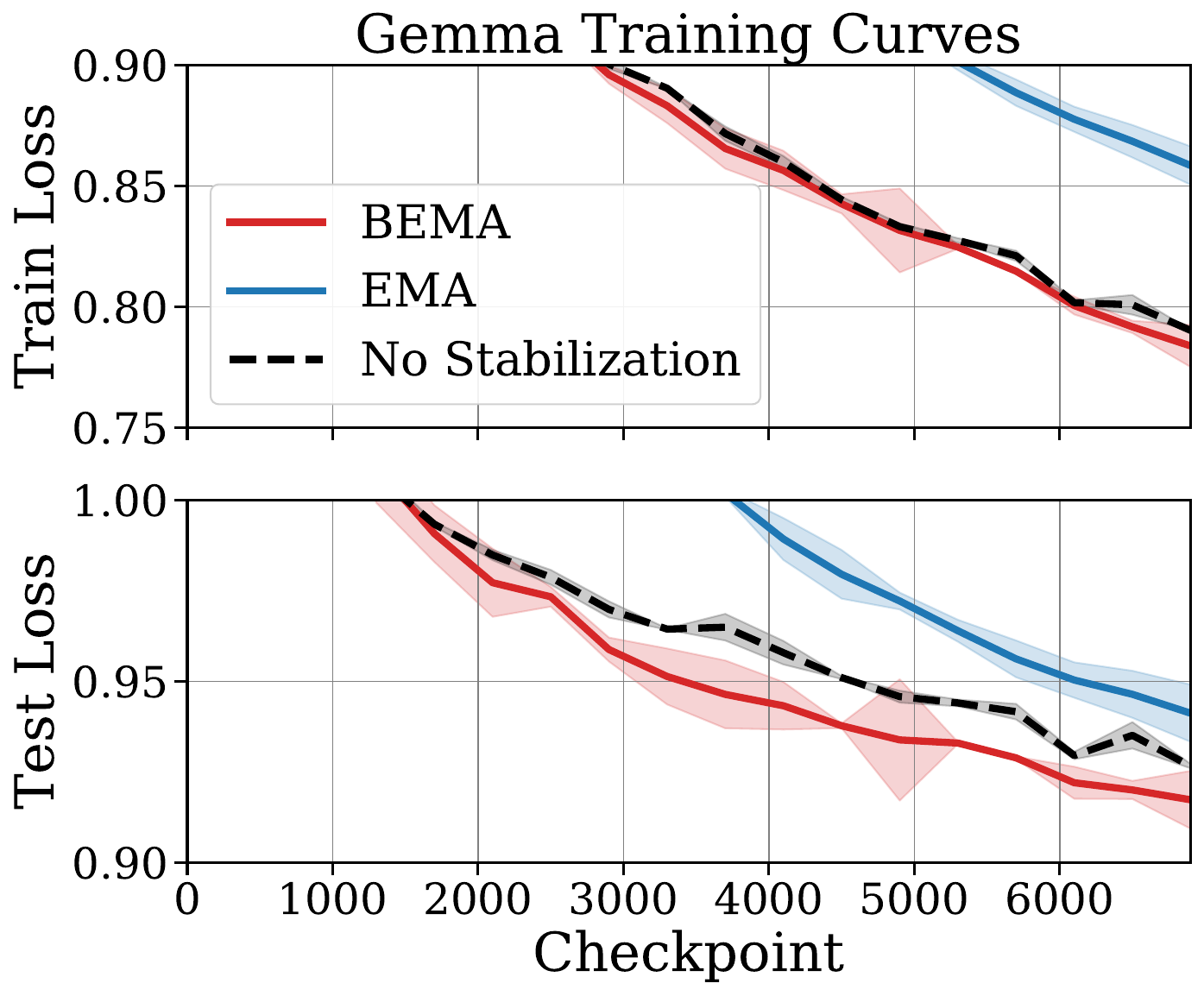}
        \label{sfig:gemma_train_curves}
    } \hfill
    \subfigure[]{
        \includegraphics[width=0.29\textwidth]{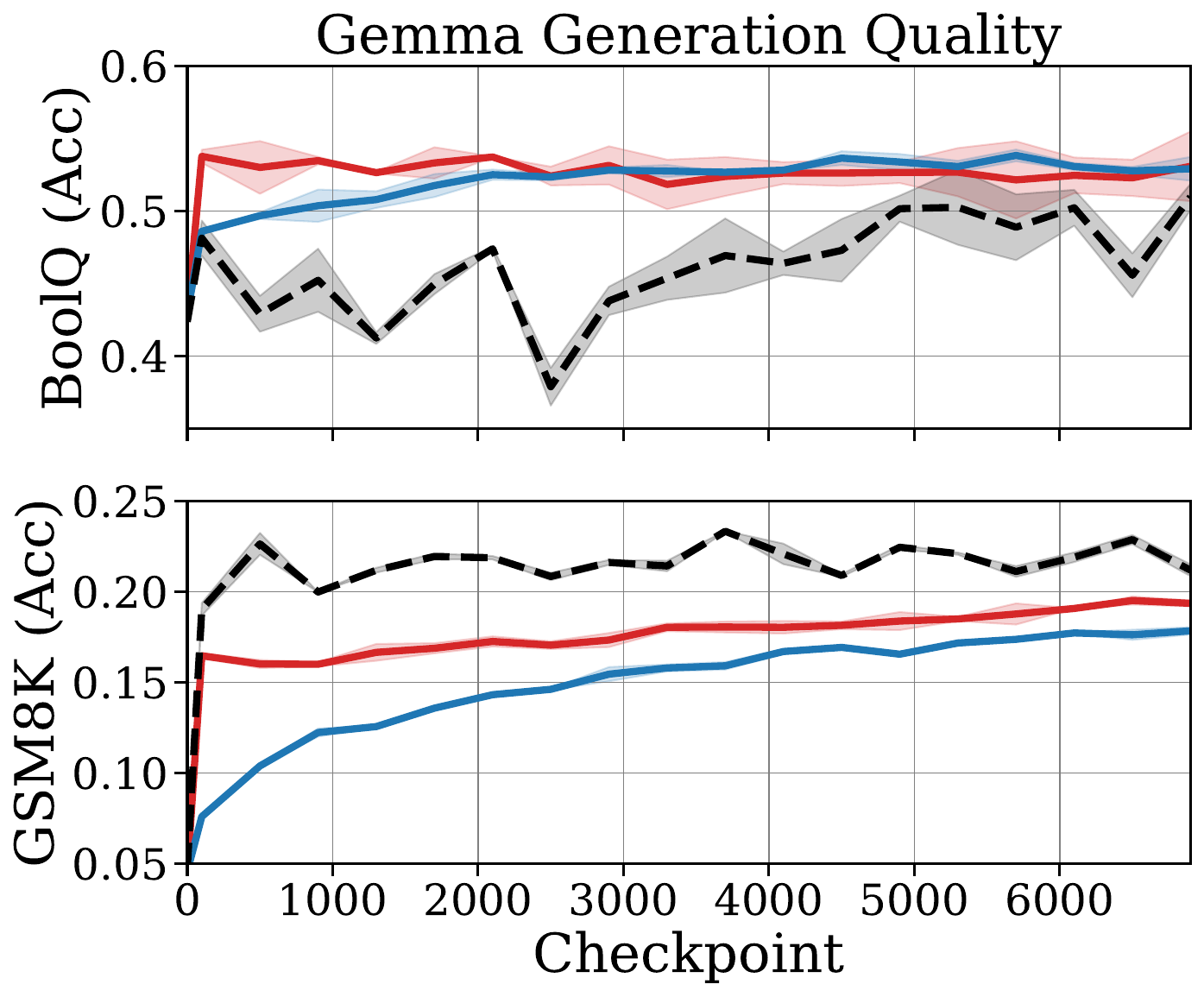}
        \label{sfig:gemma_gens_traincurve}
    } \hfill
    \subfigure[]{
        \includegraphics[width=0.29\textwidth]{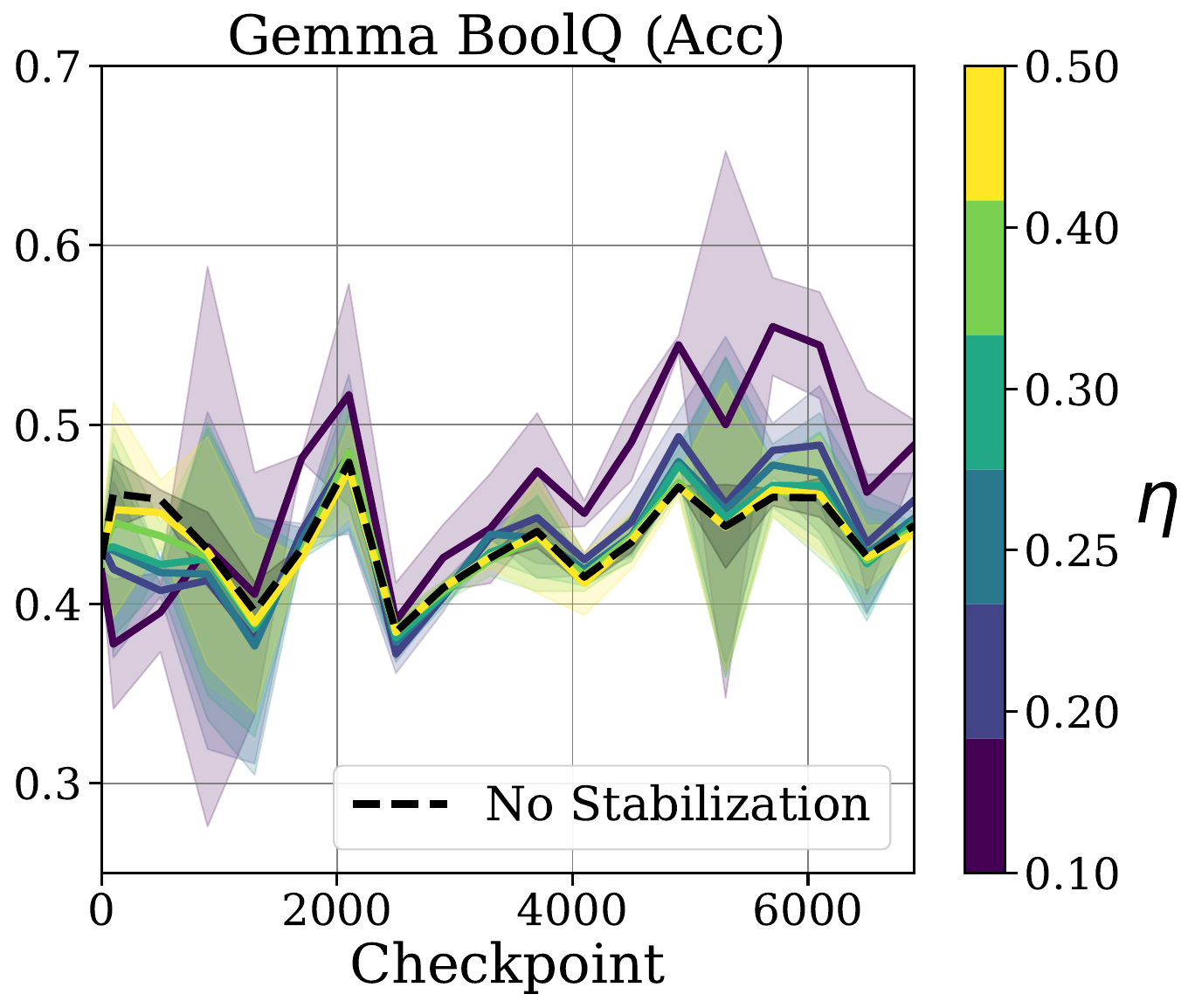}
        \label{sfig:gemma_boolq}
    }
    \caption{
        Performance of \bema\  and \ema\  on \gemma~  for \textbf{(a)} train and test loss, \textbf{(b)} generations on \boolq~ \textbf{(top)} and \gsmk~ \textbf{(bottom)}.  We also show the effect of \bema\  with $\kappa = 0$ (no EMA) for a variety of choices of $\eta$ in \textbf{(c)}.  In general, \bema\  accelerates and improves on \ema\  performance, but the effect is less pronounced than for \qwen.
    }
    \label{fig:gemma}
\end{figure}

\begin{figure}
    \centering
    \subfigure[]{
        \includegraphics[width=0.29\textwidth]{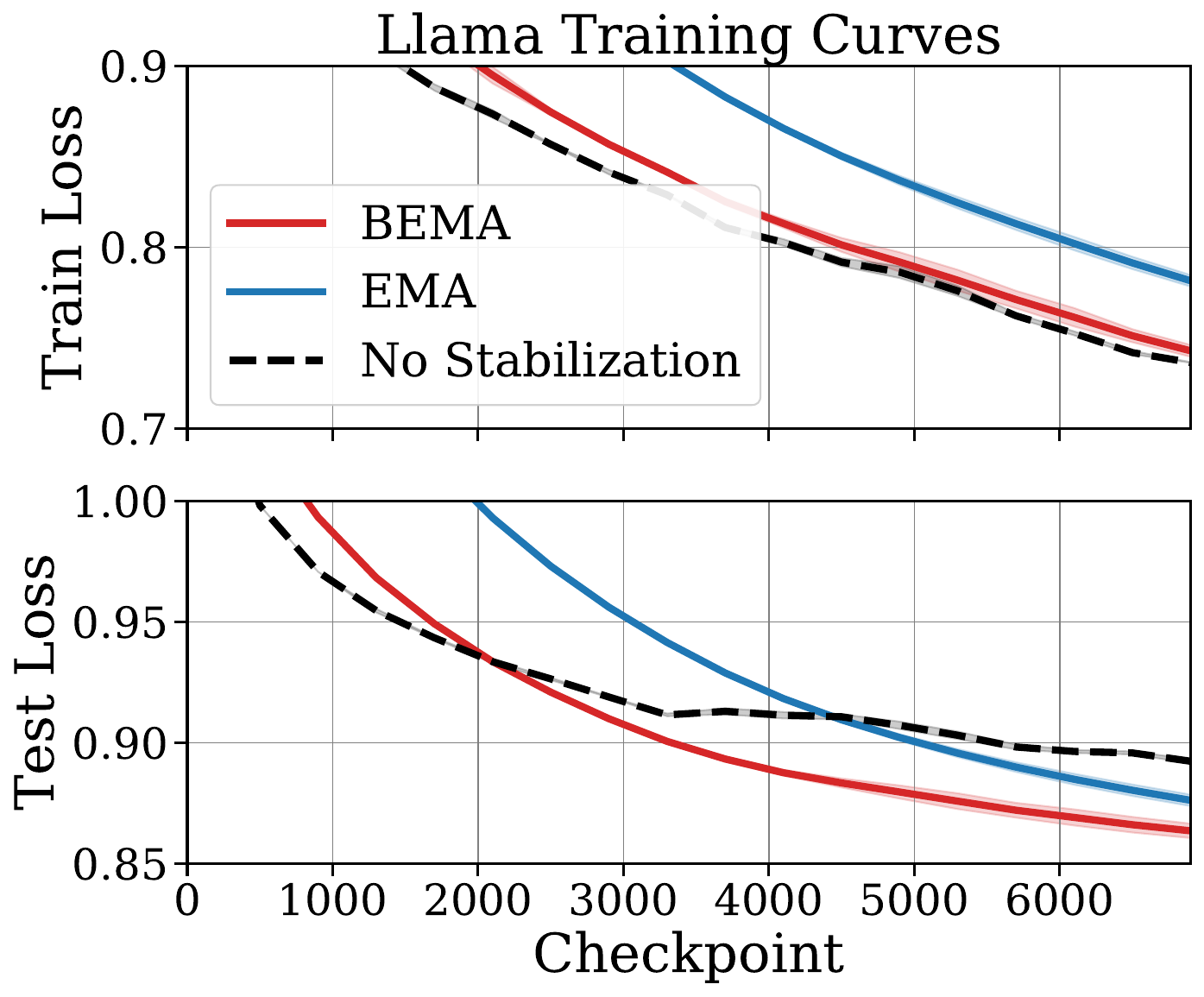}
        \label{sfig:llama_train_curves}
    } \hfill
    \subfigure[]{
        \includegraphics[width=0.29\textwidth]{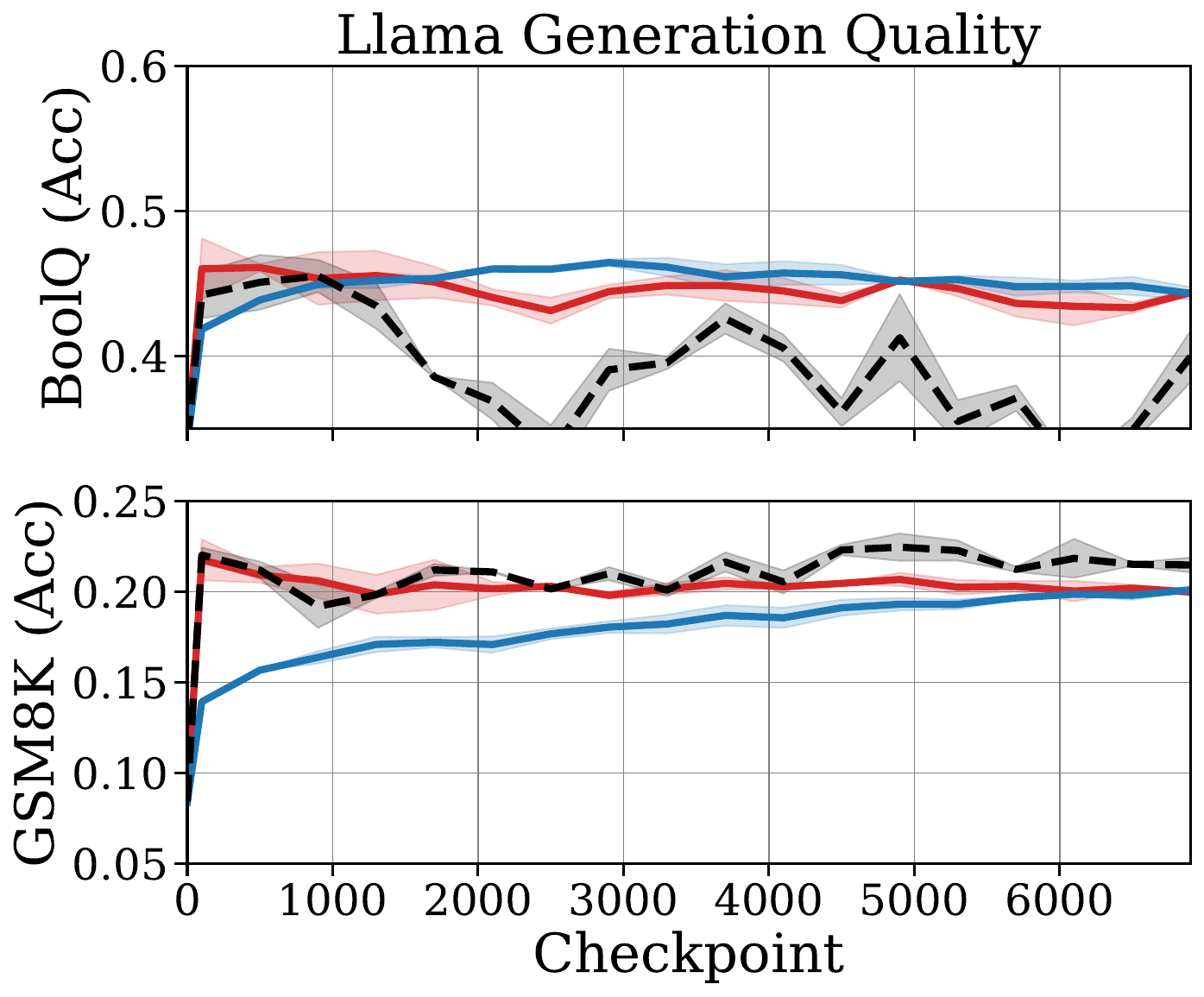}
        \label{sfig:llama_gens_traincurve}
    } \hfill
    \subfigure[]{
        \includegraphics[width=0.29\textwidth]{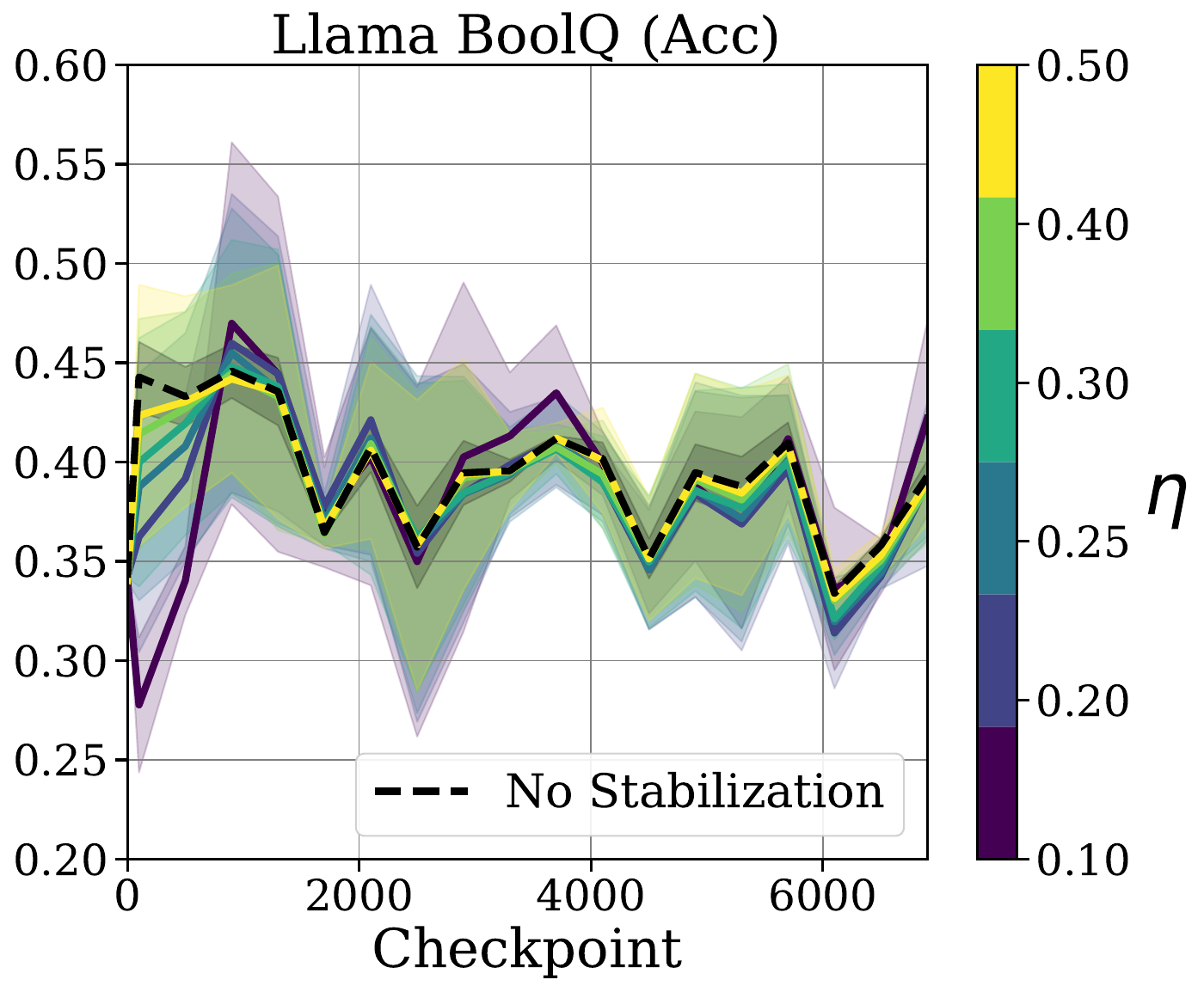}
        \label{sfig:llama}
    }
    \caption{
        Performance of \bema\  and \ema\  on \llama~  for \textbf{(a)} train and test loss, \textbf{(b)} generations on \boolq~ \textbf{(top)} and \gsmk~ \textbf{(bottom)}.  We also show the effect of \bema\  with $\kappa = 0$ (no EMA) for a variety of choices of $\eta$ in \textbf{(c)}.  It is clear that \bema\  is an improvement with respect to train and test loss, but \llama~ does not follow commands with sufficient frequency so as to perform sufficiently in either \gsmk~ or \boolq~ after finetuning on \tulu~ in order to recover a clear signal.
    }
    \label{fig:llama}
\end{figure}

\section{Additional Theoretical Results and Proofs}\label[appendix]{app:theory}

In this appendix, we provide formal proofs of the results in the main text.  We begin by proving several elementary facts about the Ornstein-Uhlenbeck process and general diffusions, as well as the lower bound for well-behaved estimators.  We then prove several results about $\mumle$ as consequences of a general theorem and then conclude by proving upper bounds on the performance of $\muouema$.

\subsection{Technical Preliminaries}\label{subsec:ou_properties}
We begin by recalling a version of the classic Girsanov theorem, which is indispensible for our analysis of the Maximum Likelihood Estimator.  For more details on generalizations and applications of Girsanov's theorem, we refer the reader to \citet{le2016brownian,liptser2013statistics1}.  The form of this result we use is as follows.
\begin{theorem}[Girsanov's Theorem]\label{prop:girsanov}
    Let $f: \rr^d \to \rr$ be a differentiable function and suppose that
    \begin{align}
        \ee\left[ \sqrt{\int_0^T \norm{\bSigma^{-1 }\nabla f(W_t - \mu)}^2 d t} \right] &< \infty, \label{eq:ui_martingale} \\
        \ee\left[ \exp\left( \frac 12 \int_0^T \bSigma^{-1} \nabla f(W_t - \mu) d W_t \right) \right] &< \infty. \label{eq:kazamaki}
    \end{align}
    Then $\thetamu$, the solution to
    \begin{align}
        d \thetamu_t = - \nabla f(\thetamu_t - \mu) d t + \bSigma d W_t,  \quad \thetamu_0 = \theta_0, \label{eq:thetamu}     
    \end{align}
    exists and
    \begin{align}
        \frac{d \ppmu}{d \ppW}\left((\theta_t)_{0 \leq t \leq T}\right)= \exp\left( - \int_0^t  \inprod{\bSigma^{-2}\nabla f(\thetamu_s - \mu)}{ d \thetamu_s} - \frac 12 \int_0^t \norm{\bSigma^{-1} \nabla f(\thetamu_s - \mu)}^2 d s \right),
    \end{align}
    where $\ppW$ is the Wiener measure.
\end{theorem}
\begin{proof}
    By \citet[Corollary 5.17]{le2016brownian}, \eqref{eq:ui_martingale} implies that the process $L_t = \int_0^t \nabla f(W_s - \mu) d W_s$ is a uniformly integrable martingale for $t \in [0, T]$.  As \eqref{eq:kazamaki} is precisely Kazamaki's condition (cf. \citet[Theorem 5.23]{le2016brownian}), Girsanov's theorem (\citet[Theorem 5.22]{le2016brownian}) implies the result.  A one-dimensional version of this result is also given, e.g., in \citet[Theorem 1.12]{kutoyants2013statistical}.
\end{proof}
\begin{remark}
    Recall that Novikov's condition (cf. \citet{le2016brownian,liptser2013statistics1}) is sufficient to ensure that the conclusion of \Cref{prop:girsanov} holds and is often easier to verify than \eqref{eq:ui_martingale} and \eqref{eq:kazamaki}.  Unfortunately, for our main application below, that of an OU process, for Novikov's condition to hold we would require $\bSigma^{-1} \bA \prec 2 \sqrt{\eta} \cdot \bI$, with $\bSigma$ and $\bA$ as in \eqref{eq:sde}.  As this is unnecessarily restrictive, we instead apply \emph{Kazamaki's Criterion} (cf. \citet[Theorem 5.23]{le2016brownian}), which is a more general condition that is satisfied by the OU process and thus allows us to apply Girsanov's theorem in this case.
\end{remark}
We now apply this result to the OU process explicitly.
\begin{proposition}\label{prop:girsanov_ou}
    Let $(\theta_t)_{0 \leq t \leq T}$ be the solution to the OU process \eqref{eq:sde} with $\bA, \bSigma \in \rr^{d \times d}$ symmetric positive definite and let $\pp^{\mustar}$ denote the measure of paths under this law.  If $\pp^{W}$ is the Wiener measure, then it holds that
    \begin{align}
        \log \frac{d \pp^{\mustar}}{d \pp^W}(\theta_t) &= - \frac{1}{\eta}\int_0^T \inprod{\bSigma^{-2} \bA (\mustar - \theta_t)}{d \theta_t} - \frac 1{2 \eta} \int_0^T \norm{\bSigma^{-1} \bA (\mustar - \theta_t)}^2 ~d t. \label{eq:log_likelihood_ou}
    \end{align}
\end{proposition}
\begin{proof}
    We apply \Cref{prop:girsanov} with $f(\theta) = \frac 12 \theta^\top \bA \theta$.  Note that
    \begin{align}
        \nabla f(\theta - \mustar) = \bA(\mustar - \theta) \quad \text{and} \quad \nabla^2 f(\theta - \mustar) = \bA.
    \end{align}
    Replacing $\bSigma$ by $\sqrt{\eta} \cdot \bSigma$ in Girsanov's theorem above yields the result, given that \eqref{eq:ui_martingale} and \eqref{eq:kazamaki} hold.  Thus it remains to establish these inequalities.  The first inequality holds by Holder, the linearity of expectation, and the fact that Gaussians have finite second moments:
    \begin{align}
        \ee\left[ \sqrt{\int_0^T \norm{\bSigma^{-1 }\nabla f(W_t - \mu)}^2 d t} \right] &\leq \sqrt{\ee\left[ \int_0^T \norm{\bSigma^{-1 }\nabla f(W_t - \mu)}^2 d t \right]} \\
        &= \sqrt{\int_0^T \ee\left[ \norm{- \bSigma^{-1 }\bA W_t}^2 \right] d t}  \\
        &= \sqrt{T \cdot \trace\left( \bSigma^{-2} \bA^2 \right)}  < \infty. 
    \end{align}
    To establish Kazamaki's criterion \eqref{eq:kazamaki}, we may directly compute that
    \begin{align}
        \exp\left( \frac 12 \int_0^T \inprod{\bSigma^{-1} \bA (\mustar - W_t)} {d W_t}\right) &= \exp\left( \frac 12 \inprod{\bSigma^{-1} \bA \mustar}{W_T} - \frac 12 \int_0^T \inprod{\bSigma^{-1} \bA  W_t}{d W_t}  \right).
    \end{align}
    By Ito's rule, it holds that
    \begin{align}
        \int_0^T \inprod{\bSigma^{-1} \bA W_t}{d W_t} = \frac 12 \inprod{\bSigma^{-1} \bA W_T}{W_T} - \trace(\bSigma^{-1} \bA) T.
    \end{align}
    As $\bSigma^{-1} \bA$ is positive definite, it then holds that
    \begin{align}
        \exp\left( \frac 12 \inprod{\bSigma^{-1} \bA \mustar}{W_T} - \frac 12 \int_0^T \inprod{\bSigma^{-1} \bA  W_t}{d W_t}  \right) \leq \exp\left( \frac 12 \inprod{\bSigma^{-1} \bA \mustar}{W_T}+ \trace\left( \bSigma^{-1} \bA \right) T \right).
    \end{align}
    The finiteness of the expectation of this last expression then follows from the fact that $W_T$ is Gaussian and the exponent is an affine function thereof.  Thus, \eqref{eq:kazamaki} holds and the result follows.
\end{proof}

We now recall several useful properties of the OU process.  To begin, we recall the standard fact that \eqref{eq:sde} admits the following closed form solution (see, e.g., \citet{le2016brownian,mandt2015continuous}):
\begin{align}
    \theta_t &= e^{-\bA t} \theta_0 + \left( \bI - e^{-\bA t} \right) \mustar + \sqrt{\eta}   \int_0^t e^{-\bA (t - s)} \bSigma ~ dW_s, \label{eq:ou}
\end{align} 
which we use.  Critically, \eqref{eq:ou} implies that $\theta_t$ is a Gaussian process with mean $\mu_t = e^{-\bA t} \theta_0 + \left( \bI - e^{-\bA t} \right) \mustar$ and a simple covariance kernel, given in the following lemma.
\begin{lemma}\label{lem:ou_cov}
    Let $(\theta_t)_{0 \leq t \leq T}$ be the solution to the OU process \eqref{eq:sde} with $\bA, \bSigma \in \rr^{d \times d}$ symmetric positive definite.  Then, for $0 \leq s < t \leq T$, we have that
    \begin{align}
        \Cov(\theta_t, \theta_s) &= K(t, s) = \frac{\eta \cdot \bA^{-1}}{2} \int_0^s e^{-\bA (t - u)} \bSigma^2 e^{-\bA (s - u)} d u \preceq \eta \cdot \normop{\bSigma}^2 \frac{\bA^{-1}}{2} \left( e^{-\bA(t - s)} - e^{- \bA(t + s)} \right). \label {eq:covariance_ou}
    \end{align}
    Moreover, when $\bSigma = \sigma \eye$, it holds that
    \begin{align}
        \Cov(\theta_t, \theta_s) = \frac{\sigma^2 \eta}{2} \bA^{-1} \left( e^{-\bA \abs{t - s}} - e^{-\bA (t + s)} \right). \label{eq:ou_cov_iso}
    \end{align}
\end{lemma}
\begin{proof}
    This is a standard fact about OU processes.  See, e.g. \citet{le2016brownian,kutoyants2013statistical,mandt2015continuous}.  Indeed, this follows immediately from \eqref{eq:ou}.
\end{proof}
We now require three lemmata that handle the first and second order moments of transformations of the OU process we use throughout the paper.  The first controls the first two moments of the total displacement of the OU process.
\begin{lemma}\label{lem:bias_variance_ou_diff}
    Let $\theta_t$ denote the solution to \eqref{eq:sde} given by \eqref{eq:ou}.  Then it holds that
    \begin{align}
        \ee\left[ \theta_T - \theta_0 \right] &= \left( \eye - e^{- \bA T} \right)\left( \mustar - \theta_0 \right) \qquad \text{and} \qquad
        \Cov\left( \theta_T - \theta_0 \right) \preceq \eta \cdot \normop{\bSigma}^2\cdot  \bA^{-1} \left( \eye - e^{-2 \bA T} \right) 
    \end{align}
    with equality in the variance when $\bSigma = \sigma \eye$.  
\end{lemma}
\begin{proof}
    By \eqref{eq:ou}, it holds that
    \begin{align}
        \theta_t &= e^{-\bA t} \theta_0 + \left( \bI - e^{-\bA t} \right) \mustar +  \sqrt{\eta} \int_0^t e^{-\bA (t - s)} \bSigma ~ dW_s. 
    \end{align}
    Note that the expectation of the final term is zero because this is a martingale.  The first equality then follows immediately. For the variance, we observe that because $\theta_0$ is deterministic, it holds by \Cref{lem:ou_cov} that
    \begin{align}
        \Cov(\theta_T - \theta_0) &= \Cov(\theta_T) \preceq \eta \normop{\bSigma}^2 \cdot \bA^{-1} \left( \eye - e^{-2 \bA T} \right),
    \end{align} 
    with equality in the case that $\bSigma = \sigma \eye$.  The result follows.
\end{proof}
We now require an analogous result for the time average of a trajectory of the OU process.

\begin{lemma}\label{lem:bias_variance_ou_int}
    Let $\theta_t$ be the solution to \eqref{eq:sde} given by \eqref{eq:ou}.  Then it holds that
    \begin{align}
        \ee\left[ \frac 1T \int_0^T \theta_t ~d t \right] &= \mustar - \frac 1T \bA^{-1} \left( \eye - e^{- \bA T} \right) \left( \mustar - \theta_0 \right).
    \end{align}
    Moreover,
    \begin{align}
        \eta \lambdamin(\Sigma)^2 \bA^{-2} \left( T \cdot \eye  - \bA^{-1}\left[ 2 \left( \eye - e^{-\bA T} \right) - \frac 12 \left( \eye - e^{- 2 \bA T} \right) \right]  \right) &\preceq \Cov\left(  \int_0^T \theta_t ~d t  \right) \\
        &\preceq T \cdot \eta \normop{\bSigma}^2  \cdot  \bA^{-2}.
    \end{align}
    In the case that $\bSigma = \sigma \eye$, the first inequality above is an equality.
\end{lemma}
\begin{proof}
    For the first statement, note that by the lineary of expectation it holds that
    \begin{align}
        \ee\left[ \frac 1T \int_0^T \theta_t d t \right] &= \frac 1T \int_0^T \ee\left[ \theta_t \right] d t \\
        &= \mustar - \frac 1T \int_0^T e^{-\bA t} \left( \mustar - \theta_0 \right) d t \\
        &= \mustar - \frac 1T \bA^{-1} \left( \eye - e^{- \bA T} \right) \left( \mustar - \theta_0 \right).
    \end{align}
    For the covaraiance, we apply \Cref{lem:ou_cov} and see that by symmetry
    \begin{align}
        \Cov\left( \int_0^T \theta_t ~ dt \right) = \int_0^T \int_0^T \Cov(\theta_t, \theta_u) d u d t &= 2 \int_0^T \int_0^T K(t, u) d u d t,
    \end{align}
    where $K(t,u)$ is as in \Cref{lem:ou_cov}.  We have that
    \begin{align}
        \eta \lambdamin(\bSigma)^2 \frac{\bA^{-1}}{2} \left( e^{- \bA \abs{t - s}} - e^{- \bA (t + s)} \right) \preceq K(s, t) \preceq \eta \lambdamax(\bSigma)^2 \frac{\bA^{-1}}{2} \left( e^{- \bA \abs{t - s}} - e^{- \bA (t + s)} \right). 
    \end{align}  
    Moreover, we compute
    \begin{align}
        \int_0^T \int_0^t \bA^{-1} \left( e^{- \bA \abs{t - s}} - e^{- \bA (t + s)} \right) ~d s d t &= \bA^{-2} \int_0^T \left( \eye - 2 e^{- \bA t} + e^{- 2 \bA t} \right) ~ d t \\
        &= \bA^{-2} \left( T \cdot \eye  - \bA^{-1}\left[ 2 \left( \eye - e^{-\bA T} \right) - \frac 12 \left( \eye - e^{- 2 \bA T} \right) \right]  \right).
    \end{align}
    Plugging this into the above display yileds the left hand side inequality, as well as the equality when $\lambdamin(\bSigma) = \lambdamax(\bSigma)$.  For the upper bound, we see that by the positive definiteness of $\bA$, we may diagonalize $\bA$ and it suffices to demonstrate that for any $x \geq 0$, it holds that
    \begin{align}
        2 (1 - e^{-x}) - \frac{1 - e^{- 2 x}}{2} \geq 0.
    \end{align}
    Letting $u = e^{-x}$, we see that this is equivalent to showing that $u^2 - 4 u + 3 \geq 0$ when $0 \leq u \leq 1$, which is immediate.  The result follows.
\end{proof} 
Finally, we require control on the covariance between the total displacement and the time average of the OU process.
\begin{lemma}\label{lem:covariance_diff_int}
    Let $\theta_t$ be the solution to \eqref{eq:sde} given by \eqref{eq:ou}.  Then it holds that
    \begin{align}
        \eta \lambdamin(\bSigma)^2 \frac{\bA^{-2}}{2} \left( \eye - 2 e^{- \bA T} + e^{- 2 \bA T} \right) \preceq \Cov\left( \theta_T - \theta_0, \int_0^T \theta_t d t \right) \preceq \eta \normop{\bSigma}^2 \frac{\bA^{-2}}{2} ,
    \end{align}
    with equality on the left hand side $\bSigma$ is a scalar multiple of the identity.
\end{lemma}
\begin{proof}
    By linearity of expectation and the fact that $\theta_0$ is deterministic, it holds that
    \begin{align}
        \Cov\left( \theta_T - \theta_0, \int_0^T \theta_t ~ d t \right) &= \Cov\left( \theta_T, \int_0^T \theta_t ~ d t \right) \\
        &= \int_0^T \Cov\left( \theta_T, \theta_t \right) ~ d t \\
        &= \int_0^T K(T, t) ~ d t,
    \end{align}
    where $K(t, s)$ is as in \Cref{lem:ou_cov}.  Using the bounds on $K(t, s)$ from \Cref{lem:ou_cov}, we see that
    \begin{align}
        \int_0^T \frac{\bA^{-1}}{2} \left( e^{-\bA \abs{T - t}} - e^{-\bA (T + t)} \right) ~ d t &= \frac{\bA^{-2}}{2} \left( \eye -  e^{- \bA T} + e^{- 2 \ba T} \right).
    \end{align}
    The result follows.
\end{proof}
Finally, we precisely characterize the error of $\theta_T$ as an estimator of $\mustar$.
\begin{proposition}
    For $T > 0$, let $(\theta_t)_{0 \leq t \leq T}$ be the solution to the OU process \eqref{eq:sde} with $\bA, \bSigma \in \rr^{d \times d}$ symmetric positive definite.  Then it holds that
    \begin{align}
        \ee_{\mustar}\left[ \norm{\theta_T - \mustar}^2 \right] \leq \norm{e^{- \bA T} \left( \theta_0 - \mustar \right)}^2 + \eta \cdot \normop{\bSigma}^2 \cdot \trace\left( \bA^{-1} \right).
    \end{align}
    If $\bSigma = \sigma \eye$, then the inequality becomes an equality.
\end{proposition}
\begin{proof}
    By the bias-variance decomposition, it holds that
    \begin{align}
        \ee_{\mustar}\left[ \norm{\theta_T - \mustar}^2 \right] &= \norm{\ee_{\mustar}\left[ \theta_T \right] - \mustar}^2 + \trace\left( \Cov(\theta_T) \right).
    \end{align}
    Applying \Cref{lem:bias_variance_ou_diff} concludes the proof.
\end{proof}

\subsection{Lower Bound on Mean Squared Error}\label{subsec:lower_bound}
We now state and prove two lower bounds on the mean squared error of estimators of $\mustar$ based on the OU process.  Both bounds are a consequence of the Cramer-Rao inequality, the main approach in classical statistics to derive lower bounds in parametric estimation problems.  We first prove a result for unbiased estimators, which is a consequence of the Cramer-Rao inequality.
\begin{proposition}\label{prop:lower_bound_unbiased}
    Let $(\theta_t)_{0 \leq t \leq T}$ be the solution to the OU process \eqref{eq:sde} with $\bA, \bSigma \in \rr^{d \times d}$ symmetric positive definite and let $\muhat$ be an unbiased estimator of $\mustar$, i.e., $\ee_{\mustar}\left[ \muhat \right] = \mustar$.  Then it holds that
    \begin{align}
        \ee\left[ \norm{\muhat - \mustar}^2 \right] \geq \frac{\eta \cdot \trace\left( \bA^{-1} \bSigma^2 \bA^{-1} \right)}{T}.
    \end{align}
    In particular, if $\bSigma = \sigma \eye$, then it holds that
    \begin{align}
        \ee\left[ \norm{\muhat - \mustar}^2 \right] \geq \frac{\sigma^2 \eta}{T} \cdot \trace\left( \bA^{-2} \right).
    \end{align}
\end{proposition}
\begin{proof}
    We apply the Cramer-Rao inequality to diffusions, as in \citet{liptser2013statistics1,liptser2013statistics2,kutoyants2013statistical}.  Indeed, by the Cramer-Rao inequality in multiple dimensions (see, e.g., \citet[\S7.8]{liptser2013statistics1} or \citet[Theorem 6.1]{lehmann2006theory}) it holds that
    \begin{align}
        \ee_{\mustar}\left[ \left( \muhat - \ee_{\mustar}[\muhat] \right) \left( \muhat - \ee_{\mustar}[\muhat] \right)^\top \right] \succeq \left( \nabla_{\mustar} \ee_{\mustar}[\muhat] \right) \fisher_{\mustar}^{-1} \left( \nabla_{\mustar} \ee_{\mustar}[\muhat] \right)^\top, \label{eq:cramer_rao_app}
    \end{align}
    where
    \begin{align}
        \fisher_{\mustar} = \ee_{\mustar}\left[- \nabla^2 \log p_{\mustar}(\theta) \right]
    \end{align}
    is the Fisher information matrix of the process $\ppmu$ with respect to the parameter $\mustar$ and $\nabla_{\mustar} \ee_{\mustar}[\muhat]$ is the Jacobian of the expectation of $\muhat$ with respect to $\mustar$.  In the case that $\muhat$ is unbiased, we have that $\nabla_{\mustar} \ee_{\mustar}[\muhat] = \eye$, and thus the Cramer-Rao inequality tells us that $\Cov(\muhat) \succeq \fisher_{\mustar}^{-1}$.  By the bias-variance decomposition, it holds that
    \begin{align}
        \ee_{\mustar}\left[ \norm{\muhat - \mustar}^2 \right] &= \norm{\ee_{\mustar}\left[ \muhat \right] - \mustar}^2 + \trace\left(\Cov(\muhat) \right).
    \end{align}
    In the case that $\muhat$ is unbiased, then, we have that $\ee_{\mustar}\left[ \norm{\muhat - \mustar}^2 \right] \geq \trace\left( \fisher_{\mustar}^{-1} \right)$.  We now use \Cref{prop:girsanov} to compute the Fisher information matrix for the OU process.  Indeed, we have by \Cref{prop:girsanov_ou} that
    \begin{align}
        \log p_{\mustar}(\theta_t) &= - \eta^{- \nicefrac 12}\int_0^T \inprod{\bSigma^{-1} \bA (\mustar - \theta_t)}{d \theta_t} - \frac 1{2 \eta} \int_0^T \norm{\bSigma^{-1} \bA (\mustar - \theta_t)}^2 ~d t
    \end{align}
    Taking the Hessian with respect to $\mustar$ yields
    \begin{align}
        - \nabla^2 \log p_{\mustar}(\theta_t) &= \eta^{-1} \int_0^T \bA \bSigma^{-2} \bA ~ d t = \eta^{-1} T \bA \bSigma^{-2} \bA.
    \end{align}
    The result follows.
\end{proof}
In addition \Cref{prop:lower_bound_unbiased}, which only holds for unbiased estimators, we also have a lower bound that holds when the bias is a \emph{contraction}, i.e., is Lipschitz with parameter $L < 1$.  This also follows from the Cramer-Rao inequality.
\begin{proposition}
    \label{prop:lower_bound_contraction}
    Let $(\theta_t)_{0 \leq t \leq T}$ be the solution to the OU process \eqref{eq:sde} with $\bA, \bSigma \in \rr^{d \times d}$ symmetric positive definite and let $\muhat$ be an estimator of $\mustar$ such that the map $\mustar \mapsto \ee_{\mustar}\left[ \muhat \right] - \mustar$ is Lipschitz with constant $L < 1$.  Then it holds that
    \begin{align}
        \ee\left[ \norm{\muhat - \mustar}^2 \right] \geq \left( 1 - L \right)^2 \cdot \frac{\eta \cdot \trace\left( \bA^{-1} \bSigma^2 \bA^{-1} \right)}{T} \geq (1  - L)^2 \cdot \frac{\eta \lambdamin(\bSigma)^2 \cdot \trace\left( \bA^{-2} \right)}{T}.
    \end{align}
\end{proposition}
\begin{proof}
    We may apply the identical argument as in the proof of \Cref{prop:lower_bound_unbiased}, i.e., following from \eqref{eq:cramer_rao_app} we have
    \begin{align}
        \ee_{\mustar}\left[ \norm{\muhat - \mustar}^2 \right] &\geq \frac{\eta}{T} \cdot \trace\left( \left( \nabla_{\mustar}\ee_{\mustar}[\muhat] \right)^\top \bA^{-1} \bSigma^2 \bA^{-1} \left( \nabla_{\mustar}\ee_{\mustar}[\muhat] \right)  \right).
    \end{align}
    By the linearity of the Jacobian, it holds that
    \begin{align}
        \nabla_{\mustar}\ee_{\mustar}[\muhat] &= \nabla_{\mustar}\left(\ee_{\mustar} [\muhat] - \mustar \right) + \nabla_{\mustar}\mustar \succeq \left( 1 - \normop{\nabla_{\mustar}\left( \ee_{\mustar}[\muhat] - \mustar \right)} \right) \eye.
    \end{align}
    By the Lipschitz condition, we have that $\normop{\nabla_{\mustar}\left( \ee_{\mustar}[\muhat] - \mustar \right)} \leq L < 1$ and thus the result follows.
\end{proof}

\subsection{Maximum Likelihood Estimation}\label{subsec:mumle}
We now prove several results related to the Maximum Likelihood Estimator (MLE) $\mumle$ of the OU process.  We begin by providing an explicit formula for the MLE, which is an immediate consequence of \Cref{prop:girsanov_ou} and the definition of the MLE.  Such a formula, especially in one dimension, is well-known in the financial mathematics literature \citep{kutoyants2013statistical}, but we include it here for completeness.
\begin{theorem}\label{thm:mumle}
    For any $T > 0$, let $(\theta_t)_{0 \leq t \leq T}$ be the solution to the OU process \eqref{eq:sde} with $\bA, \bSigma \in \rr^{d \times d}$ symmetric positive definite.  Then the Maximum Likelihood Estimator (MLE) of $\mustar$ is given by
    \begin{align}
        \mumle_T &= \frac{\bA^{-1}}{T}  \left( \theta_T - \theta_0 \right) + \frac 1T \int_0^T \theta_t ~ d t.
    \end{align}
\end{theorem}
\begin{proof}
    We have by \Cref{prop:girsanov_ou} that the log likelihood function is given by
    \begin{align}
        \eta \cdot L(\mu) &= \eta \cdot \log \frac{d \ppmu}{d \pp^W} = - \int_0^T \inprod{\bSigma^{-2} \bA\left( \mu - \theta_t  \right)}{d \theta_t} - \frac 12 \int_0^T \norm{\bSigma^{-1} \bA\left( \mu - \theta_t  \right)}^2 d t.
    \end{align}
    Note that this function is strongly concave in $\mu$ and thus attains a unique maximum at the stationary point where $\nabla L(\mumle_T) = 0$.  Taking the gradient, we see that
    \begin{align}
        0 &= \nabla L(\mumle_T) = \int_0^T \inprod{\bSigma^{-2} \bA}{d \theta_t} - \int_0^t \bA \bSigma^{-2}  \bA \left( \mumle_T - \theta_t \right) d t \\
        &= \bA \bSigma^{-2} (\theta_T - \theta_0) - \bA \bSigma^{-2} \bA \left( T \cdot  \mumle_T - \int_0^T \theta_t d t \right).
    \end{align}
    Rearranging yields the desires conclusion.
\end{proof}
We now use this characterization of the MLE to derive its distributional properties.
\begin{corollary}\label{cor:mumle_distribution}
     Let $(\theta_t)_{0 \leq t \leq T}$ be the solution to the OU process \eqref{eq:sde} with $\bA, \bSigma \in \rr^{d \times d}$ symmetric positive definite.  Then it holds that
     \begin{align}
        \mumle_T \stackrel{d}{=} \mustar + \frac{\sqrt{\eta} \cdot \bA^{-1} \bSigma}{\sqrt{T}} \cdot \cN\left( 0, \eye \right),
     \end{align}
     where $\stackrel{d}{=}$ denotes equality in distribution.  In particular, it holds that
     \begin{align}
            \ee\left[ \norm{\mumle_T - \mustar}^2 \right]  = \frac{\eta \cdot \trace\left( \bA^{-1} \bSigma^2 \bA^{-1} \right)}{T}
     \end{align}
     and, in the special case that $\bSigma = \sigma \eye$, it holds that
     \begin{align}
        \ee\left[ \norm{\mumle_T - \mustar}^2 \right] = \frac{\sigma^2 \eta \cdot \trace\left( \bA^{-2} \right)}{T}.
     \end{align}
\end{corollary}
\begin{proof}
    By \Cref{thm:mumle} it holds that
    \begin{align}
        \mumle_T &= \frac{\bA^{-1}}{T} \left( \theta_T - \theta_0 \right) + \frac 1T \int_0^T \theta_t ~ d t \\
        &= \frac 1T \left( \bA^{-1} \int_0^T \bA (\mustar - \theta_t) ~ d t + \int_0^T \bSigma d W_t + \int_0^T \theta_t d t \right) \\
        &= \mustar + \frac{\bA^{-1} \bSigma W_T}{T}.
    \end{align}
    The result now follows from the fact that $W_T \sim \cN(0, T \eye)$. 
\end{proof}

We now prove a general result on the performance of estimates of the form $\mumle$, but with a possibly different choice of $\bA$.  Let
\begin{align}
    \mumletil_T(\bAtil) = \frac{\bA^{-1}}{T} \left( \theta_T - \theta_0 \right) + \frac 1T \int_0^T \theta_t ~ d t,
\end{align}
where $\bAtil$ is a symmetric positive definite matrix.  We have the following bound on the performance of such an estimator.
\begin{theorem}\label{thm:mumle_performance}
    Let $(\theta_t)_{0 \leq t \leq T}$ be the solution to the OU process \eqref{eq:sde} with $\bA, \bSigma \in \rr^{d \times d}$ symmetric positive definite.  Then it holds that
    \begin{align}
        \ee\left[ \norm{\mumletil_T(\bAtil) - \mustar}^2 \right] &\leq \frac{\eta \normop{\bSigma}^2 \cdot \trace\left( \bA^{-2} \right)}{T}  + \frac{\eta \normop{\bSigma}^2 \cdot \normop{\bAtil^{-1}}^2 \cdot \trace\left( \bA^{-1} \right)}{T^2} \\
        &\quad + \frac{\eta \normop{\bSigma}^2 \normop{\bAtil^{-1}}\trace\left( \bA^{-2} \right)}{T^2}  + \frac{\normop{\bAtil^{-1} - \bA^{-1}}^2 \norm{\mustar - \theta_0}^2}{T^2}
    \end{align}
    If $\bSigma = \sigma \eye$, and $\bAtil$ commutes with $\bA$, then it holds that
    \begin{align}\label{eq:oumlewrongA_bound}
        \ee\left[ \norm{\mumletil_T(\bAtil) - \mustar}^2 \right] &\leq \frac{\eta \sigma^2 \trace\left( \bA^{-2} \right)}{T} + \frac{\eta \sigma^2  \trace\left( \bAtil^{-2} \bA^{-1} \right) }{T^2} +  \frac{\eta \sigma^2  \trace\left( \bAtil^{-1}\bA^{-2} \right)}{T^2} \\
        &\quad  + \frac{\normop{\bAtil^{-1} - \bA^{-1}}^2 \norm{\mustar - \theta_0}^2}{T^2}.
    \end{align}
\end{theorem}
\begin{proof}
    We apply the bias-variance decomposition and bound each separately.  For the bias, we combine the first moment bounds of \Cref{lem:bias_variance_ou_diff,lem:bias_variance_ou_int} and the linearity of expectation to see that
    \begin{align}
        \ee\left[ \mumletil_T(\bAtil) - \mustar \right] &= \frac{\bAtil^{-1}}{T} \left( \eye - e^{- \bA T} \right) \left( \mustar - \theta_0 \right) - \frac{1}{T} \bA^{-1} \left( \eye - e^{-\bA T} \right) \left( \mustar - \theta_0 \right) \\
        &= \frac{\bAtil^{-1} - \bA^{-1}}{T} \left( \eye - e^{- \bA T} \right) \left( \mustar - \theta_0 \right).
    \end{align}
    For the variance, we apply \Cref{lem:bias_variance_ou_diff,lem:bias_variance_ou_int,lem:covariance_diff_int} to see that  
    \begin{align}
        \Var(\mumletil_T(\bAtil)) &= \Var\left( \frac {\bAtil^{-1}}T \theta_T \right) + \Var\left( \frac 1T \int_0^T \theta_t ~ d t \right) \\
        &\quad + \frac 2{T^2} \cdot \Cov\left( \bAtil^{-1} \theta_T,  \int_0^T \theta_t ~ d t \right)\\
        &\leq \frac{\eta \normop{\bSigma}^2 \cdot \trace\left( \bA^{-2} \right)}{T} + \frac{\eta \normop{\bSigma}^2 \cdot \normop{\bAtil^{-1}}^2 \cdot \trace\left( \bA^{-1} \right)}{T^2}+ \frac{\eta \normop{\bSigma}^2 \normop{\bAtil^{-1}}\trace\left( \bA^{-2} \right)}{T^2}.
    \end{align}
    The first result follows.  For the second result, we can simplify the variance expression using the fact that $\bSigma$ commutes with $\bA$ and $\bAtil$. We have that
    \begin{align}
        \Var(\mumletil_T(\bAtil)) &= \frac 1{T^2} \cdot \Var\left( \bAtil^{-1} \theta_T  \right) + \Var\left( \frac 1T \int_{0}^{T} \theta_t ~ d t \right) \\
        &\quad + \frac{2}{T^2} \cdot \Cov\left( \bAtil^{-1} \theta_T,  \int_0^T \theta_t ~ d t \right)\\
        &\leq \frac{\eta \sigma^2  \trace\left( \bAtil^{-2} \bA^{-1} \right) }{T^2} + \frac{\eta \sigma^2 \trace\left( \bA^{-2} \right)}{T} + \frac{\eta \sigma^2  \trace\left( \bAtil^{-1}\bA^{-2} \right)}{T^2}.
    \end{align}
    The second result follows.  
\end{proof}
We see that with respect to asymptotic in $T$ performance, the choice of $\bAtil$ is irrelevant, as it does not affect the leading term in the error bound.  On the other hand, the higher order terms of \eqref{eq:oumlewrongA_bound} suggest that $\bAtil$, subject to being maximally close to $\bA$, should be chosen so as to be sufficiently well conditioned in order to temper the additional variance (second term of \eqref{eq:oumlewrongA_bound}).

We now instantiate \Cref{thm:mumle_performance} in order to recover a bound on $\muema_T$.
\begin{corollary}
    Let $(\theta_t)_{0 \leq t \leq T}$ be the solution to the OU process \eqref{eq:sde} with $\bA, \bSigma \in \rr^{d \times d}$ symmetric positive definite and recall that
    \begin{align}
        \muema_T = \frac 1T \int_0^T \theta_t ~ d t.
    \end{align}
    Then
    \begin{align}
        \ee\left[ \norm{\muema_T - \mustar}^2 \right] &\leq \frac{\eta \normop{\bSigma}^2 \cdot \trace\left( \bA^{-2} \right)}{T} + \frac{\normop{\bA^{-1}}^2 \norm{\mustar- \theta_0}^2}{T^2}.
    \end{align}
    and in the case that $\bSigma = \sigma \eye$, it holds that
    \begin{align}
        \ee\left[ \norm{\muema_T - \mustar}^2 \right] &\leq \frac{\eta \sigma^2 \cdot \trace\left( \bA^{-2} \right)}{T} + \frac{\normop{\bA^{-1}}^2 \norm{\mustar- \theta_0}^2}{T^2}
    \end{align}
\end{corollary}
\begin{proof}
    Note that if we let $\bAtil = c \eye$ and send $c \uparrow \infty$, then we recover $\muema_T = \mumletil_T(\bAtil)$.  Thus, we may apply \Cref{thm:mumle_performance} to see that
    \begin{align}
        \ee\left[ \norm{\muema_T - \mustar}^2 \right] &\leq \frac{\eta \normop{\bSigma}^2 \cdot \trace\left( \bA^{-2} \right)}{T} + \frac{\normop{\bA^{-1}}^2 \norm{\mustar- \theta_0}^2}{T^2}.
    \end{align}
    Both results follow immediately from this bound.
\end{proof}
We also prove a lower bound for $\muema_T$ as an estimator of $\mustar$.
\begin{proposition}
    Let $(\theta_t)_{0 \leq t \leq T}$ be the solution to the OU process \eqref{eq:sde} with $\bA \in \rr^{d \times d}$ symmetric positive definite and $\bSigma = \sigma \eye$.  Suppose that for some $0 < c < 1$ it holds that $\lambdamax(\bA) T \leq \nicefrac c2$.  Then
    \begin{align}
        \ee\left[ \norm{\muema_T - \mustar}^2 \right] \geq (1 - c)^2 \norm{\mustar - \theta_0}^2.
    \end{align}  
\end{proposition}
\begin{proof}
    We use \Cref{lem:bias_variance_ou_int} and observe that
    \begin{align}
        \ee\left[ \norm{\muema - \mustar}^2 \right] &= \frac{\norm{\bA^{-1} \left( \eye - e^{- \bA T} \right) \left( \mustar - \theta_0 \right) }^2 }{T^2} + \Var(\muema_T) \\
        &\geq \frac{\norm{\bA^{-1} \left( \eye - e^{- \bA T} \right) \left( \mustar - \theta_0 \right) }^2 }{T^2}.
    \end{align}
    Note that it holds that
    \begin{align}
        \eye - T \bA + \frac{T^2}{2} \bA^2 \succeq e^{- \bA T} \succeq \eye - \bA T
    \end{align}
    and thus
    \begin{align}
        \eye - \frac{T}{2} \bA \preceq \frac{\bA^{-1}\left( \eye - e^{- \bA T} \right)}{T} \preceq \eye.
    \end{align}
    In particular
    \begin{align}
        \lambdamin\left(  \frac{\bA^{-1}\left( \eye - e^{- \bA T} \right)}{T}  \right) \geq \min\left( 1, \left( 1 - \nicefrac{T \lambdamax(\bA)}{2} \right)^2 \right)
    \end{align}
\end{proof}
Note that we could have derived a bound for $\mumle_T$ as a special case of \Cref{thm:mumle_performance}, but it would be less tight than that which we derived above; indeed, the simplicity of the model allowed us to precisely characterize the distribution of $\mumle_T$.

\subsection{Proofs related to \ouema}\label{subsec:additional_results}
In this section we prove the additional results mentioned in \Cref{sec:theory}, especially with respect to $\muouema$.  To begin, we prove that $\muouema$ is an unbiased estimator of $\mustar$.
\begin{proposition}\label{prop:ouema_unbiased}
    Let $(\theta_t)_{0 \leq t \leq T}$ be the solution to the OU process \eqref{eq:sde} with $\bA, \bSigma \in \rr^{d \times d}$ symmetric positive definite and let
    \begin{align}
        \thetabar_t = \left( \eye - e^{- \bA T} \right)^{-1} \left( \theta_t - e^{- \bA t} \theta_0 \right).
    \end{align}
    Then it holds for any function $\alpha_T: [0, T] \to \rr^d$ satisfying $\int_0^T \alpha_T(t) ~ d t = 1$ that
    \begin{align}
        \muouema_T = \int_0^T \alpha_T(t) \thetabar_t ~ d t
    \end{align}
    is an unbiased estimator of $\mustar$, i.e., $\ee_{\mustar}\left[ \muouema_T \right] = \mustar$.
\end{proposition}
\begin{proof}
    By the definition of $\thetabar_t$ and \eqref{eq:ou}, we have
    \begin{align}
        \ee_{\mustar}\left[ \thetabar_t \right] = \left( \eye - e^{-\bA t} \right)^{-1} \left( \ee_{\mustar}\left[ \theta_t \right] - e^{-\bA t} \theta_0 \right) = \mustar.
    \end{align}
    Now, using the linearity of expectation, we have that
    \begin{align}
        \ee_{\mustar}\left[ \muouema \right] &= \ee_{\mustar}\left[ \int_0^T \thetabar_t \alpha_T(t) d t \right] \\
        &= \int_0^T \ee_{\mustar}\left[\thetabar_t\right] \alpha_T(t) d t \\
        &= \mustar \int_0^T \alpha_T(t) d t = \mustar,
    \end{align}
    by the assumption on $\alpha_T$.
\end{proof}
We now instantiate $\alpha_T$ as a flat average over $[\tau, T]$ for some positive $\tau < T$ and control the variance of $\muouema_T$.
\begin{proposition}\label{prop:ouema_unbiased_variance}
    Let $(\theta_t)_{0 \leq t \leq T}$ be the solution to the OU process \eqref{eq:sde} with $\bA, \bSigma \in \rr^{d \times d}$ symmetric positive definite and for  $0 < \tau < T$, let
    \begin{align}
        \alpha_T(t) = \begin{cases}
            0 & t < \tau \\
            \frac{1}{T - \tau} & t \geq \tau
        \end{cases}.
    \end{align}
    Then it holds that
    \begin{align}
        \ee_{\mustar}\left[ \norm{\muouema_T - \mustar}^2 \right] \leq \frac{\eta \normop{\bSigma}^2 \trace\left( \bA^{-2} \right)}{\left( 1 - e^{- \lambdamin(\bA) \tau} \right)^2 \left( 1 - \nicefrac \tau T \right)^2 \cdot T}.
    \end{align}
\end{proposition}
\begin{proof}
    By \Cref{prop:ouema_unbiased}, we have that $\muouema_T$ is an unbiased estimator of $\mustar$ and thus the expected squared error is exactly equal to the variance.  We can now apply \Cref{lem:bias_variance_ou_int} to see that
    \begin{align}
        \Var\left( \muouema_T \right) &= \Var\left( \frac{1}{T - \tau} \int_{\tau}^T \left( \eye - e^{- \bA t} \right)^{-1} \theta_t ~ d t \right) \\
        &= \left( \frac{T}{T - \tau} \right)^2 \cdot \Var\left(\frac 1T \int_{\tau}^T \left( \eye - e^{- \bA t} \right)^{-1} \theta_t ~ d t \right) \\
        &\leq  \left( \frac{T}{T - \tau} \right)^2 \normop{\left( \eye - e^{- \bA \tau} \right)^{-1}}^2 \cdot \Var\left( \frac 1T \int_0^T \theta_t ~ dt \right) \\
        &\leq \frac{\eta \normop{\bSigma}^2 \trace\left( \bA^{-2} \right)}{\left( 1 - e^{- \lambdamin(\bA) \tau} \right)^2 \left( 1 - \nicefrac \tau T \right)^2 \cdot T}.
    \end{align}
    The result follows.
\end{proof}
While we consider the flat average function as a choice for $\alpha_T$ in \Cref{prop:ouema_unbiased}, the optimal choice of $\alpha_T$ is a different function.  Indeed, applying the calculus of variations \citep{Van_Brunt2004}, it is easy to see that the optimal choice of $\alpha_T$ is given (assuming sufficient regularity and finiteness of all quantities) by a scaled version of $\Kbar_T^{-1} \cdot 1$, where in the case that $\bSigma = \sigma \eye$,
\begin{align}
    \Cov\left( \thetabar_s, \thetabar_t \right) = \Kbar_T(s, t) = \frac{\eta \sigma^2}{2} \left( \eye - e^{- \bA s} \right)\left( \eye - e^{- \bA t} \right) \bA^{-1} \left( e^{- \bA \abs{t - s}} - e^{-\bA(t + s)} \right),
\end{align}
the equality holds by \Cref{lem:ou_cov}, the inverse is defined by considering $\Kbar_T$ as an integral operator, and $1$ represents the constant one function.  Due to the difficulty of computing the inverse of $\Kbar_T$, and the fact that we anyhow consider an exponential moving average in practice, we do not pursue this further here.

\end{document}